\documentclass{article}

\usepackage{amssymb}
\usepackage{amsmath,amsthm}
  
\usepackage{enumitem}

\usepackage{pifont}
\newcommand{\cmark}{\ding{51}}%
\newcommand{\xmark}{\ding{55}}

\newtheorem{Theorem}{Theorem}
 
\newtheorem{lemma}{Lemma} 
\newtheorem{Assumption}{Assumption} 
\newtheorem{corollary}{Corollary} 
\newtheorem{Definition}{Definition} 
 
\usepackage{multirow} 
\usepackage[table,dvipsnames]{xcolor}
\usepackage{colortbl}
\newcommand*{\belowrulesepcolor}[1]{%
	\noalign{%
		\kern-\belowrulesep
		\begingroup
		\color{#1}%
		\hrule height\belowrulesep
		\endgroup
	}%
}
\newcommand*{\aboverulesepcolor}[1]{%
	\noalign{%
		\begingroup
		\color{#1}%
		\hrule height\aboverulesep
		\endgroup
		\kern-\aboverulesep
	}%
}

\usepackage{times}

\usepackage{natbib}

\usepackage{amsmath,amssymb}
\usepackage{natbib}

\usepackage{adjustbox}

\usepackage{microtype}
\usepackage{graphicx}
\usepackage{subfigure}
\usepackage{booktabs} 

\usepackage{hyperref}

\usepackage{algorithm}
\usepackage{algorithmic}

\usepackage[accepted]{icml2019}

\icmltitlerunning{Improved Zeroth-Order Variance Reduced Algorithms and Analysis for Nonconvex Optimization}

\begin{document}
	
	\twocolumn[
	\icmltitle{Improved Zeroth-Order Variance Reduced Algorithms and Analysis for Nonconvex Optimization}
	
	
	
	\icmlsetsymbol{equal}{*}
	
	\begin{icmlauthorlist}
		\icmlauthor{Kaiyi Ji}{to}
		\icmlauthor{Zhe Wang}{to}
		\icmlauthor{Yi Zhou}{goo}
		\icmlauthor{Yingbin Liang}{to}
	\end{icmlauthorlist}
	
	\icmlaffiliation{to}{Department of Electrical and Computer Engineering, The Ohio State University}
	\icmlaffiliation{goo}{Department of Electrical and Computer Engineering, Duke University}
	
	\icmlcorrespondingauthor{Kaiyi Ji}{ji.367@osu.edu}
	
	\icmlkeywords{Machine Learning, ICML}
	
	\vskip 0.3in
	]
	


\printAffiliationsAndNotice{}  
	
	\begin{abstract}
Two types of zeroth-order stochastic algorithms have recently been designed for nonconvex optimization respectively based on the first-order techniques SVRG and SARAH/SPIDER. This paper addresses several important issues that are still open in these methods. First, all existing SVRG-type zeroth-order algorithms suffer from worse function query complexities than either zeroth-order gradient descent (ZO-GD) or stochastic gradient descent (ZO-SGD). In this paper, we propose a new algorithm ZO-SVRG-Coord-Rand and develop a new analysis for an existing ZO-SVRG-Coord algorithm proposed in~\citealt{liu2018zeroth}, and show that both ZO-SVRG-Coord-Rand and ZO-SVRG-Coord (under our new analysis) outperform other exiting SVRG-type zeroth-order methods as well as ZO-GD and ZO-SGD. Second, the existing SPIDER-type algorithm SPIDER-SZO \cite{fang2018spider} has superior theoretical performance, but suffers from the generation of a large number of Gaussian random variables as well as a $\sqrt{\epsilon}$-level stepsize in practice. In this paper, we develop a new algorithm ZO-SPIDER-Coord, which is free from Gaussian variable generation and allows a large constant stepsize while maintaining the same convergence rate and query complexity, and we further show that ZO-SPIDER-Coord automatically achieves a linear convergence rate as the iterate enters into a local PL region without restart and algorithmic modification. 
	\end{abstract}
	
\section{Introduction}
Zeroth-order  optimization has recently gained increasing attention due to its wide usage in many applications where the explicit expressions of gradients of the objective function are expensive or infeasible to obtain and only function evaluations are accessible. Such a class of  applications include black-box adversarial attacks on deep neural networks (DNNs)~\citep{papernot2017practical,chen2017zoo,kurakin2016adversarial}, 
structured prediction~\citep{taskar2005learning} and reinforcement learning~\citep{choromanski2018structured}. 

Various zeroth-order algorithms
have been developed to solve the following general finite-sum optimization problem
\begin{align}\label{Objective}
\underset{x\in\mathbb{R}^d}\min \; f(\mathbf{x}):=\frac{1}{n}\sum_{i=1}^n f_i(\mathbf{x})
\end{align}
where $d$ denotes  the input dimension and $\{ f_i(\cdot) \}_{i=1}^n$ denote  smooth and nonconvex individual loss functions.  
\citealt{nesterov2011random} introduced a zeroth-order gradient descent (ZO-GD)  algorithm using a two-point Gaussian random gradient estimator,  which  yields a convergence rate of $\mathcal{O}(d/K)$ (where $K$ is the number of iterations) and a function query complexity (i.e., the number of queried function values) of $\mathcal{O}(dn/\epsilon)$, to attain a stationary point $\mathbf{x}^{\zeta}$ such that $\mathbb{E}\|\nabla f(\mathbf{x}^\zeta)\|^2\leq \epsilon$. 
~\citealt{ghadimi2013stochastic} proposed a zeroth-order stochastic gradient descent (ZO-SGD) algorithm using the same gradient estimation technique as in~\citealt{nesterov2011random}, which has  a convergence rate of $\mathcal{O}(\sqrt{d/K})$ and a function query complexity of $\mathcal{O}(d/\epsilon^2)$.  

Furthermore, two types of zeroth-order stochastic variance reduced algorithms have been developed to further improve the convergence rate of ZO-SGD. The first type refers to the SVRG-based algorithm, which
replaces the gradient in SVRG~\citep{johnson2013accelerating} by zeroth-order gradient estimators. 
In particular,~\citealt{liu2018zeroth} proposed three zeroth-order  SVRG-based algorithms, namely,  ZO-SVRG based on a two-point random gradient estimator, ZO-SVRG-Ave based on an average random gradient estimator, and ZO-SVRG-Coord based on a coordinate-wise gradient estimator. The performances of the aforementioned algorithms are summarized in Table~\ref{tas1}. Though existing studies appear comprehensive, two important questions are still left open and require conclusive answers. 

\renewcommand{\arraystretch}{1.3} 
\definecolor{LightCyan}{rgb}{0.9,1,0.9}
\begin{table*}[t] 
	\small
	\centering 
	\caption{Comparison of zeroth-order SVRG-based algorithms in terms of  convergence rate and function query complexity for nonconvex   optimization. \footnotesize{$^\clubsuit$: ZO-SVRG, ZO-SVRG-Ave and ZO-SVRG-Coord in~\citealt{liu2018zeroth} have no single-sample versions.}
		\footnotesize{$^\spadesuit$:  $p$ denotes the number of i.i.d.~smoothing vectors  for constructing the average random gradient estimator. 
		} 
		\footnotesize{$^\bigstar$: batch size $|\mathcal{S}_1|=\min\big\{n,\big\lceil (K/d)^{3/5}\big\rceil \big\}$.}
	}
	\vspace{0.2cm}
	\begin{tabular}{llll} \toprule
		{Algorithms} &Stepsize  $\eta$ &  Convergence rate & Function query complexity \\   \midrule
		ZO-GD~\citep{nesterov2011random}&$\mathcal{O}\left( \frac{1}{d}\right)$  &  $\mathcal{O}\left(\frac{d}{K}\right)$    &$\mathcal{O}\left(  \frac{dn}{\epsilon} \right)$
		\\  \midrule 
		ZO-SGD~\citep{ghadimi2013stochastic}&$\mathcal{O}\left( \frac{1}{d}\right)$  &  $\mathcal{O}\left(\sqrt{\frac{d}{K}}\right)$    &$\mathcal{O}\left( \frac{d}{\epsilon^{2}}\right)$
		\\  \midrule 
		ZO-SVRG (mini-batch)~\cite{liu2018zeroth}$^\clubsuit$&$\mathcal{O}\left( \frac{1}{d}\right)$  &  $\mathcal{O}\left(\frac{d}{K}+\frac{1}{|\mathcal{S}_2|}\right)$    &$\mathcal{O}\left( \frac{n}{\epsilon}+\frac{d}{\epsilon^{2}}  \right)$
		\\  \midrule 
		{ZO-SVRG-Ave} (mini-batch)~\citep{liu2018zeroth} &$\mathcal{O}\left( \frac{1}{d}\right)$ & $\mathcal{O}\left(\frac{d}{K}+\frac{1}{|\mathcal{S}_2|\min\{d,p\}}\right)$&  $\mathcal{O}\left(\frac{pn}{\epsilon}+\max\big\{1,\frac{p}{d}\big\}\frac{d}{\epsilon^{2}}\right)^\spadesuit$ \\ \midrule
		ZO-SVRG-Coord (mini-batch)~\citep{liu2018zeroth} &$\mathcal{O}\left( \frac{1}{d}\right)$ & $\mathcal{O}\left(\frac{d}{K}\right)$&      $\mathcal{O}\left(dn+\frac{d^2}{\epsilon}+\frac{dn}{\epsilon}  \right)$  \\  \midrule
		\belowrulesepcolor{LightCyan}
		\rowcolor{LightCyan}
		ZO-SVRG-Coord (mini-batch) (our new analysis) &{\color{red}$\mathcal{O}(1)$} & $\mathcal{O}\left(  \frac{1}{K}\right)$&     $\mathcal{O}\left(\min\left\{\frac{dn^{2/3}}{\epsilon},\,\frac{d}{\epsilon^{5/3}}\right\}\right)$    \\    
		\rowcolor{LightCyan}
		ZO-SVRG-Coord-Rand (mini-batch) &{\color{red}$\mathcal{O}(1)$} & $\mathcal{O}\left( \frac{1}{K}\right)$&     $\mathcal{O}\left(\min\left\{\frac{dn^{2/3}}{\epsilon},\,\frac{d}{\epsilon^{5/3}}\right\}\right)$   \\  
		\rowcolor{LightCyan}
		ZO-SVRG-Coord-Rand (single-sample) &$\mathcal{O}\left(\frac{1}{d|\mathcal{S}_1|^{2/3}}\right)^\bigstar$ & $\mathcal{O}\left(  \frac{d|\mathcal{S}_1|^{2/3}}{K}\right)$&     $\mathcal{O}\left(\min\left\{\frac{dn^{2/3}}{\epsilon},\,\frac{d}{\epsilon^{5/3}}\right\}\right)$    \\  
		\aboverulesepcolor{LightCyan}  \bottomrule
	\end{tabular} 
	\label{tas1}
	\vspace{-0.1cm}
\end{table*} 

\renewcommand{\arraystretch}{1.3}   
\definecolor{LightCyan}{rgb}{0.9,1,0.9}
\begin{table*}[t] 
	\small
	\centering 
	\caption{Comparison of zeroth-order SPIDER-based  algorithms in terms of function query complexity and Gaussian sample complexity for nonconvex   optimization.  
		\footnotesize{$^\clubsuit$: SPIDER-SZO in~\citealt{fang2018spider} has no single-sample version. }
		\footnotesize{$^\spadesuit$: Gaussian sample complexity refers to the total number of generated Gaussian random samples for constructing gradient estimators.} 
		\footnotesize{$^\bigstar$: The epoch length  $q=\min\big\{n,\big \lceil K^{2/3}\big \rceil\big\}$.}}
	\vspace{0.2cm}
	\begin{tabular}{lllc} \toprule
		{Algorithms} &Stepsize  $\eta$ & Function query complexity & Gaussian sample complexity$^\spadesuit$ \\   \midrule
		SPIDER-SZO (mini-batch)~\citep{fang2018spider}$^\clubsuit$ &${\color{red}\mathcal{O}(\sqrt{\epsilon})}$ &    $\mathcal{O}\left(\min\left\{\frac{dn^{1/2}}{\epsilon},\,\frac{d}{\epsilon^{3/2}}\right\}\right)$ &  $\mathcal{O}\left( \frac{d^2n^{1/2}}{\epsilon}\right)$ \\   \aboverulesepcolor{LightCyan}  \bottomrule
		\belowrulesepcolor{LightCyan}
		\rowcolor{LightCyan}
		ZO-SPIDER-Coord (mini-batch) &${\color{red}\mathcal{O}( 1)}$ &     $\mathcal{O}\left(\min\left\{\frac{dn^{1/2}}{\epsilon},\,\frac{d}{\epsilon^{3/2}}\right\}\right)$   & None \\  
		\rowcolor{LightCyan}
		ZO-SPIDER-Coord (single-sample) &$\mathcal{O}\left(\frac{1}{\sqrt{q}}\right)^\bigstar$ &     $\mathcal{O}\left(\min\left\{\frac{dn^{1/2}}{\epsilon},\,\frac{d}{\epsilon^{3/2}}\right\}\right)$  & None \\  
		\aboverulesepcolor{LightCyan}  \bottomrule
	\end{tabular} 
	\label{tas11}
	\vspace{-0.2cm}
\end{table*} 

\begin{list}{$\bullet$}{\topsep=0.ex \leftmargin=0.26in \rightmargin=0.in \itemsep =-0.035in}
	\item[Q1.1] Although the existing zeroth-order SVRG-based algorithms have improved iteration rate of convergence (i.e., the dependence on $K$), their function query complexities are all larger than either ZO-GD or ZO-SGD. Whether there exist zeroth-order SVRG-based  algorithms that outperform ZO-GD and ZO-SGD in terms of both the function query complexity and the convergence rate is an intriguing open question.
	
	\item[Q1.2] As shown in \citealt{liu2018zeroth} (see Table~\ref{tas1}), ZO-SVRG-Coord 
	suffers from approximately $\mathcal{O}(d)$ time more function queries than ZO-SVRG and ZO-SVRG-Ave. However, such inferior performance may be due to bounding technicality rather than algorithm itself. Intuitively, 
	coordinate-wise estimator used in ZO-SVRG-Coord can estimate the gradient more accurately, and hence should require fewer iterations to convergence, so that  its overall complexity can be comparable or superior than ZO-SVRG and ZO-SVRG-Ave. Thus, a refined convergence analysis is needed. 
	
\end{list}


The second type of zeroth-order variance-reduced algorithms was proposed in~\citealt{fang2018spider}, named SPIDER-SZO, which replaces gradients in the SPIDER algorithm with zeroth-order gradient estimators. Differently from SVRG, SPIDER~\citep{fang2018spider} and an earlier version SARAH~\citep{nguyen2017sarah,nguyen2017stochastic} are 
first-order stochastic variance-reduced algorithms whose inner-loop iterations {\em recursively} incorporate the fresh gradients to update the gradient estimator (see \eqref{spiders}). 
\citealt{fang2018spider} showed that SPIDER-SZO achieves an improved  query complexity over SVRG-based zeroth-order algorithms.
However, SPIDER-SZO requires the generation of a large number $\mathcal{O}(n^{1/2}d^2)$ of i.i.d.~Gaussian random variables at \textit{each} inner-loop iteration, and requires a very small stepsize $\eta=\mathcal{O}(\sqrt{\epsilon}/(\|\mathbf{v}^k\|L))$ (where $\mathbf{v}^k$ is an estimate of gradient $\nabla f(\mathbf{x}^k)$) to guarantee the convergence. Such two requirements can substantially restrict the performance of SPIDER-SZO in practice. Thus, the following two  important questions arise. 

\begin{list}{$\bullet$}{\topsep=0.ex \leftmargin=0.26in \rightmargin=0.in \itemsep =-0.035in}
	\item[Q2.1] 
	Whether using coordinate-wise estimator for both inner and outer loops and at the same time enlarging the stepsize to the constant level provide competitive query complexity? If so, such a new zeroth-order SPIDER-based algorithm eliminates the aforementioned two  restrictive requirements in SPIDER-SZO.
	
	
	
	\item[Q2.2] The existing study of zeroth-order SPIDER-based algorithms is only for smooth nonconvex optimization, which is far from comprehensive. We further want to understand their performance under specific geometries such as the Polyak-{\L}ojasiewicz (PL) condition,  convexity and for  nonconvex nonsmooth composite optimization. 
	Can SPIDER-based algorithms still outperform other existing zeroth-order algorithms for these cases? 
\end{list}
In this paper, we provide comprehensive answers to the above questions.

\vspace{-0.2cm}
\subsection{Summary of Contributions}
\vspace{-0.1cm}
For SVRG-based algorithms, we provide affirmative answers to the questions Q1.1 and Q1.2. First, we propose a new zeroth-order SVRG-based algorithm ZO-SVRG-Coord-Rand
and show that it achieves the function query complexity of $\mathcal{O}\big(\min\big\{dn^{2/3}\epsilon^{-1},\,d\epsilon^{-5/3}\big\}\big)$ for nonconvex optimization, which order-wisely improves the performance of not only all existing zeroth-order SVRG-based algorithms (see Table~\ref{tas1}) but also ZO-GD and ZO-SGD. This for the first time establishes the order-wise complexity advantage of zeroth-order SVRG-based algorithms over the zeroth-order GD and SGD-based algorithms, and thus answers Q1.1.  Furthermore, we provide a new convergence and complexity analysis for ZO-SVRG-Coord~\citep{liu2018zeroth} with order-wise tighter bound, and show that it achieves the same fantastic function query complexity as ZO-SVRG-Coord-Rand, which answers Q1.2. 
Furthermore, our new analysis allows a much larger stepsize for performance guarantee.

For SPIDER-based algorithms, we provide affirmative answers to the questions Q2.1 and Q2.2. To answer Q2.1, we first propose a novel zeroth-order algorithm ZO-SPIDER-Coord fully using coordinate-wise gradient estimators, and show that it  achieves the same superior function query complexity 
as SPIDER-SZO~\citep{fang2018spider}. ZO-SPIDER-Coord is advantageous over SPIDER-SZO~\citep{fang2018spider} by fully eliminating the cost of Gaussian random variable generation and allowing a much larger stepsize $\eta=\mathcal{O}(1)$ to enable a faster convergence in practice. Such two advantages are both due to a new convergence analysis we develop for ZO-SPIDER-Coord.
To answer Q2.2, 
under the PL condition, we show that ZO-SPIDER-Coord achieves a  linear convergence rate  {\em without restart and algorithmic modification}. As a result,  ZO-SPIDER-Coord automatically achieves a much faster convergence rate when the iterate enters  a local region where the PL condition is satisfied. 

Due to the space limitations, we relegate our results on zeroth-order  {\em nonconvex nonsmooth} composite optimization and zeroth-order {\em convex} optimization to the supplementary materials, both of which outperform the corresponding existing algorithms with order-level improvement. 

Our analysis reveals  that for zero-order variance-reduced algorithms, although the coordinate-wise gradient estimator requires more queries than the two-point gradient estimator, it guarantees much higher estimation accuracy, which leads to a larger stepsize and a faster convergence rate.

\vspace{-0.1cm}
\subsection{Related Work}
\vspace{-0.1cm}
{\bf Zeroth-order convex optimization.} 
\citealt{nemirovsky1983problem} first proposed a 
one-point random sampling scheme to estimate the gradient $\nabla f(\mathbf{x})$ by querying $f(\cdot)$ at a random location close to $\mathbf{x}$. Such a  technique was then used  in many other areas, e.g., bandit optimization~\cite{flaxman2005online, shamir2013complexity} .
Multi-point gradient estimation approach was then proposed by~\citealt{agarwal2010optimal,nesterov2011random}, and further explored in~\citealt{wainwright2008graphical,duchi2015optimal, ghadimi2013stochastic, wang2017stochastic}. For example, based on a two-point Gaussian gradient estimator, ~\citealt{ghadimi2013stochastic} developed a ZO-SGD type of method and 
\citealt{balasubramanian2018zeroth} proposed a zeroth-order conditional gradient type of algorithm. 


{\bf Zeroth-order nonconvex optimization.}~\citealt{ghadimi2013stochastic} and~\citealt{nesterov2011random} proposed ZO-GD and its stochastic counterpart ZO-SGD, respectively. In~\citealt{lian2016comprehensive}, an asynchronous zeroth-order stochastic gradient (ASZO) algorithm was proposed for parallel optimization.~\citealt{gu2018faster} further improved the convergence rate of ASZO  by combining  SVRG technique with  coordinate-wise gradient estimators.~\citealt{liu2018stochastic} proposed a stochastic zeroth-order method with variance reduction under Gaussian smoothing. 
More recently,~\citealt{liu2018zeroth} provided a comprehensive analysis on SVRG-based zeroth-order  algorithms under three different gradient estimators.~\citealt{fang2018spider} further proposed a SPIDER-based zeroth-order method named SPIDER-SZO. Our study falls into this category, where we  propose new algorithms that improve the performance of existing algorithms and develop new complexity bounds that improve existing analysis.

{\bf Stochastic first-order algorithms.} Since 
zeroth-order algorithms have been developed based on various first-order algorithms, we  briefly summarizes some of them, which include but not limited to 
SGD~\citep{robbins1951}, SAG \citep{Nicolas2012}, SAGA \citep{Defazio2014}, SVRG \citep{johnson2013accelerating,Allen_Zhu2016}, SARAH~\cite{nguyen2017sarah,nguyen2017stochastic}, SNVRG~\cite{zhou2018stochastic}, SPIDER~\citep{fang2018spider}, SpiderBoost~\cite{wang2018spiderboost} and AbaSPIDER~\cite{ji2019faster}. 
If the objective function further satisfies the PL condition, \citealt{reddi2016stochastic} proved the linear convergence for SVRG and its proximal version ProxSVRG by incorporating a restart step.~\citealt{li2018simple} proposed ProxSVRG+ as an improved version of  ProxSVRG and proved its linear convergence without restart. This paper studies a zeroth-order SPIDER-based algorithm under the PL condition without restart. 
{\bf Notations.} We use $\mathcal{O}(\cdot)$ to hide absolute constants that are independent of problem parameters, 
and 
$\|\cdot\|$ to denote the Euclidean norm of a vector or the spectral norm of a matrix. 
We use $[n]$ to denote the set $\{1,2,....,n\}$,  $|\mathcal{S}|$ to denote the cardinality of a given set $\mathcal{S}$, and $\mathbf{e}_i$ to denote the vector that has  only one non-zero entry $1$ at its $i^{th}$ coordinate. Given a set $\mathcal{S}$ whose elements are drawn from $[n]$, define  $f_{\mathcal{S}}(\cdot):=\frac{1}{|\mathcal{S}|}\sum_{i\in\mathcal{S} }f_i(\cdot)$ and $\nabla f_{\mathcal{S}}(\cdot):=\frac{1}{|\mathcal{S}|}\sum_{i\in\mathcal{S} }\nabla f_i(\cdot)$.

\section{SVRG-based Zeroth-order Algorithms for Nonconvex Optimization}\label{se:svrg+} 
In this section, we first propose  a novel zeroth-order stochastic algorithm named ZO-SVRG-Coord-Rand, and  analyze its convergence and complexity performance. We then provide an improved analysis for the existing  ZO-SVRG-Coord algorithm proposed by~\citealt{liu2018zeroth}.

\subsection{ZO-SVRG-Coord-Rand Algorithm}
We propose a new SVRG-based zeroth-order algorithm ZO-SVRG-Coord-Rand in Algorithm~\ref{ours:2}, which is conducted in a multi-epoch way. At the beginning of each epoch (i.e., each outer-loop iteration),  we estimate the gradient $\nabla f_{\mathcal{S}_1}(\mathbf{x}^k)$ over a batch set $\mathcal{S}_1$ of data samples based on a  deterministic coordinate-wise gradient estimator 
$\hat \nabla_{\text{\normalfont coord}}f_{\mathcal{S}_1}(\mathbf{x}^{k})=\sum_{i=1}^d\frac{(f_{\mathcal{S}_1}(\mathbf{x}^{k}+\delta\mathbf{e}_i)-f_{\mathcal{S}_1}(\mathbf{x}^{k}-\delta\mathbf{e}_i))\mathbf{e}_i}{2\delta}.$
In the following inner-loop iterations, we construct the stochastic gradient estimator $\mathbf{v}^k$ based on a mini-batch $\mathcal{S}_2$ of data samples as 
%
%
\begin{align}\label{mainsd}
\mathbf{v}^k=&\frac{1}{|\mathcal{S}_2|}\sum_{j=1}^{|\mathcal{S}_2|}\big(\hat \nabla_{\text{rand}} f_{a_j}\big(\mathbf{x}^k; \mathbf{u}_j^k\big)-\hat \nabla_{\text{rand}} f_{a_j}\big(\mathbf{x}^{qk_0}; \mathbf{u}_j^k\big)\big) \nonumber
\\&+\hat \nabla_{\text{\normalfont coord}}f_{\mathcal{S}_1}(\mathbf{x}^{qk_0}),
\end{align} 
where 
$\hat \nabla_{\text{rand}} f_{a_j}(\mathbf{x}; \mathbf{u}_j^k)=\frac{d(f_{a_j}(\mathbf{x}+\beta \mathbf{u}_j^{k}) -f_{a_j}(\mathbf{x}))}{\beta}\mathbf{u}_j^k$
is a two-point random gradient estimate of $\nabla f_{a_j}(\mathbf{x})$ using a smoothing vector $\mathbf{u}_j^k$ and $k_0=\lfloor k/q \rfloor$. The above construction of $\mathbf{v}^k$ is the core of our Algorithm~\ref{ours:2}, which isdifferent  from the following estimator  in ZO-SVRG~\citep{liu2018zeroth}
\begin{align*}
\mathbf{v}^k=&\frac{1}{|\mathcal{S}_2|}\sum_{j=1}^{|\mathcal{S}_2|}\big(\hat \nabla_{\text{rand}} f_{a_j}(\mathbf{x}^k; \mathbf{u}^k)-\hat \nabla_{\text{rand}} f_{a_j}(\mathbf{x}^{qk_0}; \mathbf{u}^{qk_0})\big)\nonumber
\\&+\hat \nabla_{\text{rand}} f(\mathbf{x}^{qk_0}; \mathbf{u}^{qk_0}),
\end{align*}
where $\mathbf{u}^k$  is generated from the uniform distribution over the unit sphere at the $k^{th}$ iteration.

\begin{algorithm}[t]
	\caption{ZO-SVRG-Coord-Rand}
	\label{ours:2}
	\begin{algorithmic}[1]
		\STATE {\bfseries Input:} $q$, $ K=qh, h\in\mathbb{N}$,   $|\mathcal{S}_1|, |\mathcal{S}_2|,\mathbf{x}^0,\delta, \beta>0$, $\eta$
		\FOR{$k=0$ {\bfseries to} $K$}
		\IF{ $k \mod q =0$ }
		\STATE{Sample $\mathcal{S}_1$ from $[n]$ without replacement
			\\		Compute $\mathbf{v}^k=\hat \nabla_{\text{\normalfont coord}}f_{\mathcal{S}_1}(\mathbf{x}^{k})$}
		\ELSE 
		\STATE {
			Sample  $\mathcal{S}_2=\{a_1,a_2,...,a_{|\mathcal{S}_2|}\}$ from $[n]$ with replacement
			\\ Draw  i.i.d.  $\mathbf{u}_1^k,...,\mathbf{u}_{|\mathcal{S}_2|}^k$  from  uniform distribution over unit sphere	  
			\\   Compute  $\mathbf{v}^k$ according to~\eqref{mainsd}
		}
		\ENDIF
		\STATE{$\mathbf{x}^{k+1}=\mathbf{x}^k-\eta\mathbf{v}^k$}
		\ENDFOR
		\STATE {\bfseries Output:} $\mathbf{x}_{\zeta}$ from $\{ \mathbf{x}_0,...,\mathbf{x}_K\}$ uniformly at random
	\end{algorithmic}
\end{algorithm}

There are two key differences between our construction of $\mathbf{v}^k$ and the one in~\citealt{liu2018zeroth}.  
First, our construction of $\mathbf{v}^k$ introduces $|\mathcal{S}_2|$  i.i.d.~smoothing vectors $\{\mathbf{u}_j^k\}_{j=1}^{|\mathcal{S}_2|}$ in each inner-loop iteration to estimate both $\nabla f_{\mathcal{S}_2}(\mathbf{x}^k)$ and $\nabla f_{\mathcal{S}_2}(\mathbf{x}^{qk_0})$, whereas~\citealt{liu2018zeroth} uses a single smoothing vector $\mathbf{u}^k$ to estimate $\nabla f_{\mathcal{S}_2}(\mathbf{x}^k)$ and a single vector  $\mathbf{u}^{qk_0}$ to estimate $\nabla f_{\mathcal{S}_2}(\mathbf{x}^{qk_0})$.  
Second, we adopt a coordinate-wise gradient estimator in each outer-loop iteration, whereas~\citealt{liu2018zeroth} use a two-point random gradient estimator. 
As shown  in the next subsection, our treatment does not introduce extra function query cost but  achieves a much tighter estimation of $\nabla f(\mathbf{x}^k)$ by $\mathbf{v}^k$. 

\subsection{Complexity and Convergence Analysis }
Throughout this paper, we adopt the following standard assumption for the objective function~\cite{nesterov2011random,lian2016comprehensive,gu2018faster,liu2018zeroth}. 
\begin{Assumption}\label{assumption}
	We assume that  $f(\cdot)$ in~\eqref{Objective} satisfies:
	\vspace{-0.1cm} 
	\begin{itemize}
		\item[\normalfont(1)]  $0<f(\mathbf{x}^{0})-f(\mathbf{x}^*)<\infty$, where  $\mathbf{x}^*=\arg\min_{\mathbf{x}} f(\mathbf{x})$.
		\item[\normalfont(2)]  Each $f_i(\cdot),i=1,..., n$ has a $L$-Lipschitz gradient, i.e., for any $\mathbf{x},\mathbf{y}\in\mathbb{R}^d$,
		$\|\nabla f_i(\mathbf{x})-\nabla f_i(\mathbf{y})\|\leq L\| \mathbf{x}-\mathbf{y}\|$.
		\item[\normalfont(3)]  Assume that stochastic gradient $\nabla f_i(\cdot)$ has bounded variance, i.e.,  there exists a constant $\sigma>0$ such that 
		$\frac{1}{n}\sum_{i=1}^n\|\nabla f_i(\mathbf{x})-\nabla f(\mathbf{x})\|^2 \leq \sigma^2$.
	\end{itemize}
\end{Assumption}
\vspace{-0.1cm} 
The item (3) of the variance boundedness assumption is only needed for the online case with $|S_1|<n$. For the finite-sum case (i.e., $|S_1|=n$), the the variance boundedness assumption  is not needed.  

The following lemma provides a tighter upper bound  on the estimation variance $\mathbb{E}\|	\mathbf{v}^k-\nabla f_\beta(\mathbf{x}^k)\|^2$. 
\begin{lemma}\label{le:vks}
	Under Assumption~\ref{assumption}, we have, for  any  $qk_0\leq  k\leq \min\{ q(k_0+1)-1, qh\},\,k_0=0,...,h,$ 
	\begin{align*}
	&\mathbb{E}\|	\mathbf{v}^k-\nabla f_\beta(\mathbf{x}^k)\|^2  \leq \frac{6d L^2\|\mathbf{x}^k-\mathbf{x}^{qk_0}\|^2}{|\mathcal{S}_2|}+\frac{3L^2\beta^2d^2}{|\mathcal{S}_2|}  \nonumber
	\\&+\frac{18I(|\mathcal{S}_1|<n)}{|\mathcal{S}_1|}\left(  2L^2d\delta^2+\sigma^2\right)+6L^2d\delta^2+\frac{3\beta^2L^2d^2}{2}
	\end{align*}
	where  $f_\beta(\mathbf{x})=\mathbb{E}_{\mathbf{u}}\left(f(\mathbf{x}+\beta\mathbf{u})\right)$ with  $\mathbf{u}$ drawn from the  uniform distribution over the $d$-dimensional unit Euclidean ball, and $I(A)=1$ if the event $A$ occurs and $0$ otherwise.   
\end{lemma}
The  bound in Lemma~\ref{le:vks}  improves that in Proposition 1 of  ZO-SVRG~\citep{liu2018zeroth} by eliminating its two additional error terms $\mathcal{O}(d\,\mathbb{E}\|\nabla f(\mathbf{x}^k)\|^2)$ and $\mathcal{O}(d/|\mathcal{S}_2|)$. Such an improvement is due to our development of a novel and tight  inequality 	$\mathbb{E}_k\left\|	\hat \nabla_{\normalfont \text{rand}} f_{a_j}(\mathbf{x}^k; \mathbf{u}_j^k)-\hat \nabla_{\normalfont \text{rand}} f_{a_j}(\mathbf{x}^{qk_0}; \mathbf{u}_j^k)\right\|^2\leq 3d\|\nabla f_{a_j}(\mathbf{x}^k)-\nabla f_{a_j}(\mathbf{x}^{qk_0})\|^2+\frac{3L^2d^2\beta^2}{2}\leq 3dL^2\|\mathbf{x}^k-\mathbf{x}^{qk_0}\|^2+\frac{3L^2d^2\beta^2}{2}$ (See Lemma~\ref{unifoT} in the supplementary materials), which can be of independent interest for analyzing other zeroth-order methods. 
Based on Lemma~\ref{le:vks}, we show that   ZO-SVRG-Coord-Rand algorithm  achieves significant improvements both in  the  convergence rate and the  function query complexity, as shown in the subsequent analysis. 
\begin{Theorem}\label{th:svrg}
	Let Assumptions~\ref{assumption} hold, and define
	\begin{align}\label{ppsse}
	\lambda &=\frac{\eta}{4}-\frac{4c\eta}{g}-\frac{3L\eta^2}{2}, \rho =\left(6\eta^2L+\frac{c\eta }{g}\right)L^2d^2\beta^2,  \nonumber
	\\\chi &= \beta^2L^2d^2 +\frac{9I(|\mathcal{S}_1|<n)}{|\mathcal{S}_1|}\left(  2L^2d\delta^2+\sigma^2\right)+3L^2d\delta^2,  \nonumber
	\\\tau&=\left( \frac{\eta}{2}+\frac{2c\eta}{g}+4c\eta^2+3L\eta^2 \right)\chi+\rho, 
	\end{align}
	where $g$ is a positive parameter and $c$ is a constant such that  
	\begin{align}\label{cchoose}
	0<c=\frac{9dL^3\eta^2}{|\mathcal{S}_2|}\frac{(1+\eta g +12\eta^2 dL^2/|\mathcal{S}_2|)^q-1}{\eta g +12\eta^2 dL^2/|\mathcal{S}_2|}.
	\end{align}
	Then, the output $\mathbf{x}^\zeta$ of Algorithm~\ref{ours:2} satisfies 
	\begin{align}\label{hilys}
	\mathbb{E}\|\nabla f(\mathbf{x}^\zeta)\|^2\leq\frac{f_\beta(\mathbf{x}^0)-f_{\beta}(\mathbf{x}_\beta^*)}{\lambda}\frac{1}{K+1}+\frac{\tau}{\lambda}
	\end{align}
	where $\mathbf{x}_\beta^*=\arg\min_{\mathbf{x}}f_\beta(\mathbf{x})$. 
\end{Theorem}
Compared with the standard SVRG analysis  (Theorem 2 in~\citealt{reddi2016stochastic}), Theorem~\ref{th:svrg} involves an additional  term $\tau/\lambda$ in the upper bound on $\mathbb{E}\|\nabla f(\mathbf{x}^\zeta)\|^2$. By choosing sufficiently small smoothing parameters as well as a large mini-batch size $|\mathcal{S}_1|$, we guarantee that  such an error term is  dominated by the first term in~\eqref{hilys}, as shown below. 
\begin{corollary}[mini-batch, $|\mathcal{S}_2|>1$]\label{co1:svrg}
	Under the setting of Theorem~\ref{th:svrg},  let $g=4000d\eta^2L^3q/|\mathcal{S}_2|$ and choose
	\begin{align}\label{pacos}
	&\eta=\frac{1}{20L},   |\mathcal{S}_1|=\min\{n,K\}, q=\big \lceil |\mathcal{S}_1|^{1/3}\big\rceil \nonumber
	\\  &|\mathcal{S}_2|=dq^2, \beta=\frac{1}{Ld\sqrt{K}}, \delta=\frac{1}{L\sqrt{dK}},
	\end{align}
	where $e$ is the Euler's number. Then,  Algorithm~\ref{ours:2} satisfies  $\mathbb{E}\|\nabla f(\mathbf{x}^\zeta)\|^2\leq \mathcal{O}(1/K)$
	
	To achieve an $\epsilon$-stationary point, i.e.,  $\mathbb{E}\|\nabla f(\mathbf{x}^\zeta)\|^2\leq \epsilon$, the number of function queries required by Algorithm~\ref{ours:2}  is at most 
	$\mathcal{O}\left(\min\left\{n^{2/3}d\epsilon^{-1}, d\epsilon^{-5/3} \right\}\right)$.
\end{corollary}
Corollary~\ref{co1:svrg} implies that mini-batch ZO-SVRG-Coord-Rand  achieves a convergence rate of $\mathcal{O}(1/K)$, which improves the best rate of existing zeroth-order algorithms for nonconvex optimization by a factor of $\mathcal{O}(d)$.  In particular,  the function query complexity of our ZO-SVRG-Coord-Rand algorithm improves upon that of ZO-SGD by  a factor of $\mathcal{O}(\epsilon^{-1/3})$, and that of ZO-GD by a factor of $\mathcal{O}(n^{1/3})$.  As  far as we know, this is the first SVRG-based zeroth-order algorithm that outperforms both ZO-GD and ZO-SGD in terms of the function query complexity.


The mini-batching strategy in Corollary~\ref{co1:svrg} may require a parallel computation of $\mathbf{v}^k$. 
For nonparallel scenarios, we provide the following single-sample ZO-SVRG-Coord-Rand, which achieves the same function query complexity as mini-batch  ZO-SVRG-Coord-Rand.
\begin{corollary}[Single-sample, $|\mathcal{S}_2|=1$] \label{co:svrg2}
	Under the setting of Theorem~\ref{th:svrg},  let $g=4000 dq\eta^2 L^3$ and 
	\begin{align}\label{p22}
	&|\mathcal{S}_1|=\min\big\{n,\big\lceil (K/d)^{3/5}\big\rceil \big\}, q= |\mathcal{S}_1|d, \, \beta=\frac{|\mathcal{S}_1|^{1/3}}{L\sqrt{dK}} \nonumber
	\\& \eta=\frac{1}{20d^{1/3}q^{2/3}L}, \delta=\frac{|\mathcal{S}_1|^{1/3}}{L\sqrt{K}}, 
	\end{align}
	where $e$ is the Euler's number. Then, Algorithm~\ref{ours:2} satisfies 
	$\mathbb{E}\|\nabla f(\mathbf{x}^\zeta)\|^2
	\leq\mathcal{O}\left( d|\mathcal{S}_1|^{2/3}/K\right)$.
	
	To achieve an $\epsilon$-stationary point, i.e., $\mathbb{E}\|\nabla f(\mathbf{x}^\zeta)\|^2\leq \epsilon$, the number of function queries required by Algorithm~\ref{ours:2}  is at most $\mathcal{O}\left(\min\left\{n^{2/3}d\epsilon^{-1}, d\epsilon^{-5/3} \right\}\right)$.
\end{corollary}

\subsection{New Analysis for ZO-SVRG-Coord}
In this subsection, we provide an improved analysis for the ZO-SVRG-Coord algorithm proposed by~\citealt{liu2018zeroth}, which 
adopts the same outer-loop iteration as mini-batch ZO-SVRG-Coord-Rand, i.e., Algorithm~\ref{ours:2}, but updates the inner-loop estimator $\mathbf{v}^k$ coordinate-wisely  by 
\begin{align}\label{vnew}
\mathbf{v}^k=&\frac{1}{|\mathcal{S}_2|}\sum_{j=1}^{|\mathcal{S}_2|}\hat \nabla _{\text{coord}} f_{a_j}(\mathbf{x}^k)- \frac{1}{|\mathcal{S}_2|}\sum_{j=1}^{|\mathcal{S}_2|}\hat \nabla _{\text{coord}} f_{a_j}(\mathbf{x}^{qk_0}) \nonumber
\\&+\hat \nabla_{\text{\normalfont coord}}f_{\mathcal{S}_1}(\mathbf{x}^{qk_0}).
\end{align}
We first show that although the coordinate-wise gradient estimator in~\eqref{vnew} requires $d$ times more function queries than the two-point random gradient estimator at each inner-loop iteration, it  achieves  more accurate gradient estimation, as stated in the following lemma. 

\begin{lemma}\label{le:newnew}
	Under Assumption~\ref{assumption}, we have, for  any  $qk_0\leq  k\leq \min\{ q(k_0+1)-1, qh\},\,k_0=0,...,h,$ 
	\begin{align*}
	\mathbb{E}\|	\mathbf{v}^k &-\hat \nabla_{\text{\normalfont coord}}f(\mathbf{x}^k)\|^2\leq  \frac{12L^2d\delta^2}{|\mathcal{S}_2|}+\frac{6L^2}{|\mathcal{S}_2|}\|\mathbf{x}^k-\mathbf{x}^{qk_0}\|^2
	\\&+\frac{6I(|\mathcal{S}_1|<n)}{|\mathcal{S}_1|}\left(  2L^2d\delta^2+\sigma^2\right).
	\end{align*}
\end{lemma}
It can be seen that the above bound in Lemma~\ref{le:newnew} contains a tighter error term  $\frac{6L^2}{|\mathcal{S}_2|}\|\mathbf{x}^k-\mathbf{x}^{qk_0}\|^2$ than that in 
Lemma~\ref{le:vks} by a factor of $\mathcal{O}(d)$. More importantly, this bound is  tighter than that in Theorem 3 in~\citealt{liu2018zeroth} for ZO-SVRG-Coord by a factor of $\mathcal{O}(d)$. Based on Lemma~\ref{le:newnew}, we have the following theorem. 
\begin{Theorem}\label{th:newsa}
	Let Assumption~\ref{assumption} hold, and select 
	\begin{align}\label{ppsassss}
	&\eta=\frac{1}{15L},  |\mathcal{S}_1|=\min\{n,K\}, q=\left\lceil |\mathcal{S}_1|^{1/3}\right\rceil \nonumber
	\\& |\mathcal{S}_2|=q^2, \delta=\frac{1}{L\sqrt{dK}},
	\end{align}
	where $e$ is the Euler's number. Then, ZO-SVRG-Coord satisfies 
	$\mathbb{E}\|\nabla f(\mathbf{x}^\zeta)\|^2\leq  \mathcal{O}(1/K).$
	
	To achieve an $\epsilon$-stationary point, i.e., $\mathbb{E}\|\nabla f(\mathbf{x}^\zeta)\|^2\leq \epsilon$, the number of function queries required by ZO-SVRG-Coord  is  at most $\mathcal{O}\left(\min\left\{n^{2/3}d\epsilon^{-1}, d\epsilon^{-5/3} \right\}\right)$. 
\end{Theorem}

Theorem~\ref{th:newsa} order-wisely improves the complexity bound  in~\citealt{liu2018zeroth} by a factor of $\mathcal{O}(\max\{n^{1/3}, dn^{-2/3}\})$ due to our new analysis. Furthermore,   Theorem~\ref{th:newsa}
shows that  ZO-SVRG-Coord achieves the same performance as ZO-SVRG-Coord-Rand, both of which order-wisely improves ZO-GD and ZO-SGD  in the convergence rate as well as the  function query complexity for nonconvex optimization. Moreover, both ZO-SVRG-Coord-Rand and ZO-SVRG-Coord (under our new  analysis) allows a much larger stepsize $\mathcal{O}(1)$ than $\eta=\mathcal{O}(1/d)$ used in ZO-SGD and all zeroth-order SVRG-based  algorithms in~\citealt{liu2018zeroth}, and hence converges much faster in practice, as demonstrated in our experiments. 

\section{ZO-SPIDER-Coord Algorithm for Nonconvex Optimization} \label{spidersss}

Recently,~\citealt{nguyen2017sarah,nguyen2017stochastic} and~\citealt{fang2018spider} proposed a new first-order variance-reduced stochastic gradient estimator named SARAH and SPIDER respectively, which estimates stochastic gradients in a {\em recursive} way as 
\begin{align}\label{spiders}
\mathbf{v}^k=\nabla f_{\mathcal{S}_2}(\mathbf{x}^k)-\nabla f_{\mathcal{S}_2}(\mathbf{x}^{k-1})+\mathbf{v}^{k-1}.
\end{align}
In this section, we explore the performance of  this estimator in  zeroth-order nonconvex optimization. Motivated by our new analysis for ZO-SVRG-Coord, 
we propose  a zeroth-order SPIDER-based algorithm ZO-SPIDER-Coord, as shown in Algorithm~\ref{ours:3}.  
Our 
ZO-SPIDER-Coord extends the estimator~\eqref{spiders} for zeroth-order optimization by 
\begin{align}\label{codnew}
\mathbf{v}^k=\hat \nabla _{\text{coord}} f_{\mathcal{S}_2}(\mathbf{x}^k)-\hat \nabla _{\text{coord}} f_{\mathcal{S}_2}(\mathbf{x}^{k-1}) +\mathbf{v}^{k-1}.
\end{align}
where $ \hat \nabla _{\text{coord}} f_{\mathcal{S}_2}(\mathbf{x})=\sum_{i=1}^d\frac{(f_{\mathcal{S}_2}(\mathbf{x}+\delta\mathbf{e}_i)-f_{\mathcal{S}_2}(\mathbf{x}-\delta\mathbf{e}_i))\mathbf{e}_i}{2\delta}$. 
Differently from the existing SPIDER-based zeroth-order algorithm SPIDER-SZO proposed in \citealt{fang2018spider}, which requires to generate totally $\mathcal{O}(d^2\sqrt{n}\epsilon^{-1})$ Gaussian random variables (see Theorem 8 in~\citealt{fang2018spider}), our ZO-SPIDER-Coord eliminates Gaussian variable generation due to the utilization of coordinate-wise gradient estimator, and still achieves the same complexity performance as SPIDER-SZO, as shown in the next subsection. 
In addition, mini-batch ZO-SPIDER-Coord allows a large constant stepsize (see Corollary~\ref{comain}), as apposed to the small stepsize $\mathcal{O}(\sqrt{\epsilon}/\|\mathbf{v}^k\|)$ used in SPIDER-SZO for guaranteeing the convergence.  Similar idea has also been  used in SpiderBoost~\cite{wang2018spiderboost} to enhance the stepsize of SPIDER~\cite{fang2018spider}.

\vspace{-0.1cm}
\subsection{Convergence and Complexity Analysis }
The following theorem provides the convergence guarantee for ZO-SPIDER-Coord.
\begin{algorithm}[t]
	\caption{ZO-SPIDER-Coord}
	\label{ours:3}
	\begin{algorithmic}[1]
		\STATE {\bfseries Input:}  $K$,   $q$,  $|\mathcal{S}_1|\leq n$, $|\mathcal{S}_2|$, $\mathbf{x}^0$,  $\delta>0$,  $\eta>0$.
		\FOR{$k=0$ {\bfseries to} $K$}
		\IF{$k\mod q =0$ }
		\STATE{Sample $\mathcal{S}_1$ from $[n]$ without replacement
			\\		Compute $\mathbf{v}^k=\hat \nabla_{\text{\normalfont coord}}f_{\mathcal{S}_1}(\mathbf{x}^{k})$}
		\ELSE 
		\STATE {
			Sample  $\mathcal{S}_2=\big\{a_1,a_2....,a_{|\mathcal{S}_2|}\big\}$  from $[n]$ with replacement
			\\ Compute $\mathbf{v}^k$ according to~\eqref{codnew}}
		\ENDIF
		\STATE{$\mathbf{x}^{k+1}=\mathbf{x}^k-\eta\mathbf{v}^k$}
		\ENDFOR
		\STATE {\bfseries Output:} $\mathbf{x}_{\zeta}$ from $\{ \mathbf{x}_0,...,\mathbf{x}_K\}$ uniformly at random.
	\end{algorithmic}
\end{algorithm}

\begin{Theorem}\label{mainTT}
	Let Assumption~\ref{assumption} hold, and define 
	\begin{align}\label{pallo}
	&\phi=  \frac{\eta}{2}-\frac{\eta^2L}{2}- \frac{3L^2\eta^3q}{|\mathcal{S}_2|},\nonumber \\&\theta=\frac{3qL^2d\delta^2}{|\mathcal{S}_2|}+\frac{3I(|\mathcal{S}_1|<n)}{|\mathcal{S}_1|}\left(  2L^2d\delta^2+\sigma^2\right).
	\end{align}
	Then, the output $\mathbf{x}^\zeta$ of Algorithm~\ref{ours:3} satisfies 
	\begin{align}\label{uioj}
	\mathbb{E}&\|f(\mathbf{x}^{\zeta})\|^2 \leq 3L^2d\delta^2+3\theta \nonumber
	\\&+\left( \frac{9q\eta^2L^2}{\phi|\mathcal{S}_2|} +\frac{3}{\phi}  \right)\left(   \frac{\Delta}{K}+\eta \big(\theta+L^2d\delta^2\big)\right),
	\end{align}
	where $\Delta:=f(\mathbf{x}^{0})-f(\mathbf{x}^*)$ with $\mathbf{x}^*:=\arg\min_{\mathbf{x}} f(\mathbf{x})$.
\end{Theorem}
Based on Theorem~\ref{mainTT}, 
we provide an analysis on mini-batch ZO-SPIDER-Coord.
\begin{corollary}[Mini-batch, $|\mathcal{S}_2|>1$]\label{comain}
	Under the setting of Theorem~\ref{mainTT}, we choose stepsize $\eta=\frac{1}{4L}$ and 
	\begin{small}
		\begin{align}\label{pabu1}
		|\mathcal{S}_1|=\min\{n,K\}, |\mathcal{S}_2|=q=\lceil |\mathcal{S}_1|^{1/2}\rceil,\delta=\frac{1}{\sqrt{Kd}L}.
		\end{align}
	\end{small}
	\hspace{-0.12cm}Then, Algorithm~\ref{ours:3} satisfies 
	$\mathbb{E}\|\nabla f(\mathbf{x}^\zeta)\|^2\leq \mathcal{O}(1/K)$. 
	
	To achieve an $\epsilon$-stationary point, i.e., $\mathbb{E}\|\nabla f(\mathbf{x}^\zeta)\|^2\leq \epsilon$, the number of function queries required by Algorithm~\ref{ours:3}  is 
	$\mathcal{O}\left(\min\left\{n^{1/2}d\epsilon^{-1}, d\epsilon^{-3/2} \right\}\right)$.
\end{corollary}
As shown in Corollary~\ref{comain}, mini-batch ZO-SPIDER-Coord achieves the convergence rate of $\mathcal{O}(1/K)$, and  improves the function query complexity of ZO-SVRG-Coord-Rand by a factor of $\min\{\epsilon^{-1/6}, n^{1/6}\}$.   
The following corollary  analyzes  single-sample  ZO-SPIDER-Coord, which achieves the same query complexity as mini-batch ZO-SPIDER-Coord.  
\begin{corollary}[Single-sample, $|\mathcal{S}_2|=1$] \label{nonbatch}
	Under the setting of Theorem~\ref{mainTT}, we choose $ 
	\eta=\frac{1}{4L\sqrt{q}},\delta=\frac{1}{\sqrt{qKd}L},q=|\mathcal{S}_1|=\min\left\{n,\big \lceil K^{2/3}\big \rceil\right\}$. 
	Then, Algorithm~\ref{ours:3} satisfies 
	$\mathbb{E}\|\nabla f(\mathbf{x}^\zeta)\|^2\leq \mathcal{O}(\sqrt{|\mathcal{S}_1|}/K)$.
	
	To achieve an $\epsilon$-stationary point, i.e., $\mathbb{E}\|\nabla f(\mathbf{x}^\zeta)\|^2\leq \epsilon<1$, the number of function queries required by Algorithm~\ref{ours:3}  is $\mathcal{O}\left(\min\left\{n^{1/2}d\epsilon^{-1}, d\epsilon^{-3/2} \right\}\right)$. 
	
	
\end{corollary}

\subsection{ZO-SPIDER-Coord  under PL  without Restart}

Many nonconvex machine learning and deep learning problems satisfy the following Polyak-{\L}ojasiewicz  (PL) (i.e., gradient dominance) condition in  local regions  around  global minimizers~\citep{zhou2016geometrical,zhong2017recovery}.
\begin{Definition}[\cite{polyak1963gradient}]\label{d1}
	Let $\mathbf{x}^*=\arg\min_{\mathbf{x}\in\mathbb{R}^d}f(\mathbf{x})$. Then, the function $f$ is said to be $\gamma$-gradient dominated if for any $\mathbf{x}\in\mathbb{R}^d$, 
	$f(\mathbf{x})-f(\mathbf{x}^*)\leq \gamma \|\nabla f(\mathbf{x})\|^2.$
\end{Definition}
In this subsection, we explore whether ZO-SPIDER-Coord algorithm achieves a faster convergence rate when it enters the local areas where the loss function satisfies the PL condition. The following theorem provides an affirmative answer. For the simplicity of presentation, we choose $|\mathcal{S}_1|=n$. 
\vspace{0.05cm}
\begin{Theorem}\label{th:gd}
	Under the  parameters selected in  Corollary~\ref{comain}, we take $|\mathcal{S}_2|=\lceil \gamma LB_\gamma \rceil$ and  $\delta=\mathcal{O}(\sqrt{\epsilon}/(L\sqrt{\gamma d}))$. We further  assume that $\gamma>q/(b_\gamma L)$, where $b_\gamma$ and $B_\gamma$ are two positive constants satisfying  
	$\frac{1}{8}\left(1-\frac{b_\gamma}{16q}\right)^q -\frac{3 }{B_\gamma}>0$.
	Then, ZO-SPIDER-Coord satisfies 
	\begin{small}
		\begin{align*}
		\mathbb{E}(f(\mathbf{x}^{K})-f(\mathbf{x}^*))  \leq&  \left(1-\frac{1}{16L\gamma}\right)^K(f(\mathbf{x}^{0})-f(\mathbf{x}^*))+\epsilon,
		\end{align*}
	\end{small}
	\hspace{-0.15cm}and requires $  \mathcal{O}\left( d(\gamma n^{1/2}+\gamma^2)\log \left( \frac{1}{\epsilon}\right)\right)$ function queries. 
\end{Theorem}
The assumption that $\gamma>q/(b_\gamma L)=\mathcal{O}(n^{1/2}/L)$ has been widely adopted in optimization under the PL condition, e.g., in \citealt{reddi2016stochastic,li2018simple}. 
In contrast to the restart technique commonly used in the first-order algorithms, e.g., GD-SVRG~\cite{reddi2016stochastic}, for proving the convergence under the PL condition,  
our proof of the  linear convergence rate for ZO-SPIDER-Coord does not  require restart and algorithmic modification. 
This implies that ZO-SPIDER-Coord can be initialized in a general nonconvex landscape and then automatically achieves  a faster convergence rate as  it enters a PL landscape. In addition, unlike  SPIDER~\cite{fang2018spider} and SpiderBoost~\citep{wang2018spiderboost}, our 
proof of Theorem~\ref{th:gd} does not need to upper-bound $\sum_{k=1}^K\mathbb{E}\|\mathbf{v}^k\|^2$, which is much simpler and 
can also be applied to  both SPIDER and SpiderBoost for  first-order nonconvex optimization under the PL condition. 

\vspace{-0.15cm}
\section{Experiments}
\vspace{-0.1cm}
In this section, we  compare the empirical performance of our proposed  ZO-SVRG-Coord-Rand,  ZO-SPIDER-Coord and ZO-SVRG-Coord (for which we provide improved analysis that allows a larger stepsize)  with  ZO-SGD~\cite{ghadimi2013stochastic}, ZO-SVRG-Ave (p=10)\footnote{ZO-SVRG-Ave (p=10) has the best performance among the three methods proposed by~\citealt{liu2018zeroth}.}~\cite{liu2018zeroth} and SPIDER-SZO~\cite{fang2018spider}. We conduct two experiments, i.e.,  generation of black-box adversarial  examples and nonconvex logistic regression. The parameter settings for these algorithms are further specified in the supplementary materials due to the space limitations. 

\begin{figure*}[t]
	\centering     
	\subfigure[Loss vs iterations: $n=10$] {\label{fig:a}\includegraphics[width=42mm]{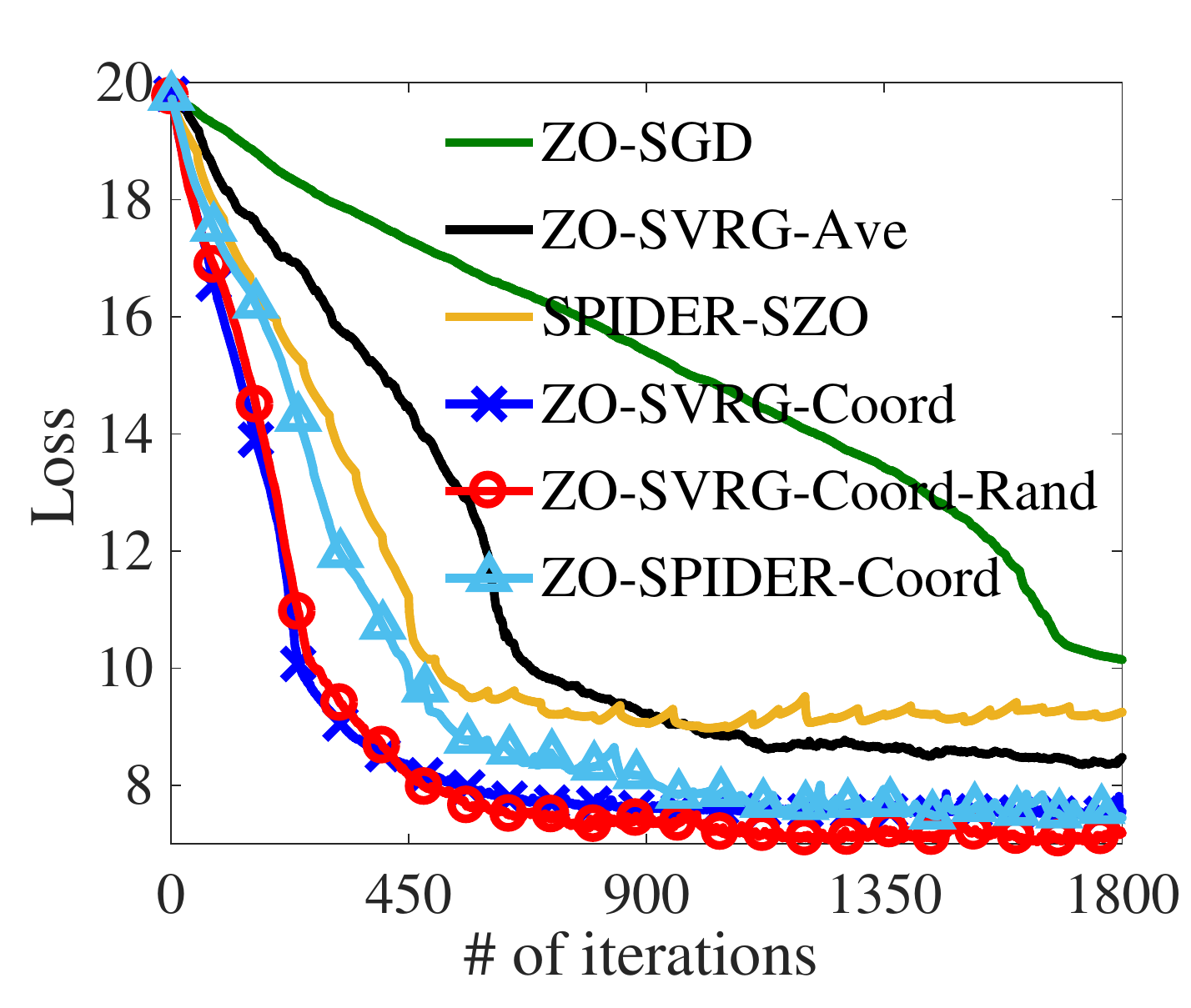}}
	\subfigure[Loss vs queries: $n=10$]{\label{fig:b}\includegraphics[width=42mm]{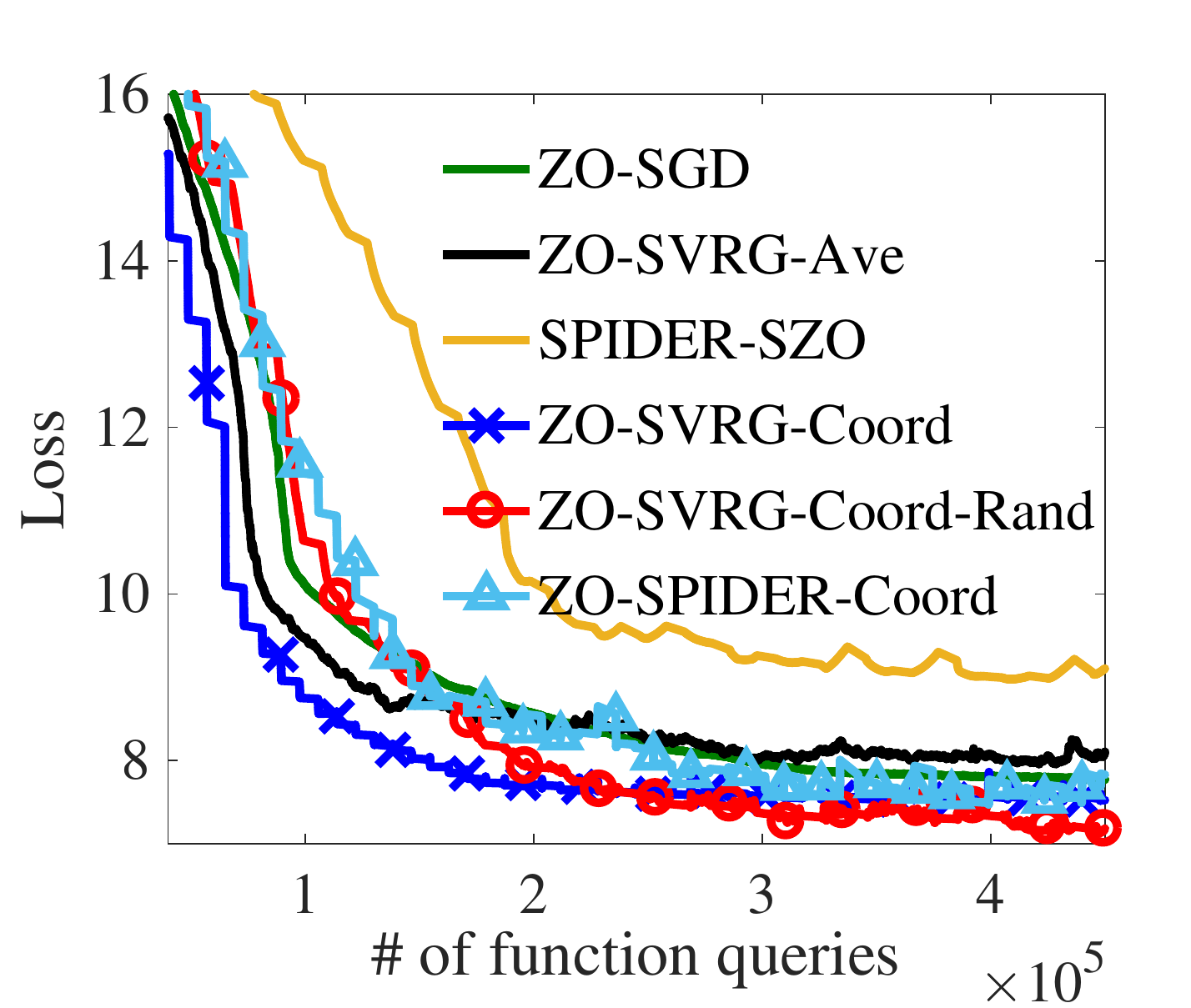}}
	\subfigure[Loss vs iterations: $n=100$]{\label{fig:c}\includegraphics[width=42mm]{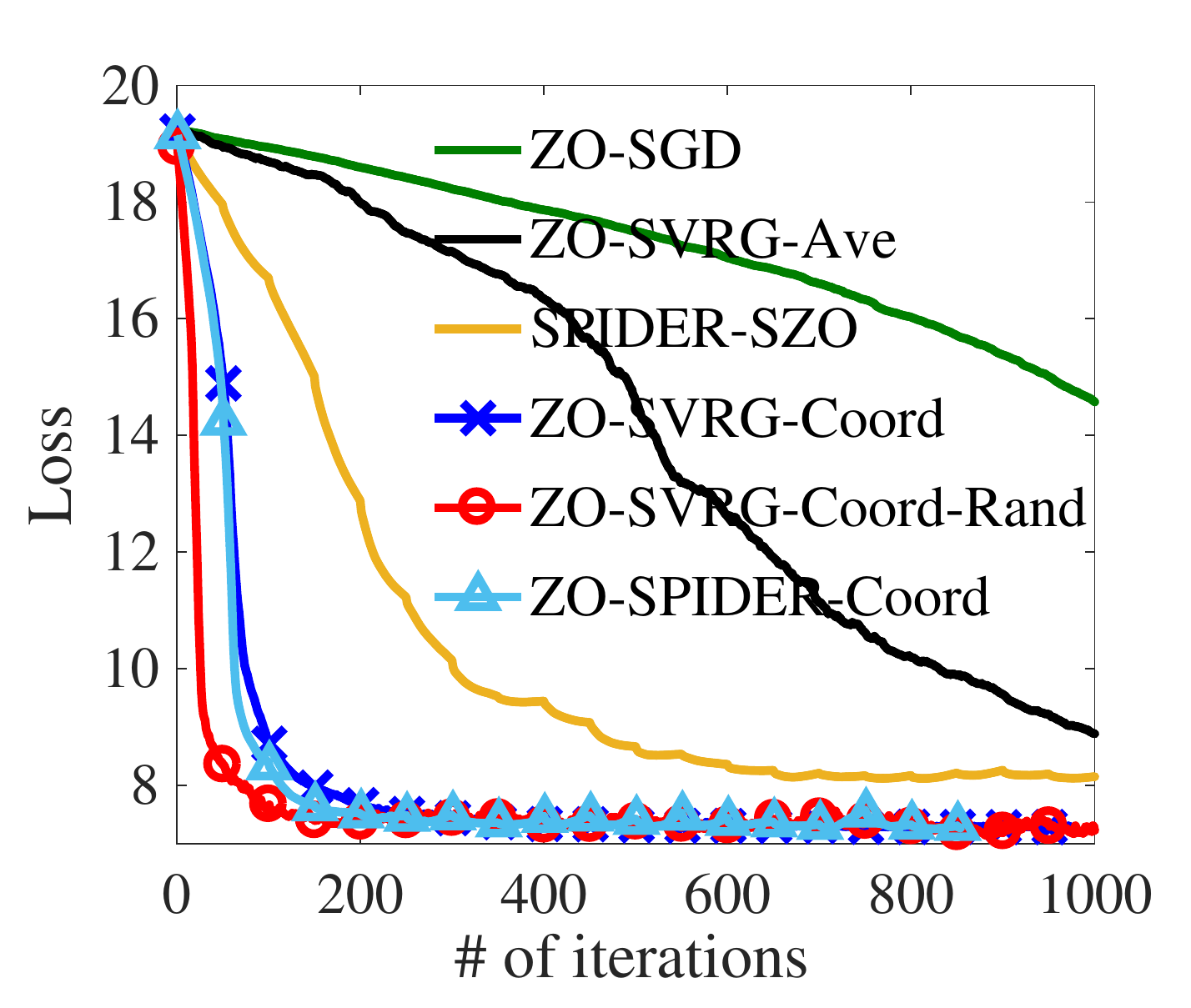}}
	\subfigure[Loss vs  queries: $n=100$]{\label{fig:d}\includegraphics[width=42mm]{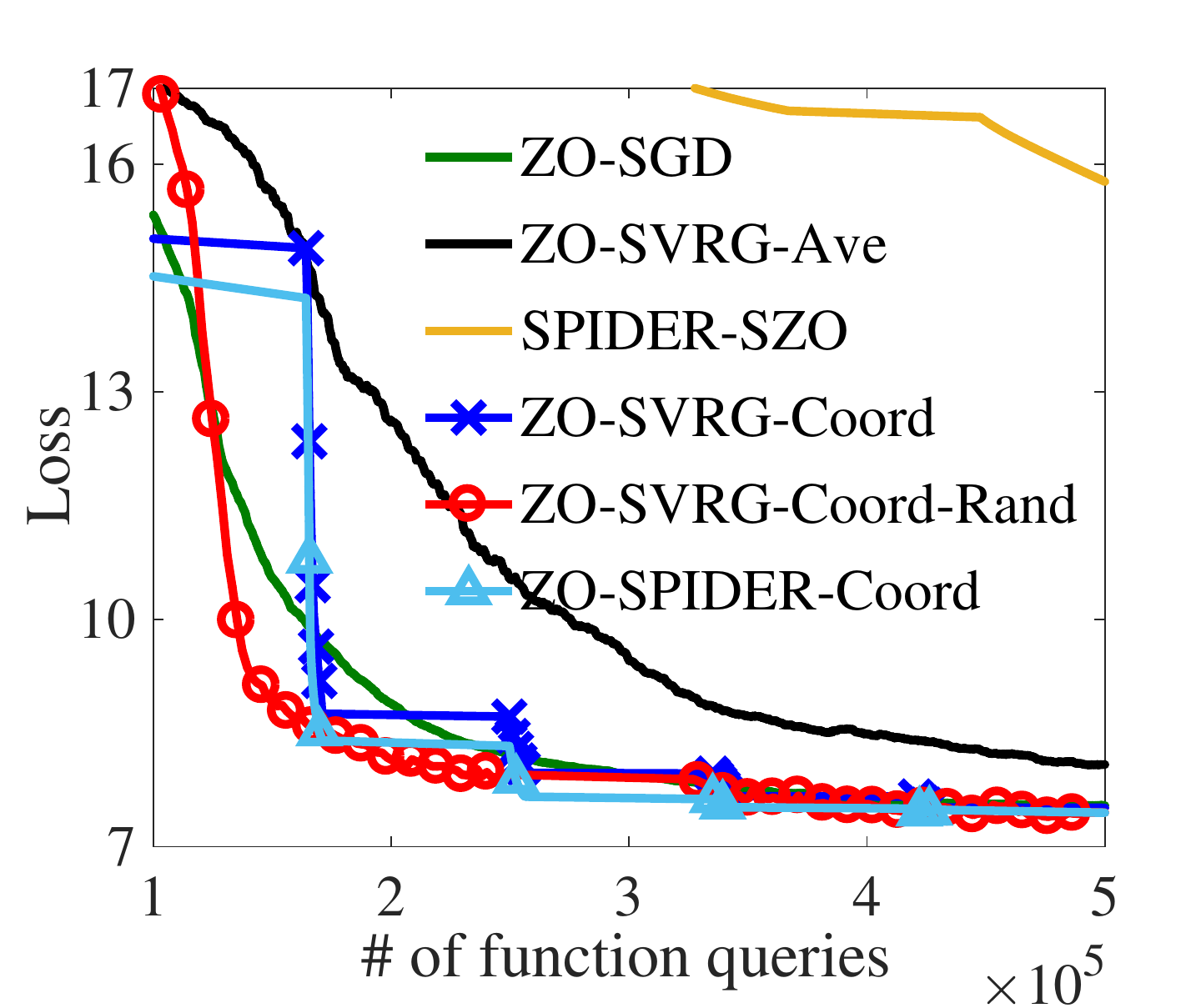}}
	\vspace{-0.1cm}
	\caption{Comparison of different zeroth-order algorithms for generating black-box adversarial examples for digit ``1'' class }\label{figure:result}
	\vspace{-0.15cm}
\end{figure*}

\begin{figure*}[t]
	\centering     
	\subfigure[Loss vs iterations: german ]{\label{fig2:a}\includegraphics[width=42mm]{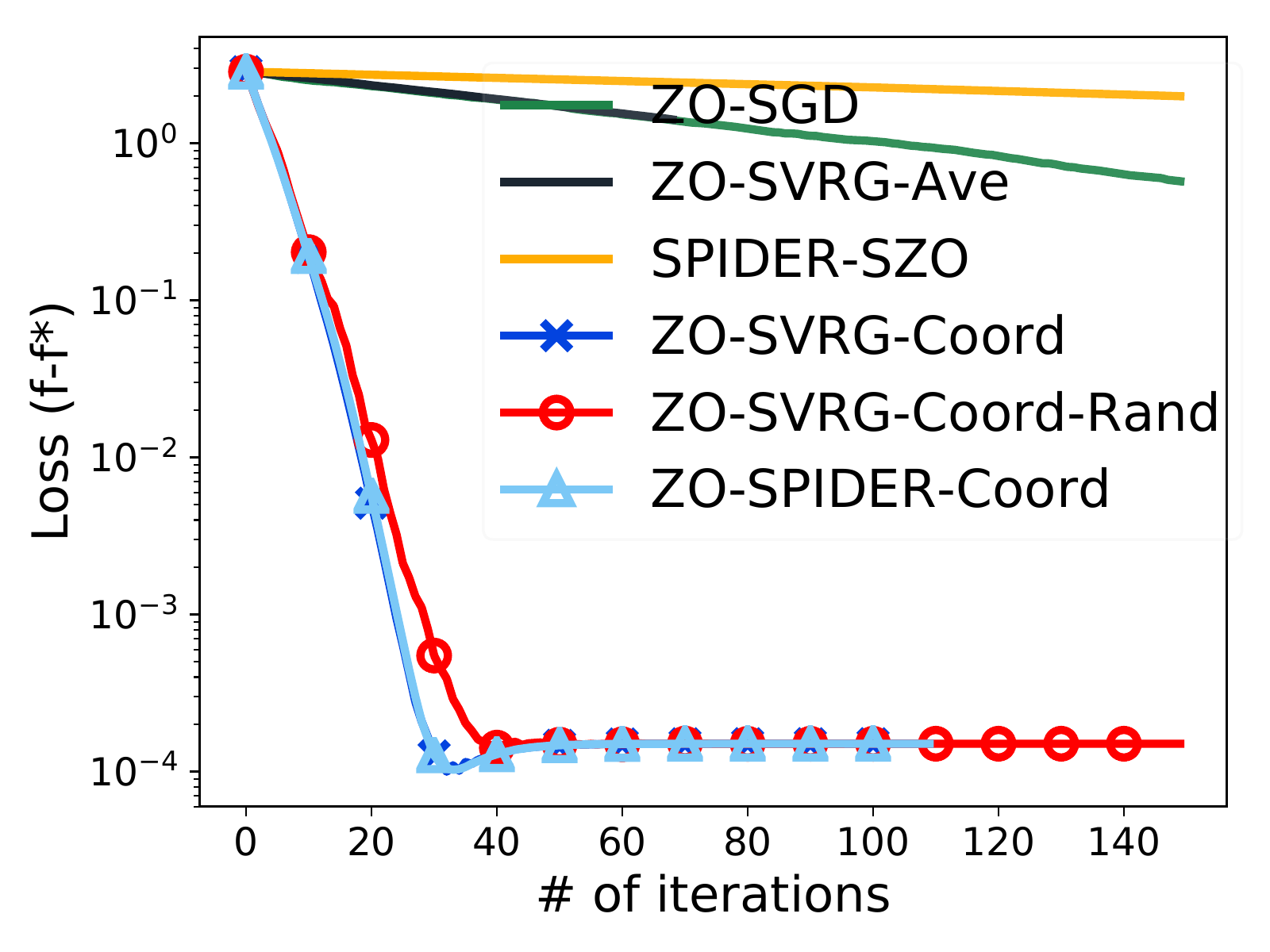}} 
	\subfigure[Loss vs queries: german ]{\label{fig2:b}\includegraphics[width=42mm]{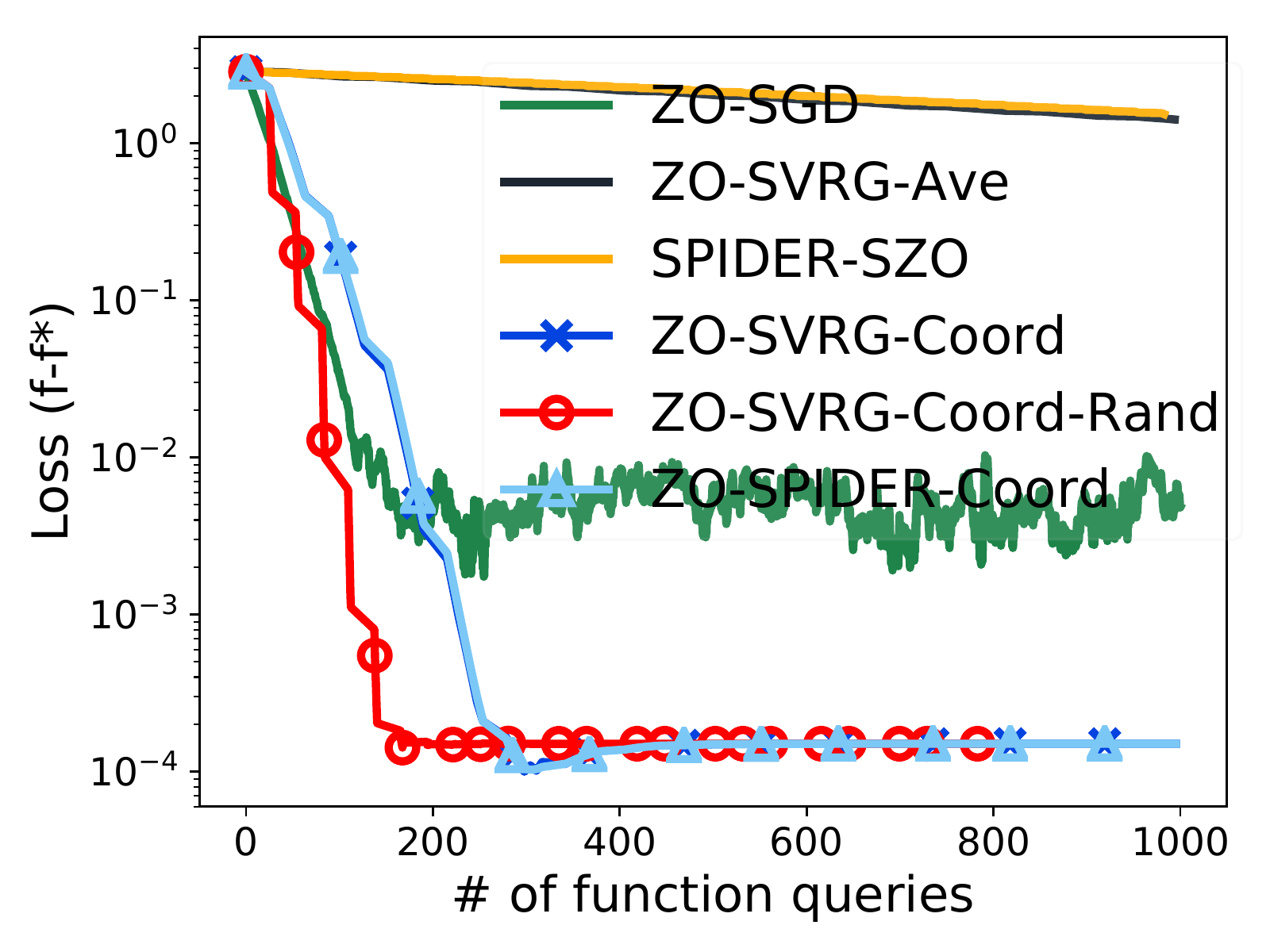}} 
	\subfigure[Loss vs iterations: ijcnn1 ]{\label{fig2:c}\includegraphics[width=42mm]{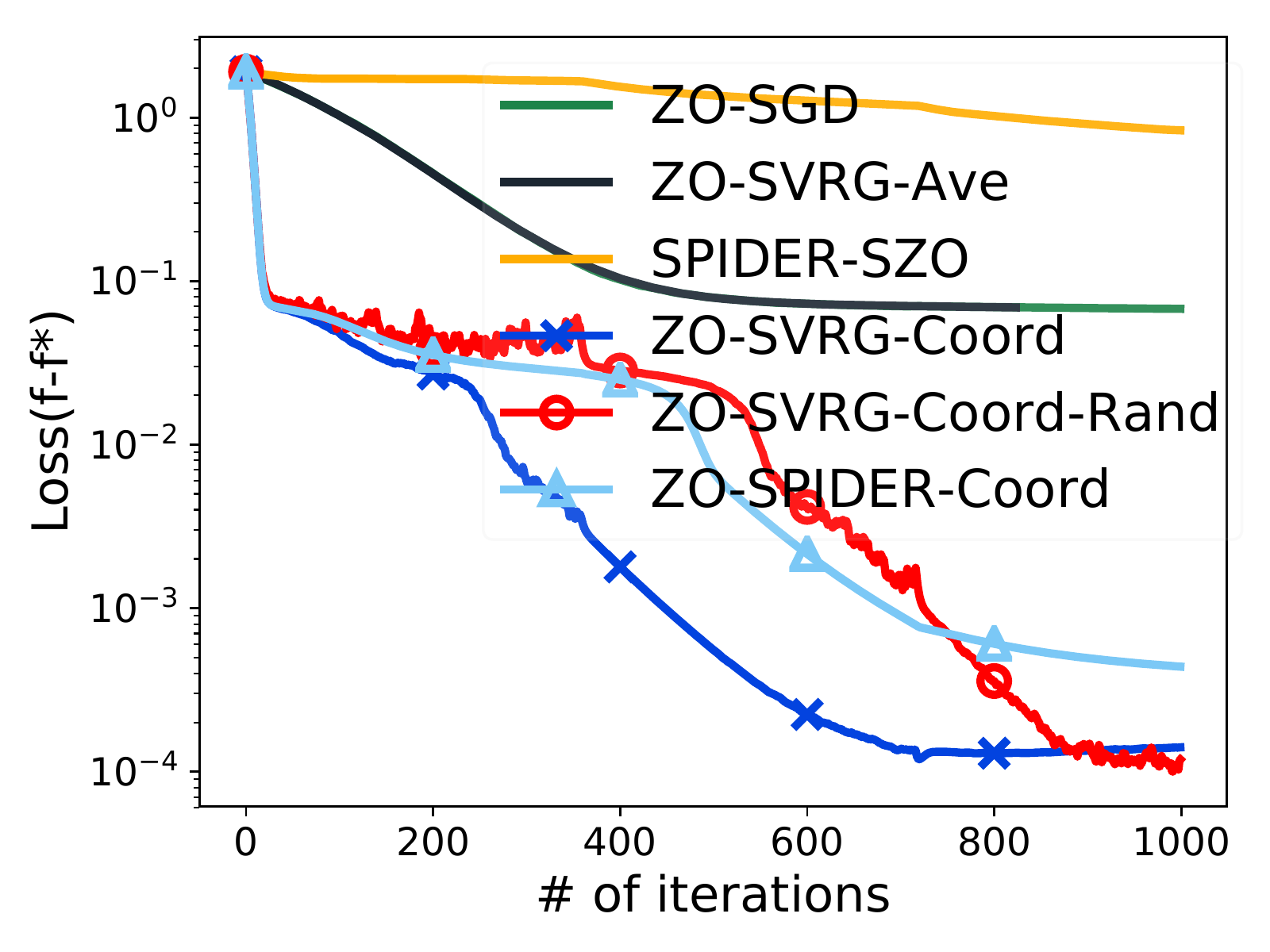}} 
	\subfigure[Loss vs  queries: ijcnn1]{\label{fig2:d}\includegraphics[width=42mm]{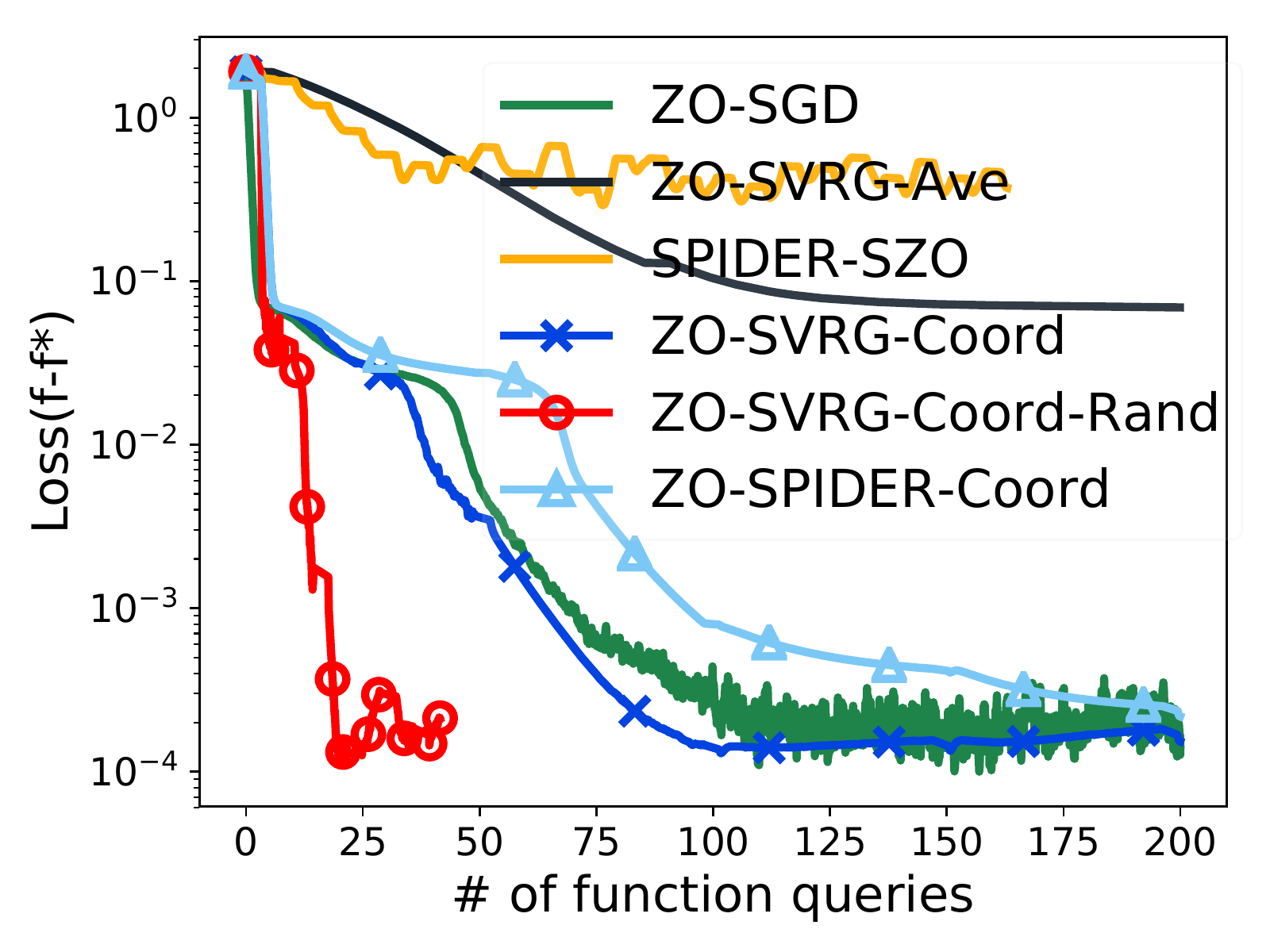}}
	\vspace{-0.15cm}
	\caption{Comparison of different zeroth-order algorithms for logistic regression problem with a nonconvex regularizer }\label{figure:result1}
	\vspace{-0.15cm}
\end{figure*}

\vspace{-0.1cm}
\subsection{Generation of Black-Box Adversarial  Examples }
In image classification, adversary attack crafts input images with imperceptive perturbation to mislead a  trained classifier. The resulting perturbed images are called adversarial examples, which are commonly  used to understand the robustness of learning models. In the black-box setting, the attacker can  access only  
the model evaluations, and hence the problem falls into the framework of zeroth-order optimization. 

We use a well-trained DNN\footnote{\fontsize{8}{8} \url {https://github.com/carlini/nn_robust_attacks}} $F(\cdot)=[F_1(\cdot),...,F_K(\cdot)]$ for the MNIST handwritten digit classification as the target black-box model, where $F_k(\cdot)$ returns the prediction score of the $k^{th}$ class. We attack a batch of $n$ correctly-classified images $\{\mathbf{a}_i\}_{i=1}^n$ from the same class, and adopt the same black-box attacking loss as in~\citealt{chen2017zoo,liu2018zeroth}. The  $i^{th}$  individual loss function $f_i(\mathbf{x})$ is given by 
\begin{align*}
f_i(\mathbf{x}) = &\max\big\{\log F_{y_i}\left(\mathbf{a}^{adv}_i\right)-\max_{t\neq y_i}\log F_t\left( \mathbf{a}_i^{adv}\right), 0\big\}
\\&+ \lambda\|\mathbf{a}^{adv}_i-\mathbf{a}_i\|^2,
\end{align*}
where  $\mathbf{a}^{adv}_i=0.5\tanh\left(\tanh^{-1}\left(2\mathbf{a}_i\right) +\mathbf{x}\right)$ is the adversarial example of the $i^{th}$ natural image $\mathbf{a}_i$, and $y_i$ is the true label of image $\mathbf{a}_i$. In our experiment, we set the regularization parameter $\lambda=1$ for  digit ``1''  image class, and set  $\lambda=0.1$ for digit ``4'' class. 

Fig.~\ref{figure:result} and Fig.~\ref{figure:result2} (in the supplementary materials) provide comparison of the performance for the algorithms of interest. Two major observations can be made. First, our proposed two algorithms ZO-SVRG-Coord-Rand and ZO-SPIDER-Coord as well as  ZO-SVRG-Coord (with the large stepsize due to our improved analysis) have much better performance both in convergence rate (iteration complexity) and function query complexity than ZO-SGD, ZO-SVRG-Ave and SPIDER-SZO. Among them,  ZO-SVRG-Coord-Rand achieves  the best performance. Second, our ZO-SPIDER-Coord  algorithm converges much faster than SPIDER-SZO in the initial optimization stage, and more importantly, has much lower function query complexity, which is largely due to the $\epsilon$-level stepsize required by SPIDER-SZO.  In addition, we present the generated adversarial examples for attacking digit ``4''  class in Table~\ref{table:digit1} in the supplementary materials,  where our ZO-SVRG-Coord-Rand  achieves the lowest  image distortion.

Interestingly, though SPIDER-based algorithms have been shown to outperfom SVRG-based algorithms in theory, our experiments suggest that SVRG-based algorithms in fact 
achieve comparable and sometimes even better performance 
in practice. The same observations have also been made  in~\citealt{fang2018spider} and~\citealt{nguyen2017sarah,nguyen2017stochastic}.

%

\vspace{-0.1cm}
\subsection{Nonconvex Logistic Regression}
In this subsection, 
we consider the following zeroth-order nonconvex logistic regression problem with two classes
$\min_{\mathbf{w}\in \mathbb{R}^d} \frac{1}{n} \sum_{i=1}^{n} \ell(\mathbf{w}^T \mathbf{x}_i, y_i)+ \alpha \sum_{i=1}^{d} \frac{ w_i^2}{1  + w_i^2 }$, 
where $\mathbf{x}_i\in \mathbb{R}^d$ denote the features, $y_i\in \{\pm 1 \}$ are the classification labels,  $\ell$ is the cross-entropy loss, and we set $\alpha=0.1$. For this problem, we use two datasets from LIBSVM \cite{chang2011libsvm}: the german dataset ($n=1000, d=24$) and the ijcnn1 dataset ($n=49990, d=22$).

As shown in Fig.~\ref{figure:result1}, ZO-SVRG-Coord-Rand, ZO-SVRG-Coord and ZO-SPIDER-Coord  converges much faster than ZO-SGD , ZO-SVRG-Ave and SPIDER-SZO in terms of  number of iterations for both datasets. 
In terms of function query complexity, ZO-SVRG-Coord converges much faster  than ZO-SVRG-Ave for both datasets and slightly faster than ZO-SGD for ijcnn1 dataset, which corroborates our new complexity analysis for ZO-SVRG-Coord. The convergence and complexity performance of ZO-SPIDER-Coord is  similar to ZO-SVRG-Coord. Among these algorithms, ZO-SVRG-Coord-Rand has the best function query complexity for both datasets. 

\vspace{-0.15cm}
\section{Conclusion}
\vspace{-0.1cm}
In this paper, we developed two  novel zeroth-order variance-reduced algorithms named ZO-SVRG-Coord-Rand  and ZO-SPIDER-Coord as well as an improved  analysis on ZO-SVRG-Coord proposed by~\citealt{liu2018zeroth}. We showed that  ZO-SVRG-Coord-Rand and ZO-SVRG-Coord (under our new analysis) outperform ZO-GD, ZO-SGD and all other existing SVRG-based zeroth-order algorithms. Furthermore, compared with SPIDER-SZO~\cite{fang2018spider}, our ZO-SPIDER-Coord allows a much larger constant stepsize and is free from the generation of a large number of Gaussian random variables while maintaining the same function query complexity. Our experiments demonstrate the superior performance of our proposed algorithms. 

\section*{Acknowledgements}
The work was supported in part by U.S. National Science Foundation under the grants CCF-1761506 and CCF-1801855.
	\bibliography{refs}

\begin{thebibliography}{40}
\providecommand{\natexlab}[1]{#1}
\providecommand{\url}[1]{\texttt{#1}}
\expandafter\ifx\csname urlstyle\endcsname\relax
  \providecommand{\doi}[1]{doi: #1}\else
  \providecommand{\doi}{doi: \begingroup \urlstyle{rm}\Url}\fi

\bibitem[Agarwal et~al.(2010)Agarwal, Dekel, and Xiao]{agarwal2010optimal}
Agarwal, A., Dekel, O., and Xiao, L.
\newblock Optimal algorithms for online convex optimization with multi-point
  bandit feedback.
\newblock In \emph{Conference on Learning Theory (COLT)}, pp.\  28--40, 2010.

\bibitem[Allen-Zhu \& Hazan(2016)Allen-Zhu and Hazan]{Allen_Zhu2016}
Allen-Zhu, Z. and Hazan, E.
\newblock Variance reduction for faster non-convex optimization.
\newblock In \emph{International Conference on Machine Learning (ICML)}, pp.\
  699--707, 2016.

\bibitem[Balasubramanian \& Ghadimi(2018)Balasubramanian and
  Ghadimi]{balasubramanian2018zeroth}
Balasubramanian, K. and Ghadimi, S.
\newblock Zeroth-order (non)-convex stochastic optimization via conditional
  gradient and gradient updates.
\newblock In \emph{Advances in Neural Information Processing Systems
  (NeurIPS)}, pp.\  3459--3468, 2018.

\bibitem[Chang \& Lin(2011)Chang and Lin]{chang2011libsvm}
Chang, C.-C. and Lin, C.-J.
\newblock Libsvm: a library for support vector machines.
\newblock \emph{ACM Transactions on Intelligent Systems and Technology},
  2\penalty0 (3):\penalty0 27, 2011.

\bibitem[Chen et~al.(2017)Chen, Zhang, Sharma, Yi, and Hsieh]{chen2017zoo}
Chen, P.-Y., Zhang, H., Sharma, Y., Yi, J., and Hsieh, C.-J.
\newblock Zoo: Zeroth order optimization based black-box attacks to deep neural
  networks without training substitute models.
\newblock In \emph{Proceedings of the 10th ACM Workshop on Artificial
  Intelligence and Security}, pp.\  15--26, 2017.

\bibitem[Choromanski et~al.(2018)Choromanski, Rowland, Sindhwani, Turner, and
  Weller]{choromanski2018structured}
Choromanski, K., Rowland, M., Sindhwani, V., Turner, R.~E., and Weller, A.
\newblock Structured evolution with compact architectures for scalable policy
  optimization.
\newblock \emph{arXiv preprint arXiv:1804.02395}, 2018.

\bibitem[Defazio et~al.(2014)Defazio, Bach, and Lacoste-Julien]{Defazio2014}
Defazio, A., Bach, F., and Lacoste-Julien, S.
\newblock {SAGA: A fast incremental gradient method with support for
  non-strongly convex composite objectives}.
\newblock In \emph{Advances in Neural Information Processing Systems (NIPS)},
  pp.\  1646--1654. 2014.

\bibitem[Duchi et~al.(2015)Duchi, Jordan, Wainwright, and
  Wibisono]{duchi2015optimal}
Duchi, J.~C., Jordan, M.~I., Wainwright, M.~J., and Wibisono, A.
\newblock Optimal rates for zero-order convex optimization: The power of two
  function evaluations.
\newblock \emph{IEEE Transactions on Information Theory}, 61\penalty0
  (5):\penalty0 2788--2806, 2015.

\bibitem[Fang et~al.(2018)Fang, Li, Lin, and Zhang]{fang2018spider}
Fang, C., Li, C.~J., Lin, Z., and Zhang, T.
\newblock Spider: Near-optimal non-convex optimization via stochastic path
  integrated differential estimator.
\newblock \emph{arXiv preprint arXiv:1807.01695}, 2018.

\bibitem[Flaxman et~al.(2005)Flaxman, Kalai, and McMahan]{flaxman2005online}
Flaxman, A.~D., Kalai, A.~T., and McMahan, H.~B.
\newblock Online convex optimization in the bandit setting: gradient descent
  without a gradient.
\newblock In \emph{Proceedings of the Sixteenth Annual ACM-SIAM Symposium on
  Discrete Algorithms}, pp.\  385--394, 2005.

\bibitem[Gao et~al.(2014)Gao, Jiang, and Zhang]{gao2014information}
Gao, X., Jiang, B., and Zhang, S.
\newblock On the information-adaptive variants of the admm: an iteration
  complexity perspective.
\newblock \emph{Journal of Scientific Computing}, pp.\  1--37, 2014.

\bibitem[Ghadimi \& Lan(2013)Ghadimi and Lan]{ghadimi2013stochastic}
Ghadimi, S. and Lan, G.
\newblock Stochastic first-and zeroth-order methods for nonconvex stochastic
  programming.
\newblock \emph{SIAM Journal on Optimization}, 23\penalty0 (4):\penalty0
  2341--2368, 2013.

\bibitem[Ghadimi et~al.(2016)Ghadimi, Lan, and Zhang]{ghadimi2016mini}
Ghadimi, S., Lan, G., and Zhang, H.
\newblock Mini-batch stochastic approximation methods for nonconvex stochastic
  composite optimization.
\newblock \emph{Mathematical Programming}, 155\penalty0 (1-2):\penalty0
  267--305, 2016.

\bibitem[Gu et~al.(2018)Gu, Huo, Deng, and Huang]{gu2018faster}
Gu, B., Huo, Z., Deng, C., and Huang, H.
\newblock Faster derivative-free stochastic algorithm for shared memory
  machines.
\newblock In \emph{International Conference on Machine Learning (ICML)}, pp.\
  1807--1816, 2018.

\bibitem[Ji et~al.(2019)Ji, Wang, Zhou, and Liang]{ji2019faster}
Ji, K., Wang, Z., Zhou, Y., and Liang, Y.
\newblock Faster stochastic algorithms via history-gradient aided batch size
  adaptation.
\newblock \emph{arXiv preprint arXiv:1910.09670}, 2019.

\bibitem[Johnson \& Zhang(2013)Johnson and Zhang]{johnson2013accelerating}
Johnson, R. and Zhang, T.
\newblock Accelerating stochastic gradient descent using predictive variance
  reduction.
\newblock In \emph{Advances in Neural Information Processing Systems (NIPS)},
  pp.\  315--323, 2013.

\bibitem[Kurakin et~al.(2016)Kurakin, Goodfellow, and
  Bengio]{kurakin2016adversarial}
Kurakin, A., Goodfellow, I., and Bengio, S.
\newblock Adversarial machine learning at scale.
\newblock \emph{arXiv preprint arXiv:1611.01236}, 2016.

\bibitem[Li \& Li(2018)Li and Li]{li2018simple}
Li, Z. and Li, J.
\newblock A simple proximal stochastic gradient method for nonsmooth nonconvex
  optimization.
\newblock \emph{arXiv preprint arXiv:1802.04477}, 2018.

\bibitem[Lian et~al.(2016)Lian, Zhang, Hsieh, Huang, and
  Liu]{lian2016comprehensive}
Lian, X., Zhang, H., Hsieh, C.-J., Huang, Y., and Liu, J.
\newblock A comprehensive linear speedup analysis for asynchronous stochastic
  parallel optimization from zeroth-order to first-order.
\newblock In \emph{Advances in Neural Information Processing Systems (NIPS)},
  pp.\  3054--3062, 2016.

\bibitem[Liu et~al.(2018{\natexlab{a}})Liu, Cheng, Hsieh, and
  Tao]{liu2018stochastic}
Liu, L., Cheng, M., Hsieh, C.-J., and Tao, D.
\newblock Stochastic zeroth-order optimization via variance reduction method.
\newblock \emph{arXiv preprint arXiv:1805.11811}, 2018{\natexlab{a}}.

\bibitem[Liu et~al.(2018{\natexlab{b}})Liu, Kailkhura, Chen, Ting, Chang, and
  Amini]{liu2018zeroth}
Liu, S., Kailkhura, B., Chen, P.-Y., Ting, P., Chang, S., and Amini, L.
\newblock Zeroth-order stochastic variance reduction for nonconvex
  optimization.
\newblock In \emph{Advances in Neural Information Processing Systems
  (NeurIPS)}, pp.\  3731--3741, 2018{\natexlab{b}}.

\bibitem[Nemirovsky \& Yudin(1983)Nemirovsky and Yudin]{nemirovsky1983problem}
Nemirovsky, A.~S. and Yudin, D.~B.
\newblock Problem complexity and method efficiency in optimization.
\newblock 1983.

\bibitem[Nesterov(2013)]{nesterov2013introductory}
Nesterov, Y.
\newblock \emph{Introductory Lectures on Convex Optimization: A Basic Course},
  volume~87.
\newblock Springer Science \& Business Media, 2013.

\bibitem[Nesterov \& Spokoiny(2011)Nesterov and Spokoiny]{nesterov2011random}
Nesterov, Y. and Spokoiny, V.
\newblock Random gradient-free minimization of convex functions.
\newblock Technical report, Universit{\'e} catholique de Louvain, Center for
  Operations Research and Econometrics (CORE), 2011.

\bibitem[Nguyen et~al.(2017{\natexlab{a}})Nguyen, Liu, Scheinberg, and
  Tak{\'a}{\v{c}}]{nguyen2017sarah}
Nguyen, L.~M., Liu, J., Scheinberg, K., and Tak{\'a}{\v{c}}, M.
\newblock Sarah: A novel method for machine learning problems using stochastic
  recursive gradient.
\newblock In \emph{International Conference on Machine Learning (ICML)}, pp.\
  2613--2621, 2017{\natexlab{a}}.

\bibitem[Nguyen et~al.(2017{\natexlab{b}})Nguyen, Liu, Scheinberg, and
  Tak{\'a}{\v{c}}]{nguyen2017stochastic}
Nguyen, L.~M., Liu, J., Scheinberg, K., and Tak{\'a}{\v{c}}, M.
\newblock Stochastic recursive gradient algorithm for nonconvex optimization.
\newblock \emph{arXiv preprint arXiv:1705.07261}, 2017{\natexlab{b}}.

\bibitem[Papernot et~al.(2017)Papernot, McDaniel, Goodfellow, Jha, Celik, and
  Swami]{papernot2017practical}
Papernot, N., McDaniel, P., Goodfellow, I., Jha, S., Celik, Z.~B., and Swami,
  A.
\newblock Practical black-box attacks against machine learning.
\newblock In \emph{Proceedings of the 2017 ACM on Asia Conference on Computer
  and Communications Security}, pp.\  506--519. ACM, 2017.

\bibitem[Polyak(1963)]{polyak1963gradient}
Polyak, B.~T.
\newblock Gradient methods for the minimisation of functionals.
\newblock \emph{USSR Computational Mathematics and Mathematical Physics},
  3\penalty0 (4):\penalty0 864--878, 1963.

\bibitem[Reddi et~al.(2016{\natexlab{a}})Reddi, Hefny, Sra, Poczos, and
  Smola]{reddi2016stochastic}
Reddi, S.~J., Hefny, A., Sra, S., Poczos, B., and Smola, A.
\newblock Stochastic variance reduction for nonconvex optimization.
\newblock In \emph{International Conference on Machine Learning (ICML)}, pp.\
  314--323, 2016{\natexlab{a}}.

\bibitem[Reddi et~al.(2016{\natexlab{b}})Reddi, Sra, Poczos, and
  Smola]{Reddi2016}
Reddi, S.~J., Sra, S., Poczos, B., and Smola, A.
\newblock Proximal stochastic methods for nonsmooth nonconvex finite-sum
  optimization.
\newblock In \emph{Advances in Neural Information Processing Systems (NIPS)},
  pp.\  1145--1153. 2016{\natexlab{b}}.

\bibitem[Robbins \& Monro(1951)Robbins and Monro]{robbins1951}
Robbins, H. and Monro, S.
\newblock A stochastic approximation method.
\newblock \emph{The Annals of Mathematical Statistics}, 22\penalty0
  (3):\penalty0 400--407, 09 1951.

\bibitem[Roux et~al.(2012)Roux, Schmidt, and Bach]{Nicolas2012}
Roux, N.~L., Schmidt, M., and Bach, F.~R.
\newblock A stochastic gradient method with an exponential convergence rate for
  finite training sets.
\newblock In \emph{Advances in Neural Information Processing Systems (NIPS)},
  pp.\  2663--2671. 2012.

\bibitem[Shamir(2013)]{shamir2013complexity}
Shamir, O.
\newblock On the complexity of bandit and derivative-free stochastic convex
  optimization.
\newblock In \emph{Conference on Learning Theory (COLT)}, pp.\  3--24, 2013.

\bibitem[Taskar et~al.(2005)Taskar, Chatalbashev, Koller, and
  Guestrin]{taskar2005learning}
Taskar, B., Chatalbashev, V., Koller, D., and Guestrin, C.
\newblock Learning structured prediction models: A large margin approach.
\newblock In \emph{International Conference on Machine Learning (ICML)}, pp.\
  896--903, 2005.

\bibitem[Wainwright et~al.(2008)Wainwright, Jordan,
  et~al.]{wainwright2008graphical}
Wainwright, M.~J., Jordan, M.~I., et~al.
\newblock Graphical models, exponential families, and variational inference.
\newblock \emph{Foundations and Trends{\textregistered} in Machine Learning},
  1\penalty0 (1--2):\penalty0 1--305, 2008.

\bibitem[Wang et~al.(2017)Wang, Du, Balakrishnan, and
  Singh]{wang2017stochastic}
Wang, Y., Du, S., Balakrishnan, S., and Singh, A.
\newblock Stochastic zeroth-order optimization in high dimensions.
\newblock \emph{arXiv preprint arXiv:1710.10551}, 2017.

\bibitem[Wang et~al.(2018)Wang, Ji, Zhou, Liang, and
  Tarokh]{wang2018spiderboost}
Wang, Z., Ji, K., Zhou, Y., Liang, Y., and Tarokh, V.
\newblock Spiderboost: A class of faster variance-reduced algorithms for
  nonconvex optimization.
\newblock \emph{arXiv preprint arXiv:1810.10690}, 2018.

\bibitem[Zhong et~al.(2017)Zhong, Song, Jain, Bartlett, and
  Dhillon]{zhong2017recovery}
Zhong, K., Song, Z., Jain, P., Bartlett, P.~L., and Dhillon, I.~S.
\newblock Recovery guarantees for one-hidden-layer neural networks.
\newblock In \emph{International Conference on Machine Learning (ICML)}, pp.\
  4140--4149, 2017.

\bibitem[Zhou et~al.(2018)Zhou, Xu, and Gu]{zhou2018stochastic}
Zhou, D., Xu, P., and Gu, Q.
\newblock Stochastic nested variance reduced gradient descent for nonconvex
  optimization.
\newblock In \emph{Advances in Neural Information Processing Systems
  (NeurIPS)}, pp.\  3921--3932, 2018.

\bibitem[Zhou et~al.(2016)Zhou, Zhang, and Liang]{zhou2016geometrical}
Zhou, Y., Zhang, H., and Liang, Y.
\newblock Geometrical properties and accelerated gradient solvers of non-convex
  phase retrieval.
\newblock In \emph{54th Annual Allerton Conference on Communication, Control,
  and Computing (Allerton)}, pp.\  331--335, 2016.

\end{thebibliography}
	\bibliographystyle{icml2019}
	
\newpage
\onecolumn 
\appendix

{\Large{\bf Supplementary Materials}}

\section{Further Specification of Experiments and Additional  Results}

\subsection{Generation of Black-Box Adversarial Examples }
{\bf Parameter selection for algorithms under comparison.} For ZO-SGD and ZO-SVRG-Ave, we adopt the  implementations\footnote{\url{https://github.com/IBM/ZOSVRG-BlackBox-Adv}} from~\citealt{liu2018zeroth}. As recommended by~\citealt{liu2018zeroth}, we set the epoch length $q=10$ for ZO-SVRG-Ave, and select the mini-batch size $|\mathcal{S}_2|$ from $\{5, 10,50\}$ and the stepsize $\eta$ from $\{1,10,20,30,40\}/d$ for both ZO-SGD and ZO-SVRG-Ave, and we present the best performance among these parameters, 
where $d=28\times 28$ is the input dimension. 
For SPIDER-SZO, we set the parameters by 
Theorem 8 in~\citealt{fang2018spider}. Namely, we  choose the epoch length $q$ from $\{30, 50,80\}$, mini-batch size $|\mathcal{S}_2|$ from $\{5, 80, 700\}$, and $\eta$ from $\{0.1,0.01\}/\|\mathbf{v}^k\|$, and we present the best performance among these parameters. The parameters chosen for our ZO-SVRG-Coord-Rand, ZO-SVRG-Coord (based on our new analysis, which allows a larger stepsize with performance guarantee) and ZO-SPIDER-Coord are listed in Table~\ref{niubiss}.  
For all algorithms, we choose $|\mathcal{S}_1|=n$, and set the smoothing parameters $\beta=0.01$ and $\delta=0.001$.

\begin{figure}[h]
	\centering     
	\subfigure[Loss vs iterations: $n=10$] {\label{fig4:a}\includegraphics[width=55mm]{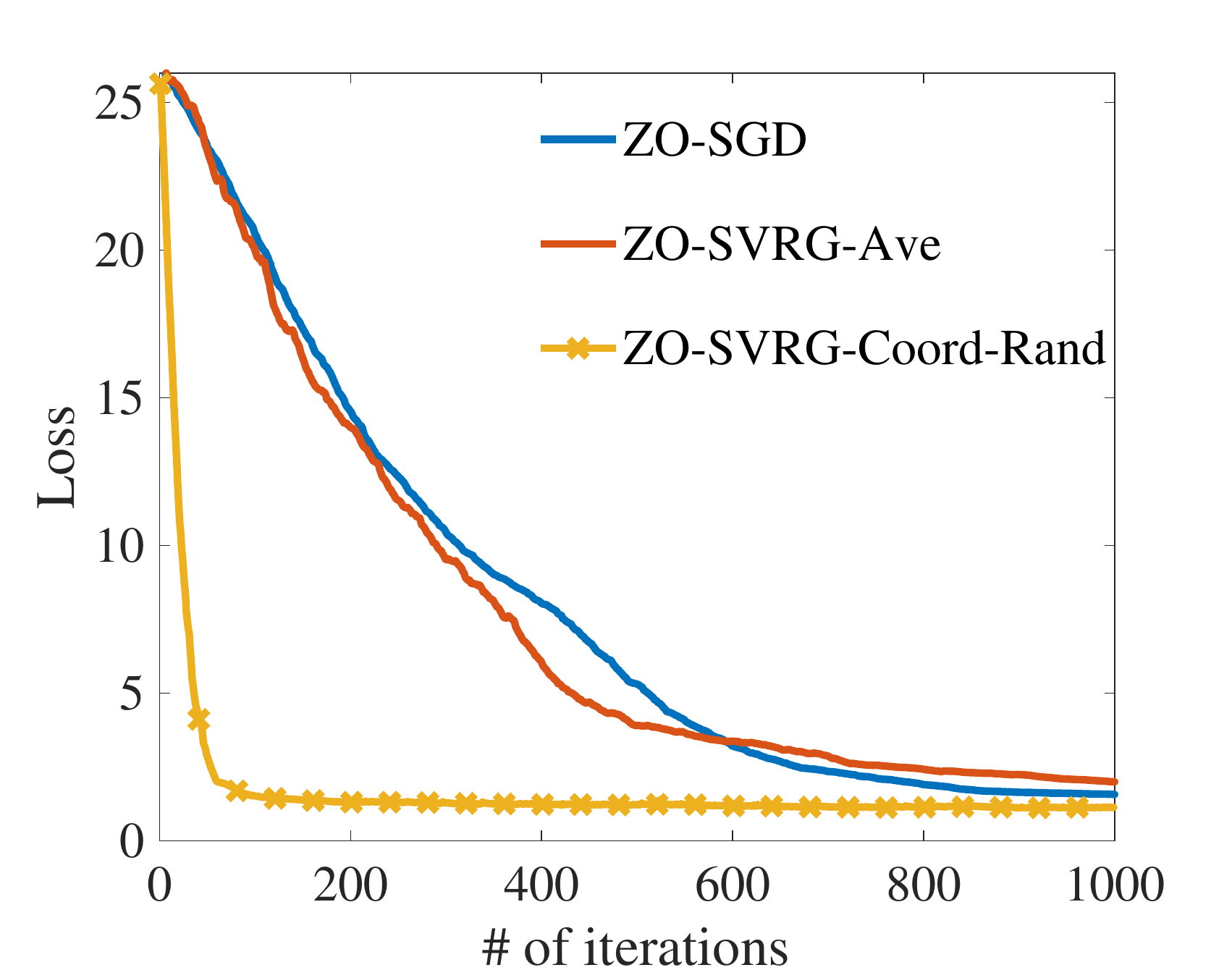}}
	\subfigure[Loss vs queries: $n=10$]{\label{fig4:b}\includegraphics[width=55mm]{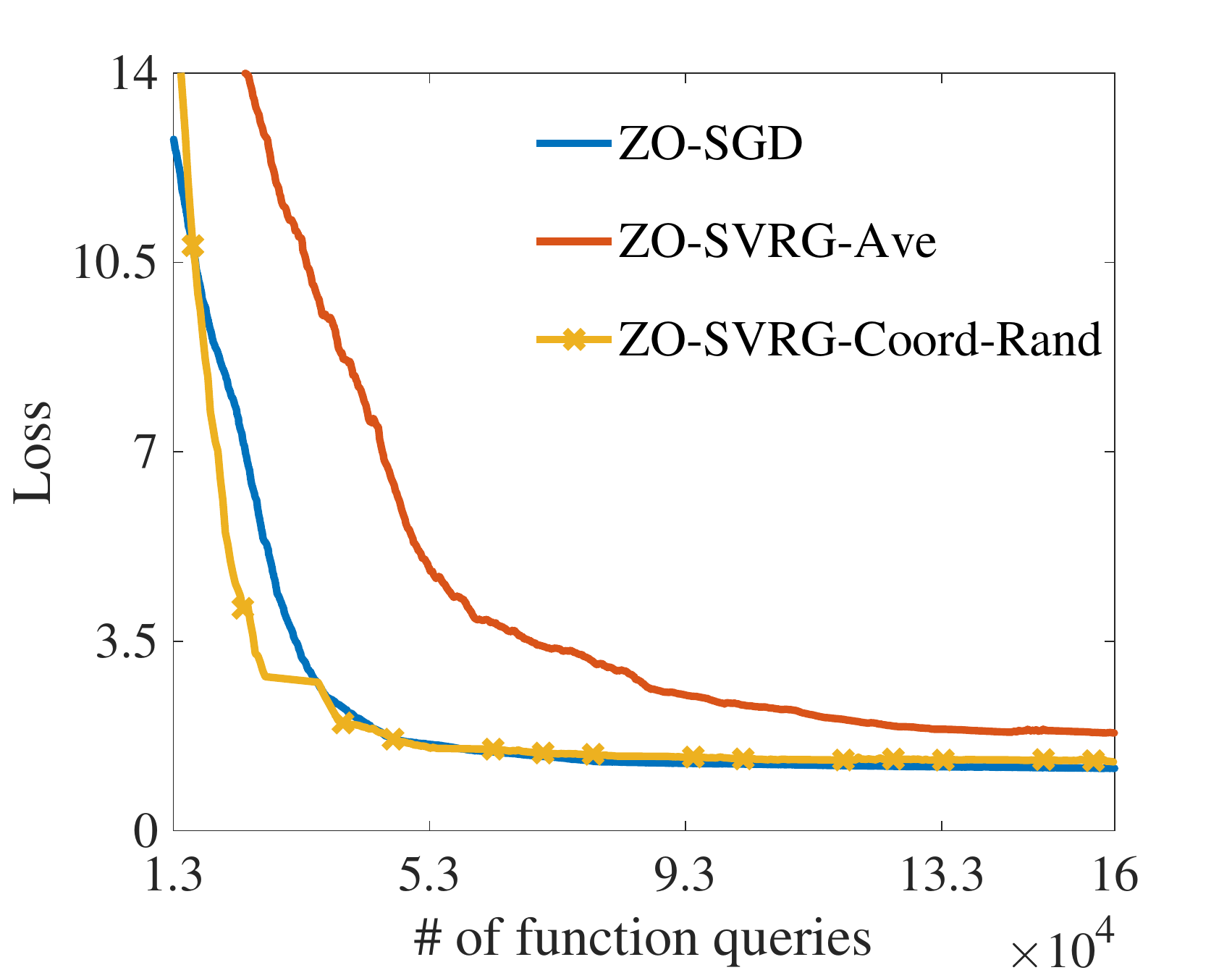}}
	\caption{Comparison of different zeroth-order algorithms for generating black-box adversarial  examples for digit ``4'' class}\label{figure:result2}
\end{figure}
\vspace{-0.5cm}
\begin{table*}[ht]
	\small
	\caption{Generated adversarial examples  for digit ``4''  class, where image distortion is defined as $\frac{1}{n}\sum_{i=1}^n \|\mathbf{a}^{adv}_i-\mathbf{a}_i\|^2$.}
	\vspace{0.2cm}
	\begin{adjustbox}{max width=\textwidth }
		\begin{tabular}
			{cccccccccccc}
			\hline
			Image ID & $4$& $6$ & $19$ & $24$ & $27$ & $33$ & $42$ & $48$ & $49$ & $56$ & Image  distortion \\
			\hline &&&&&&&&&& \vspace{-0.2cm} \\
			ZO-SGD &
			\parbox[c]{2.2em}{\includegraphics[width=0.4in]{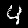}} &
			\parbox[c]{2.2em}{\includegraphics[width=0.4in]{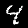}} &
			\parbox[c]{2.2em}{\includegraphics[width=0.4in]{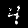}} &
			\parbox[c]{2.2em}{\includegraphics[width=0.4in]{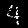}} &
			\parbox[c]{2.2em}{\includegraphics[width=0.4in]{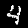}} &
			\parbox[c]{2.2em}{\includegraphics[width=0.4in]{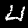}} &
			\parbox[c]{2.2em}{\includegraphics[width=0.4in]{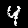}} &
			\parbox[c]{2.2em}{\includegraphics[width=0.4in]{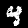}} &
			\parbox[c]{2.2em}{\includegraphics[width=0.4in]{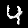}} &
			\parbox[c]{2.2em}{\includegraphics[width=0.4in]{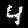}} &$11.46$\vspace{0.1cm}\\
			\hline
			Classified as & $\bf 9$ & $\bf 8$ & $\bf 1$ & $\bf 3$ & $\bf 2$ & $\bf 2$ & $\bf 9$ & $\bf 9$ & $\bf 9$ & $\bf 9$& \\
			
			\hline &&&&&&&&&& \vspace{-0.2cm} \\
			ZO-SVRG-Ave&
			\parbox[c]{2.2em}{\includegraphics[width=0.4in]{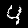}} &
			\parbox[c]{2.2em}{\includegraphics[width=0.4in]{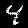}} &
			\parbox[c]{2.2em}{\includegraphics[width=0.4in]{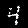}} &
			\parbox[c]{2.2em}{\includegraphics[width=0.4in]{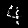}} &
			\parbox[c]{2.2em}{\includegraphics[width=0.4in]{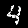}} &
			\parbox[c]{2.2em}{\includegraphics[width=0.4in]{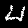}} &
			\parbox[c]{2.2em}{\includegraphics[width=0.4in]{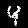}} &
			\parbox[c]{2.2em}{\includegraphics[width=0.4in]{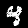}} &
			\parbox[c]{2.2em}{\includegraphics[width=0.4in]{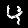}} &
			\parbox[c]{2.2em}{\includegraphics[width=0.4in]{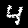}} & $13.85$\vspace{0.1cm}\\
			\hline
			Classified as & $\bf 9$& $\bf 8$ & $\bf 2$ & $\bf 3$ & $\bf 2$ & $\bf 2$ & $\bf 9$ & $\bf 9$ & $\bf 9$ & $\bf 3$ & \\
			
			\hline &&&&&&&&&& \vspace{-0.2cm} \\
			ZO-SVRG-Coord-Rand &
			\parbox[c]{2.2em}{\includegraphics[width=0.4in]{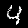}} &
			\parbox[c]{2.2em}{\includegraphics[width=0.4in]{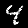}} &
			\parbox[c]{2.2em}{\includegraphics[width=0.4in]{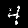}} &
			\parbox[c]{2.2em}{\includegraphics[width=0.4in]{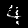}} &
			\parbox[c]{2.2em}{\includegraphics[width=0.4in]{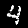}} &
			\parbox[c]{2.2em}{\includegraphics[width=0.4in]{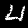}} &
			\parbox[c]{2.2em}{\includegraphics[width=0.4in]{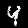}} &
			\parbox[c]{2.2em}{\includegraphics[width=0.4in]{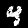}} &
			\parbox[c]{2.2em}{\includegraphics[width=0.4in]{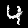}} &
			\parbox[c]{2.2em}{\includegraphics[width=0.4in]{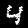}} & $11.21$\vspace{0.1cm}\\
			\hline
			Classified as & $\bf 9$ & $\bf 8$ & $\bf 2$ & $\bf 3$ & $\bf 2$ & $\bf 2$ & $\bf 9$ & $\bf 9$ & $\bf 9$ & $\bf 9$ \\
			\hline
		\end{tabular}
	\end{adjustbox}
	\label{table:digit1}
\end{table*}
\begin{table*} [h] 
	\centering  
	\caption{Parameter settings for ZO-SVRG-Coord-Rand  (left), ZO-SVRG-Coord (middle) and ZO-SPIDER-Coord (right).
	}
	\scalebox{0.9}{
		\subtable{  
			\begin{tabular}{ccc} \toprule
				{Parameters} &$n=10$  & $n=100$  \\   \midrule 
				$q$ &$50$  & $80$ \par    
				\\  \midrule 
				$|\mathcal{S}_2|$ &$80$  &  $700$\par   
				\\  \midrule 
				{$\eta$} &$0.102$  & $0.663$ \par \\ \midrule
			\end{tabular} 
			\hspace{0.4cm}
		}  
		
		\subtable{  
			\begin{tabular}{ccc} \toprule
				{Parameters} &$n=10$  & $n=100$  \\   \midrule 
				$q$ &$30$  & $50$ \par    
				\\  \midrule 
				$|\mathcal{S}_2|$ &$5$  &  $70$\par   
				\\  \midrule 
				{$\eta$} &$0.102$  & $0.255$ \par \\ \midrule
			\end{tabular}		
			\hspace{0.4cm}
		}  
		
		\subtable{          
			\begin{tabular}{ccc} \toprule
				{Parameters} &$n=10$  & $n=100$  \\   \midrule 
				$q$ &$30$  & $50$ \par    
				\\  \midrule 
				$|\mathcal{S}_2|$ &$5$  &  $70$\par   
				\\  \midrule 
				{$\eta$} &$0.064$  & $0.255$ \par \\ \midrule
			\end{tabular} 
		}
	}  
	\label{niubiss}
\end{table*}

\subsection{Nonconvex logistic regression}
{\bf Parameter selection for algorithms under comparison.} For all algorithms, we choose fixed mini-batch sizes $|\mathcal{S}_1|=n$ and $|\mathcal{S}_2|=128$,  the epoch length $q = n/128$ for german dataset, and choose fixed mini-batch sizes $|\mathcal{S}_1|=50*256$  and $|\mathcal{S}_2|=256$,  the epoch length $q=n/256$ for ijcnn1 dataset.  In addition, we set the learning rate for all algorithms according to their convergence guarantee. In specific, we choose  $\eta=0.8$ for ZO-SVRG-Coord-Rand, ZO-SPIDER-Coord, ZO-SVRG-Coord, and choose  $\eta=0.8/d$ for ZO-SGD,  ZO-SVRG-Ave, and set $\eta=0.8\sqrt{\epsilon}/\|v_k\|$ for SPIDER-SZO, as specified in~\citealt{fang2018spider}.   

\section{Zeroth-Order Nonconvex  Nonsmooth  Composite Optimization}
Zeroth-order optimization has been studied for nonconvex and nonsmooth objective function in~\citep{ghadimi2016mini}, where a zeroth-order stochastic algorithm named RSPGF has been proposed. Here, we propose a zeroth-order stochastic variance-reduced algorithm for the same objective function, and show that it order-wisely outperforms RSPGF. 
\subsection{PROX-ZO-SPIDER-Coord for Composite Optimization}\label{sarahnons}
In this subsection, we extend our study of  ZO-SPIDER-Coord to the following nonconvex and nonsmooth composite  problem
\begin{align}\label{compo}
\min_{\mathbf{x}\in\mathbb{R}^d} \Psi(\mathbf{x}) =f(\mathbf{x})+h(\mathbf{x}),  f(x)=\frac{1}{n}\sum_{i=1}^n f_i(\mathbf{x}),
\end{align}
where each $f_i(\mathbf{x})$ is smooth and  nonconvex, $h(\mathbf{x})$ is a   nonsmooth convex function ( e.g., $h(\mathbf{x})=\lambda\|\mathbf{x}\|_1,\lambda>0$). To address the nonsmooth term $h(\mathbf{x})$ in the objective function~\eqref{compo}, 
we propose PROX-ZO-SPIDER-Coord algorithm, which replaces line 8 in Algorithm~\ref{ours:3}  by 
a proximal gradient step
\begin{align*}
\quad\mathbf{x}^{k+1}=\arg\min_{\mathbf{x}\in\mathbb{R}^d}\big\{ \langle \mathbf{v}^k,\mathbf{x}\rangle  +\frac{1}{2\eta} \|\mathbf{x}-\mathbf{x}^k\|^2+h(\mathbf{x}) \big\}.
\end{align*} 
Similarly to~\citealt{ghadimi2016mini}, we define 
\begin{align}\label{GG}
G(\mathbf{x},\nabla f(\mathbf{x}),\eta)=\frac{1}{\eta}\left(\mathbf{x}-\mathbf{x}^+\right)
\end{align}
as a generalized projected gradient of $\Psi(\cdot)$ at the point $\mathbf{x}$ and use it to characterize the convergence criterion, where the point $\mathbf{x}^+$ is given by the proximal mapping 
\begin{align*}
\mathbf{x}^+=\arg\min_{\mathbf{z}\in\mathbb{R}^d}\left\{ \langle \nabla f(\mathbf{x}),\mathbf{z}\rangle  +\frac{1}{2\eta} \|\mathbf{z}-\mathbf{x}\|^2+h(\mathbf{z}) \right\}.
\end{align*}
Based on the above notations, we provide the following convergence guarantee for  PROX-ZO-SPIDER-Coord. 
\begin{Theorem}\label{coprox}
	Let Assumption~\ref{assumption} hold,  and we choose  the same parameters as in Corollary~\ref{comain}.  Then our PROX-ZO-SPIDER-Coord satisfies 
	$\mathbb{E}\|G(\mathbf{x}^\zeta,\nabla f(\mathbf{x}^\zeta),\eta) \|^2\leq(60\Delta_\psi L+80+69\sigma^2)/K+138/K^2 $, 
	where 
	$0<\Delta_\psi=\psi(\mathbf{x}^{0})-\psi(\mathbf{x}^{*})<\infty$ and  $ \mathbf{x}^*=\arg\min_{\mathbf{x}\in\mathbb{R}^d}\psi(\mathbf{x})$.
	
	To achieve $\mathbb{E}\|G(\mathbf{x}^\zeta,\nabla f(\mathbf{x}^\zeta),\eta)\|^2\leq \epsilon$, the number of function queries is at most  $\mathcal{O}\left(\min\left\{n^{1/2}d\epsilon^{-1}, d\epsilon^{-3/2} \right\}\right)$.
	
\end{Theorem}
Let us compare our  PROX-ZO-SPIDER-Coord algorithm with the randomized stochastic projected gradient free algorithm RSPGF, introduced by~\citealt{ghadimi2016mini}.  Casting Corollary 8 in~\citealt{ghadimi2016mini} to the setting of our Theorem~\ref{coprox} yields 
$\mathbb{E}\|G(\mathbf{x}^\zeta,\nabla f(\mathbf{x}^\zeta),\eta)\|^2\leq \mathcal{O}\left( \frac{d}{\widetilde K} + \frac{\sqrt{d}}{\sqrt{\widetilde K}} \right)$, 
where $\widetilde K$ is the total number of function queries. Thus, RSPGF requires at most $\mathcal{O}(d/\epsilon^2)$ function queries to achieve  $\mathbb{E}\|G(\mathbf{x}^\zeta,\nabla f(\mathbf{x}^\zeta),\eta)\|^2\leq \epsilon<1$. As a comparison, the function query complexity of  PROX-ZO-SPIDER-Coord outperforms  that of RSPGF~\cite{ghadimi2016mini} by a factor of $\mathcal{O}(\max\{\epsilon^{-1}n^{-1/2}, \epsilon^{-1/2}\})$. 

\section{Zeroth-Order Variance-Reduced Algorithms for Convex Optimization}

In this paper, we have proposed two new zeroth-order variance-reduced algorithms ZO-SVRG-Coord-Rand and ZO-SPIDER-Coord, and have studied their performance for nonconvex optimization. 
In this section, we study the performance of these two algorithms for convex optimization, where each individual function $f_i(\cdot)$ is convex. We note that there was no proven convergence guarantee for previously proposed zeroth-order SVRG-based and SPIDER-based algorithms for convex optimization. 
\subsection{ZO-SVRG-Coord-Rand-C Algorithm }
In this subsection, we explore the convergence performance of ZO-SVRG-Coord-Rand for convex optimization. To fully utilize the convexity of the objective function, we  
propose a variant of our ZO-SVRG-Coord-Rand, which we refer to as ZO-SVRG-Coord-Rand-C.
Differently from ZO-SVRG-Coord-Rand, the outer-loop iteration (i.e., $k \mod q = 0$) of ZO-SVRG-Coord-Rand-C 
chooses $\mathbf{x}^k$ from $\{\mathbf{x}^{k-q},...,\mathbf{x}^{k-1}\}$ uniformly at random, which is a typical treatment used in convex first-order optimization~\citep{reddi2016stochastic,nguyen2017sarah}. 
In the meanwhile, the inner-loop iteration of ZO-SVRG-Coord-Rand-C  is the same as single-sample ZO-SVRG-Coord-Rand, which 
computes $\mathbf{v}^k=\hat \nabla_{\text{rand}} f_{i_k}(\mathbf{x}^k; \mathbf{u}^k)-\hat \nabla_{\text{rand}} f_{i_k}(\mathbf{x}^{qk_0}; \mathbf{u}^k)+\mathbf{v}^{qk_0}$ with a single sample $i_k$ drawn from $[n]$ and a smoothing vector $\mathbf{u}^k$  drawn from the uniform distribution over the unit sphere. .

The following theorem  provides the function query complexity for ZO-SVRG-Coord-Rand-C.

\begin{Theorem}\label{coco:convex}
	Under Assumption~\ref{assumption}, let   $\eta=1/(27dL), \beta=\epsilon/(c_\beta dL), \delta=\epsilon/(c_\delta \sqrt{d}L), q=c_qd/ \epsilon$, $h=\log_2 (c_h/\epsilon))$ and $|\mathcal{S}|=\min\{n,\lceil c_s /\epsilon\rceil \}$, where $c_q, c_h, c_\beta,c_\delta$ and $c_s$ are  sufficiently large  positive constants.  Then, to achieve an $\epsilon$-accuracy solution, i.e.,  $\mathbb{E}(f(\mathbf{x}^K)-f(\mathbf{x}^*) )\leq \epsilon$, the number of function queries required by ZO-SVRG-Coord-Rand-C algorithm is at most $\mathcal{O}(d\min\{n,1/\epsilon\}\log(1/\epsilon))$.
\end{Theorem}
Let us compare our  result with that of ZO-SGD given by~\citealt{ghadimi2013stochastic}. Casting  Corollary 3.3 in~\citealt{ghadimi2013stochastic} under the setting of our Corollary~\ref{coco:convex} implies that the function query complexity of  ZO-SGD is $\mathcal{O}(d/\epsilon^{2})$, which is worse than that of  our ZO-SVRG-Coord-Rand-C by a factor of $ \mathcal{\widetilde O}(\max\{\epsilon^{-2}n^{-1},\epsilon^{-1}\})$.

\subsection{ZO-SPIDER-Coord-C Algorithm} 
In this subsection, we generalize our ZO-SPIDER-Coord to solving  convex optimization problem, and proposes the  ZO-SPIDER-Coord-C algorithm. 
ZO-SPIDER-Coord-C has the same outer-loop iteration as ZO-SVRG-Coord-Rand-C, but 
updates $\mathbf{v}^k$ in a different way  by
$\mathbf{v}^k=\hat \nabla _{\text{coord}} f_{i_k}(\mathbf{x}^k)- \hat \nabla _{\text{coord}} f_{i_k}(\mathbf{x}^{k-1})+
\mathbf{v}^{k-1}$ at each inner-loop iteration. 


Based on  Lemma~\ref{le:coord}, we obtain the following 
complexity result for ZO-SPIDER-Coord-C. 
\begin{Theorem}\label{co:sarass}
	Under Assumption~\ref{assumption}, let   $\eta=1/(24L), q=c_q/ \epsilon$, $h=\log_2 (c_h/\epsilon)), \delta=\epsilon/( c_q\sqrt{d}L)$ and $|\mathcal{S}|=\min\{n,\lceil c_s /\epsilon\rceil \}$, where $c_q, c_h$ and $c_s$ are  sufficiently large  positive constants.  Then, to achieve an $\epsilon$-accuracy solution, i.e.,  $\mathbb{E}\|\nabla f(\mathbf{x}^{K})\|^2\leq \epsilon$, the number of function queries required by ZO-SPIDER-Coord-C is at most $\mathcal{O}(d\min\{n,1/\epsilon\}\log(1/\epsilon))$
\end{Theorem}
Note that ZO-SPIDER-Coord-C achieves the same function query complexity as  that of ZO-SVRG-Coord-Rand-C, and  improves that of ZO-SGD~\citep{ghadimi2013stochastic} by a factor of  $ \mathcal{\widetilde O}(\max\{\epsilon^{-2}n^{-1},\epsilon^{-1}\})$ w.r.t.   stationary gap $\mathbb{E}\|\nabla f(\mathbf{x}^{K})\|^2$. The detailed comparison among our algorithms and other exiting algorithms is summarized in Table~\ref{tas2}.

\renewcommand{\arraystretch}{1.3} 
\definecolor{LightCyan}{rgb}{0.9,1,0.9}
\begin{table*}[!h] 
	\caption{Comparison of zeroth-order algorithms in terms of the function query complexity for convex  optimization.
	}
	\vspace{0.2cm}
	\small
	\centering 
	\begin{tabular}{llcc} \toprule
		{Algorithms} &&Function query complexity  & Function value convergence  \\   \midrule
		ZO-SGD &\citep{ghadimi2013stochastic}&$\mathcal{O}\left( \frac{d}{\epsilon^{2}}\right)$  &  \cmark \par    
		\\  \midrule 
		ZSCG &\citep{balasubramanian2018zeroth}&$\mathcal{O}\left( \frac{d}{\epsilon^{3}}\right)$  &  \cmark \par   
		\\  \midrule 
		{M-ZSCG} &\citep{balasubramanian2018zeroth} &$\mathcal{O}\left( \frac{d}{\epsilon^{2}}\right)$  & \cmark \par \\ \midrule
		\belowrulesepcolor{LightCyan}
		\rowcolor{LightCyan}
		ZO-SPIDER-Coord-C & (This work)&$\mathcal{O}(\min\{dn,\frac{d}{\epsilon}\}\log\left(\frac{1}{\epsilon}\right))$& \xmark  \\   \aboverulesepcolor{LightCyan}  \bottomrule
		\belowrulesepcolor{LightCyan}
		\rowcolor{LightCyan}
		ZO-SVRG-Coord-Rand-C  &(This work)&$\mathcal{O}(\min\{dn,\frac{d}{\epsilon}\}\log\left(\frac{1}{\epsilon}\right))$ & \cmark \par\\   \aboverulesepcolor{LightCyan}  \bottomrule
	\end{tabular} 
	\label{tas2}
\end{table*}

\newpage
{\Large{\bf Technical Proofs}}
\vspace{-0.2cm}
\section{Proof for ZO-SVRG-Coord-Rand  }
\subsection{Auxiliary Lemmas}
Before proving our main results, we first establish three useful lemmas. 
\begin{lemma}\label{coordinate}
	For any given smoothing parameter $\delta>0$ and any $\mathbf{x}\in\mathbb{R}^d$, we have 
	\begin{align*}
	\|   \hat \nabla_{\text{\normalfont coord}}f(\mathbf{x})-\nabla f (\mathbf{x})\|^2\leq L^2d\delta^2.
	\end{align*}
\end{lemma}
\begin{proof}
	Applying the mean value theorem (MVT) to the gradient $\nabla f(\mathbf{x})$, we have, for any given $\delta>0$,
	\begin{align*}
	\|   \hat \nabla_{\text{\normalfont coord}}f(\mathbf{x})-\nabla f (\mathbf{x})\|^2&=\Big\|\frac{1}{2\delta}\sum_{i=1}^d(2\delta \mathbf{e}_i\mathbf{e}_i^T\nabla f(\mathbf{x}+(2t_i-1)\delta\mathbf{e}_i))-\nabla f (\mathbf{x})\Big\|^2, \; \text{for  }\,0<t_i<1,
	\\&
	\overset{\text{(i)}}= \sum_{i=1}^d  \Big\|    \mathbf{e}_i\mathbf{e}_i^T (\nabla f(\mathbf{x}+(2t_i-1)\delta\mathbf{e}_i)-\nabla f (\mathbf{x}))  \Big\|^2
	\\& \leq \sum_{i=1}^d  \Big\|    \nabla f(\mathbf{x}+(2t_i-1)\delta\mathbf{e}_i)-\nabla f (\mathbf{x}) \Big\|^2 
	\overset{\text{(ii)}}\leq L^2 \sum_{i=1}^d\|(2t_i-1)\delta\mathbf{e}^i\|^2\leq  L^2d\delta^2
	\end{align*}
	where (i) follows from the definition of $\mathbf{e}_i$ and Euclidean norm, and (ii) follows from Assumption~\ref{assumption}.
\end{proof}

\begin{lemma}\label{evs1}
	For  any given $k_0\leq \lfloor K /q\rfloor$, we have 
	\begin{align*}
	\mathbb{E}\|\mathbf{v}^{qk_0}-\hat \nabla_{\text{\normalfont coord}}f(\mathbf{x}^{qk_0})\|^2\leq  \frac{3I(|\mathcal{S}_1|<n)}{|\mathcal{S}_1|}\left(  2L^2d\delta^2+\sigma^2\right),
	\end{align*}
	where $I(\cdot)$ is the indicator function. 
\end{lemma}
\begin{proof}
	To simplify notation, we let $\mathbf{z}_j= \hat \nabla_{\text{coord}}f_j(\mathbf{x}^{qk_0})-\hat \nabla_{\text{coord}}f(\mathbf{x}^{qk_0}) $ and $I_j=I(j\in\mathcal{S}_1)$, where $I(\cdot)$ is the indicator function. First note that 
	$	\mathbb{E}(I^2_j)=\frac{|\mathcal{S}_1|}{n}\;\;\text{and}\;\; \mathbb{E}I_iI_j=C^2_{|\mathcal{S}_1|}/C_n^2=\frac{|\mathcal{S}_1|(|\mathcal{S}_1|-1)}{n(n-1)}, i\neq j$.
	Then, based on the above equalities,  we have 
	\begin{small}
		\begin{align*}
		&\mathbb{E}\|	\mathbf{v}^{qk_0}-\hat \nabla_{\text{coord}}f(\mathbf{x}^{qk_0})\|^2=
		\frac{1}{|\mathcal{S}_1|^2}\left(\sum_{j=1}^n\mathbb{E} I_j^2\Big\| \mathbf{z}_j  \Big\|^2+\sum_{i\neq j} \mathbb{E}I_iI_j\big\langle \mathbf{z}_i,
		\mathbf{z}_j\big\rangle\right)\nonumber
		\\ &=\frac{1}{|\mathcal{S}_1|^2}\left(\frac{|\mathcal{S}_1|}{n}\sum_{j=1}^n \Big\| \mathbf{z}_j  \Big\|^2+\frac{|\mathcal{S}_1|(|\mathcal{S}_1|-1)}{n(n-1)}\sum_{i\neq j} \big\langle \mathbf{z}_i,
		\mathbf{z}_j\big\rangle\right)\nonumber
		\\
		&=\frac{1}{|\mathcal{S}_1|^2}\left(\left(\frac{|\mathcal{S}_1|}{n}-\frac{|\mathcal{S}_1|(|\mathcal{S}_1|-1)}{n(n-1)}\right)\sum_{j=1}^n \Big\| \mathbf{z}_j  \Big\|^2+\frac{|\mathcal{S}_1|(|\mathcal{S}_1|-1)}{n(n-1)}\left\| \sum_{j=1}^n \mathbf{z}_j  \right\|^2
		\right)
		\\
		&=\frac{n-|\mathcal{S}_1|}{(n-1)|\mathcal{S}_1|}\frac{1}{n}\sum_{j=1}^n\Big\| \mathbf{z}_j  \Big\|^2+\frac{(|\mathcal{S}_1|-1)}{n(n-1)|\mathcal{S}_1|}\left\| \sum_{j=1}^n \mathbf{z}_j  \right\|^2
		\\
		&\leq \frac{I(|\mathcal{S}_1|<n)}{|\mathcal{S}_1|}\frac{1}{n}\sum_{j=1}^n \Big\| \mathbf{z}_j  \Big\|^2 + \frac{1}{n^2}\left\| \sum_{j=1}^n \mathbf{z}_j  \right\|^2
		\\
		&= \frac{I(|\mathcal{S}_1|<n)}{|\mathcal{S}_1|}\frac{1}{n}\sum_{j=1}^n\Big\| \hat \nabla_{\text{coord}}f_j(\mathbf{x}^{qk_0})-\hat \nabla_{\text{\normalfont coord}}f(\mathbf{x}^{qk_0}) \Big\|^2 \nonumber
		\\&\leq \frac{3I(|\mathcal{S}_1|<n)}{|\mathcal{S}_1|}\frac{1}{n}\sum_{j=1}^n\Big( \|\hat \nabla_{\text{coord}}f_j(\mathbf{x}^{qk_0})- \nabla f_j(\mathbf{x}^{qk_0})\|^2+\| \nabla f_j(\mathbf{x}^{qk_0})-\nabla f(\mathbf{x}^{qk_0}) \|^2  \nonumber
		\\&\quad+\|\hat \nabla_{\text{coord}}f(\mathbf{x}^{qk_0})- \nabla f(\mathbf{x}^{qk_0})\|^2\Big)\leq\frac{3I(|\mathcal{S}_1|<n)}{|\mathcal{S}_1|}\left(  2L^2d\delta^2+\sigma^2\right),
		\end{align*}
	\end{small}
	\hspace{-0.12cm}	where the last inequality follows from Assumption~\ref{assumption} and Lemma~\ref{coordinate}.
	Then, the proof is complete.
\end{proof}

\begin{lemma}\label{unifoT}
	Let $f_\beta(\mathbf{x})=\mathbb{E}_{\mathbf{u}\sim \text{\normalfont U}_B}\left(f(\mathbf{x}+\beta\mathbf{u})\right)$ be a smooth approximation of  $f(\mathbf{x})$, where $\text{\normalfont U}_B$ is  the uniform distribution over the $d$-dimensional unit Euclidean ball $B$. Then, 
	\begin{enumerate}
		\item[\normalfont(1)] $|f_{\beta}(\mathbf{x})-f(\mathbf{x})|\leq \frac{\beta^2L}{2}$ and $\|\nabla f_\beta(\mathbf{x})-\nabla f(\mathbf{x})\|\leq \frac{\beta Ld}{2}\,$ for any $\mathbf{x}\in\mathbb{R}^d$
		\item[\normalfont(2)] 
		$\mathbb{E}_k\left(  \frac{1}{|\mathcal{S}_2|}\sum_{j=1}^{|\mathcal{S}_2|}\hat \nabla_{\normalfont \text{rand}} f_{a_j}(\mathbf{x}; \mathbf{u}_j^k)\right) =\nabla f_\beta(\mathbf{x})$, where $\mathbf{x}$ is either $\mathbf{x}^k \text{ or } \mathbf{x}^{qk_0} $
		\item[\normalfont(3)] 
		$\mathbb{E}_k\left\|	\hat \nabla_{\normalfont \text{rand}} f_{a_j}(\mathbf{x}^k; \mathbf{u}_j^k)-\hat \nabla_{\normalfont \text{rand}} f_{a_j}(\mathbf{x}^{qk_0}; \mathbf{u}_j^k)\right\|^2
		\\\leq 3d^2\mathbb{E}_k \langle \nabla f_{a_j}(\mathbf{x}^k)- \nabla f_{a_j}(\mathbf{x}^{qk_0}),\mathbf{u}_j^k\rangle^2+\frac{3L^2d^2\beta^2}{2}
		\\\leq 3dL^2\|\mathbf{x}^k-\mathbf{x}^{qk_0}\|^2+\frac{3L^2d^2\beta^2}{2}$,
	\end{enumerate}
	where the shorthand $\mathbb{E}_k(\cdot)=\mathbb{E}(\cdot\; | \mathbf{x}^0,....,\mathbf{x}^k\;)$.
\end{lemma}
\begin{proof}
	The proof of item (1) directly follows from Lemma 4.1 in~\citealt{gao2014information}. 
	
	We next prove item (2).
	Based on the equation (3.4) in~\citealt{gao2014information}, we have 
	\begin{align}\label{deltaBb}
	\nabla f_\beta(\mathbf{x}^k)=&\mathbb{E}_{\mathbf{u}\sim {\text{ \normalfont U}}_{S_p}}\Big( \frac{d}{\beta}f(\mathbf{x}^k+\beta\mathbf{u})\mathbf{u}  \Big)\overset{\text {(i)}}=\mathbb{E}_{\mathbf{u}\sim {\text{ \normalfont U}}_{S_p}}\left( \frac{d}{\beta} \Big(f(\mathbf{x}^k+\beta\mathbf{u})-f(\mathbf{x}^k)\Big)\mathbf{u}  \right)
	\nonumber 
	\\=&\mathbb{E}_{\mathbf{u}\sim {\text{ \normalfont U}}_{S_p}} \Big( \frac{1}{n}\sum_{i=1}^n \frac{d}{\beta}\left(f_i(\mathbf{x}^k+\beta\mathbf{u})-f_i(\mathbf{x}^k)\right)\mathbf{u}  \Big) =\frac{1}{n}\sum_{i=1}^n \mathbb{E}_{\mathbf{u}\sim {\text{ \normalfont U}}_{S_p}} \Big(  \frac{d}{\beta}\left(f_i(\mathbf{x}^k+\beta\mathbf{u})-f_i(\mathbf{x}^k)\right)\mathbf{u}  \Big),
	\end{align}
	where the random vector $\mathbf{u}$ is independent of  $\mathbf{x}^k$, $\text{\normalfont U}_{S_p}$ is  the uniform distribution over the unit sphere $S_p$ and (i) follows from the fact that $\mathbb{E}_{\mathbf{u}\sim {\text{ \normalfont U}}_{S_p}} \left(f(\mathbf{x}^k)\mathbf{u} \right)=0$. To simplify notation, we let $\mathbb{E}_k(\cdot)=\mathbb{E}(\cdot | \mathbf{x}^1,...,\mathbf{x}^k)$. 
	Conditioned on $\mathbf{x}^0,....,\mathbf{x}^k$ and noting that the random samples in $\mathcal{S}_2$ and $\mathbf{u}_j^k,j=1,...,|\mathcal{S}_2|$ generated at the $k^{th}$ iteration are independent of $\mathbf{x}^0,....,\mathbf{x}^k$, we have
	\begin{align}\label{ggsss}
	\mathbb{E}_{k}&\bigg(  \frac{d}{\beta|\mathcal{S}_2|}\sum_{j=1}^{|\mathcal{S}_2|}(f_{a_j}(\mathbf{x}^k+\beta \mathbf{u}_j^k) -f_{a_j}(\mathbf{x}^{k}))\mathbf{u}_j^k  \bigg) =  \frac{1}{|\mathcal{S}_2|}\sum_{j=1}^{|\mathcal{S}_2|} \mathbb{E}_{k}\bigg(  \frac{d}{\beta}(f_{a_j}(\mathbf{x}^k+\beta \mathbf{u}_j^k) -f_{a_j}(\mathbf{x}^{k}))\mathbf{u}_j^k  \bigg)  \nonumber
	\\&=\frac{1}{|\mathcal{S}_2|}\sum_{j=1}^{|\mathcal{S}_2|} \mathbb{E}_{k}\bigg( \mathbb{E}_{a_j}\Big( \frac{d}{\beta}(f_{a_j}(\mathbf{x}^k+\beta \mathbf{u}_j^k) -f_{a_j}(\mathbf{x}^{k}))\mathbf{u}_j^k  \, \Big| \, \mathbf{u}_j^k\Big)\bigg)   \nonumber
	\\&\overset{\text{(i)}}=\frac{1}{|\mathcal{S}_2|}\sum_{j=1}^{|\mathcal{S}_2|} \mathbb{E}_{k}\bigg( \frac{1}{n}\sum_{i=1}^n \frac{d}{\beta}(f_{i}(\mathbf{x}^k+\beta \mathbf{u}_j^k) -f_{i}(\mathbf{x}^{k}))\mathbf{u}_j^k  \bigg)   \nonumber
	\\ &\overset{\text{(ii)}}=\frac{1}{|\mathcal{S}_2|}\sum_{j=1}^{|\mathcal{S}_2|} \left(\nabla f_\beta(\mathbf{x}^k)  \right)=\nabla f_\beta(\mathbf{x}^k), 
	\end{align}
	where (i) follows from the definition of the set $\mathcal{S}_1$ and (ii) follows from~\eqref{deltaBb}.  
	Taking steps similar to~\eqref{ggsss} and conditioning on $\mathbf{x}^0,...,\mathbf{x}^k$, we have
	\begin{align}\label{ucb2}
	\mathbb{E}_{k}&\bigg(  \frac{d}{\beta|\mathcal{S}_2|}\sum_{j=1}^{|\mathcal{S}_2|}(f_{a_j}(\mathbf{x}^{qk_0}+\beta \mathbf{u}_j^{k}) -f_{a_j}(\mathbf{x}^{k}))\mathbf{u}_j^{k}  \bigg)=\nabla f_\beta(\mathbf{x}^{qk_0}).
	\end{align}
	
	Our final step is to prove item (3). Note that 
	\begin{align}\label{geness11}
	\mathbb{E}\bigg\| 	&\frac{d(f_{a_j}(\mathbf{x}^{k}+\beta \mathbf{u}_j^{k}) -f_{a_j}(\mathbf{x}^k))}{\beta}\mathbf{u}_j^k-\frac{d(f_{a_j}(\mathbf{x}^{qk_0}+\beta \mathbf{u}_j^{k}) -f_{a_j}(\mathbf{x}^{qk_0}))}{\beta}\mathbf{u}_j^{k}\bigg\|^2\nonumber
	\\=&
	d^2\mathbb{E}_k\bigg\|\frac{(f_{a_j}(\mathbf{x}^k+\beta \mathbf{u}_j^{k}) -f_{a_j}(\mathbf{x}^k)-\langle \nabla f_{a_j}(\mathbf{x}^k),\beta\mathbf{u}_j^k\rangle)\mathbf{u}_j^k}{\beta}+  \left(\langle \nabla f_{a_j}(\mathbf{x}^k),\mathbf{u}_j^k\rangle\mathbf{u}_j^k-\langle \nabla f_{a_j}(\mathbf{x}^{qk_0}),\mathbf{u}_j^k\rangle \mathbf{u}_j^k\right) \nonumber
	\\&-\frac{(f_{a_j}(\mathbf{x}^{qk_0}+\beta \mathbf{u}_j^{k}) -f_{a_j}(\mathbf{x}^{qk_0})-\langle \nabla f_{a_j}(\mathbf{x}^{qk_0}),\beta\mathbf{u}_j^k\rangle)\mathbf{u}_j^k}{\beta}\bigg\|^2
	\end{align}
	Then, using the  inequality that $f_{a_j}(\mathbf{y})-f_{a_j}(\mathbf{x})-\langle \nabla f_{a_j}(\mathbf{x}),\mathbf{y-x}\rangle\leq \frac{L}{2}\|\mathbf{y-x}\|^2$ in~\eqref{geness11} yields 
	\begin{align}\label{longp}
	\mathbb{E}\bigg\| 	&\frac{d(f_{a_j}(\mathbf{x}^{k}+\beta \mathbf{u}_j^{k}) -f_{a_j}(\mathbf{x}^k))}{\beta}\mathbf{u}_j^k-\frac{d(f_{a_j}(\mathbf{x}^{qk_0}+\beta \mathbf{u}_j^{k}) -f_{a_j}(\mathbf{x}^{qk_0}))}{\beta}\mathbf{u}_j^{k}\bigg\|^2\nonumber
	\\\leq& 3d^2\mathbb{E}_k\|\langle \nabla f_{a_j}(\mathbf{x}^k),\mathbf{u}_j^k\rangle\mathbf{u}_j^k-\langle \nabla f_{a_j}(\mathbf{x}^{qk_0}),\mathbf{u}_j^k\rangle \mathbf{u}_j^k\|^2+\frac{3L^2d^2\beta^2}{2} \nonumber
	\\\overset{\text{(i)}}=& 3d^2\mathbb{E}_k \langle \nabla f_{a_j}(\mathbf{x}^k)- \nabla f_{a_j}(\mathbf{x}^{qk_0}),\mathbf{u}_j^k\rangle^2+\frac{3L^2d^2\beta^2}{2} \nonumber
	\\=& 3d^2 \left( \nabla f_{a_j}(\mathbf{x}^k)- \nabla f_{a_j}(\mathbf{x}^{qk_0})\right)^T \mathbb{E}(\mathbf{u}_j^k(\mathbf{u}_j^k)^T)\left( \nabla f_{a_j}(\mathbf{x}^k)- \nabla f_{a_j}(\mathbf{x}^{qk_0})\right)+\frac{3L^2d^2\beta^2}{2}
	\end{align}
	where (i) follows from the fact that $\|\mathbf{u}_j^k\|=1$. Based on the definition of $\mathbf{u}_j^k$ , we rewrite $\mathbf{u}_j^k=\mathbf{r}/\|\mathbf{r}\|$ and define a matrix  $\mathbf{U}=\mathbb{E}(\mathbf{u}_j^k(\mathbf{u}_j^k)^T)$, where $\mathbf{r}$ is a $d$-dimensional Gaussian standard random vector. Let $\mathbf{r}(i)$ denote the $i^{th}$ entry of $\mathbf{r}$, and $\mathbf{U}(i,j)$ denote $(i,j)^{th}$ entry of $\mathbf{U}$. Then, we have, for $i=1,...,d$ 
	\begin{align}
	\mathbf{U}(i,i)=\int \frac{\mathbf{r}(i)^2}{\sum_{t=1}^d\mathbf{r}(t)^2} e^{-\frac{\sum_{t=1}^d\mathbf{r}(t)^2}{2}}d \mathbf{r}(1)\cdots d. \mathbf{r}(d)
	\end{align}
	Since $\mathbf{r}(i),i=1,...,d$ are i.i.d. standard Gaussian random variables,  we have 
	\begin{align*}
	\mathbf{U}(1,1)=\cdots  \mathbf{U}(d,d), \text{ and } \sum_{i=1}^d\mathbf{U}(i,i)=\int e^{-\frac{\sum_{t=1}^d\mathbf{r}(t)^2}{2}}d \mathbf{r}(1)\cdots d \mathbf{r}(d)=1,
	\end{align*}
	which implies that $\mathbf{U}(i,i)=1/d$ for all $i=1,...,d$. In addition, for any $i\neq j$, we have
	\begin{align*}
	\mathbf{U}(i,j)=\int \frac{\mathbf{r}(i)\mathbf{r}(j)}{\sum_{t=1}^d\mathbf{r}(t)^2} e^{-\frac{\sum_{t=1}^d\mathbf{r}(t)^2}{2}}d \mathbf{r}(1)\cdots d\mathbf{r}(d),
	\end{align*}
	which, noting the symmetry between $\mathbf{r}(i)$ and $\mathbf{r}(j)$, implies that $\mathbf{U}(i,j)=0$.  Combining the above two results yields that $\mathbf{U}=\frac{1}{d}\mathbf{I}_d$, where $\mathbf{I}_d$ is a $d$-dimensional identity matrix. Thus, plugging $\mathbb{E}(\mathbf{u}_j^k(\mathbf{u}_j^k)^T)=\mathbf{U}=\frac{1}{d}\mathbf{I}_d$ in~\eqref{longp} yields 
	\begin{align}
	\mathbb{E}\bigg\| 	&\frac{d(f_{a_j}(\mathbf{x}^{k}+\beta \mathbf{u}_j^{k}) -f_{a_j}(\mathbf{x}^k))}{\beta}\mathbf{u}_j^k-\frac{d(f_{a_j}(\mathbf{x}^{qk_0}+\beta \mathbf{u}_j^{k}) -f_{a_j}(\mathbf{x}^{qk_0}))}{\beta}\mathbf{u}_j^{k}\bigg\|^2\nonumber
	\\\leq&3d\|\nabla f_{a_j}(\mathbf{x}^k)- \nabla f_{a_j}(\mathbf{x}^{qk_0})\|^2+\frac{3L^2d^2\beta^2}{2}\leq 3dL^2\|\mathbf{x}^k-\mathbf{x}^{qk_0}\|^2+\frac{3L^2d^2\beta^2}{2},
	\end{align}
	which finishes the proof. 
\end{proof}

\subsection{Proof of Lemma~\ref{le:vks}}
Using Lemmas~\ref{coordinate},~\ref{evs1},~\ref{unifoT}, we now prove Lemma~\ref{le:vks}. 
Based on the updating step of Algorithm~\ref{ours:2}, we obtain 
\begin{align}\label{aboves}
\mathbf{v}^k-\nabla f_\beta(\mathbf{x}^k)=&\frac{d}{\beta|\mathcal{S}_2|}\sum_{j=1}^{|\mathcal{S}_2|}(f_{a_j}(\mathbf{x}^{k}+\beta \mathbf{u}_j^{k}) -f_{a_j}(\mathbf{x}^k))\mathbf{u}_j^k-\nabla f_\beta(\mathbf{x}^k)  \nonumber
\\&-\frac{d}{\beta|\mathcal{S}_2|}\sum_{j=1}^{|\mathcal{S}_2|}(f_{a_j}(\mathbf{x}^{qk_0}+\beta \mathbf{u}_j^{k}) -f_{a_j}(\mathbf{x}^{qk_0}))\mathbf{u}_j^{k}+\nabla f_\beta(\mathbf{x}^{qk_0})+\mathbf{v}^{qk_0}-\nabla f_\beta(\mathbf{x}^{qk_0}).  
\end{align}
To simplify notation, we define 
\begin{align*}
H_j(\mathbf{x})=\frac{d}{\beta}(f_{a_j}(\mathbf{x}+\beta \mathbf{u}_j^{k}) -f_{a_j}(\mathbf{x}))\mathbf{u}_j^k,\,\mathbf{x}=\mathbf{x}^k \text{ or }\,\mathbf{x}^{qk_0},
\end{align*}
and use the shorthand $\mathbb{E}_k(\cdot)$ to denote $\mathbb{E}(\cdot\,|\,\mathbf{x}^0,...,\mathbf{x}^k)$. 
Then, using~\eqref{aboves}, we obtain 
\begin{align}\label{spls}
\mathbb{E}_k\|\mathbf{v}^k-\nabla f_\beta(\mathbf{x}^k)\|^2\leq &\frac{2}{|\mathcal{S}_2|}\mathbb{E}_k\|H_j(\mathbf{x}^k)-H_j(\mathbf{x}^{qk_0})-(\nabla f_\beta(\mathbf{x}^k)-\nabla f_\beta(\mathbf{x}^{qk_0}))\|^2 \nonumber
\\&+2\sum_{i \neq j}\mathbb{E}_k\langle H_i(\mathbf{x}^k)-H_i(\mathbf{x}^{qk_0})-(\nabla f_\beta(\mathbf{x}^k)-\nabla f_\beta(\mathbf{x}^{qk_0})),
H_j(\mathbf{x}^k)-H_j(\mathbf{x}^{qk_0}) \nonumber
\\&\quad -(\nabla f_\beta(\mathbf{x}^k)-\nabla f_\beta(\mathbf{x}^{qk_0}))\rangle +2\|\mathbf{v}^{qk_0}-\nabla f_\beta(\mathbf{x}^{qk_0})\|^2  \nonumber
\\\overset{\text{(i)}}=&\frac{2}{|\mathcal{S}_2|}\mathbb{E}_k\|H_j(\mathbf{x}^k)-H_j(\mathbf{x}^{qk_0})-(\nabla f_\beta(\mathbf{x}^k)-\nabla  f_\beta(\mathbf{x}^{qk_0}))\|^2\nonumber
\\&+2\|\mathbf{v}^{qk_0}-\nabla f_\beta(\mathbf{x}^{qk_0})\|^2
\end{align}
where (i) follows from the fact that $a_i$ and $\mathbf{u}_i^k$ are independent of $a_j$ and $\mathbf{u}_j^k$ for any $i\neq j$, and  from the following equalities 
\begin{align}\label{nices}
\mathbb{E}_k(H_j(\mathbf{x}^k))&=\mathbb{E}_{\mathbf{u}_j^k}\Big( \frac{d}{\beta}f(\mathbf{x}^k+\beta\mathbf{u}_j^k)\mathbf{u}_j^k  \Big)=\nabla f_\beta(\mathbf{x}^k)\nonumber
\\ \mathbb{E}_k(H_j(\mathbf{x}^{qk_0}))&=\mathbb{E}_{\mathbf{u}_j^k}\Big( \frac{d}{\beta}f(\mathbf{x}^{qk_0}+\beta\mathbf{u}_j^k)\mathbf{u}_j^k  \Big)=\nabla f_\beta(\mathbf{x}^{qk_0}).
\end{align}
Then, we further simplify~\eqref{spls} to obtain 
\begin{align}\label{sjkx}
\mathbb{E}_k\|\mathbf{v}^k-\nabla f_\beta(\mathbf{x}^k)\|^2\leq&\frac{2}{|\mathcal{S}_2|}\left(\mathbb{E}_k\|H_j(\mathbf{x}^k)-H_j(\mathbf{x}^{qk_0})\|^2+\mathbb{E}_k\|\nabla f_\beta(\mathbf{x}^k)-\nabla  f_\beta(\mathbf{x}^{qk_0})\|^2\right)\nonumber
\\&-\frac{4}{|\mathcal{S}_2|}\mathbb{E}_k\left\langle H_j(\mathbf{x}^k)-H_j(\mathbf{x}^{qk_0}),\nabla f_\beta(\mathbf{x}^k)-\nabla  f_\beta(\mathbf{x}^{qk_0}) \right\rangle \nonumber
\\&+2\|\mathbf{v}^{qk_0}-\nabla f_\beta(\mathbf{x}^{qk_0})\|^2 \nonumber
\\\overset{\text{(i)}}\leq&  \frac{2}{|\mathcal{S}_2|}\mathbb{E}_k\|H_j(\mathbf{x}^k)-H_j(\mathbf{x}^{qk_0})\|^2+2\|\mathbf{v}^{qk_0}-\nabla f_\beta(\mathbf{x}^{qk_0})\|^2\nonumber
\\\leq&  \frac{2}{|\mathcal{S}_2|}\mathbb{E}_k\|H_j(\mathbf{x}^k)-H_j(\mathbf{x}^{qk_0})\|^2+6\|\mathbf{v}^{qk_0}-\hat \nabla_{\text{\normalfont coord}}f(\mathbf{x}^{qk_0})\|^2\nonumber
\\&+6\|\nabla f_\beta(\mathbf{x}^{qk_0})- \nabla f(\mathbf{x}^{qk_0})\|^2+6\|\nabla f(\mathbf{x}^{qk_0})-\hat \nabla_{\text{\normalfont coord}}f(\mathbf{x}^{qk_0})\|^2\nonumber
\\\overset{\text{(ii)}}\leq & \frac{2}{|\mathcal{S}_2|}\mathbb{E}_k\|H_j(\mathbf{x}^k)-H_j(\mathbf{x}^{qk_0})\|^2+\frac{18I(|\mathcal{S}_1|<n)}{|\mathcal{S}_1|}\left(  2L^2d\delta^2+\sigma^2\right)\nonumber
\\&+6L^2d\delta^2+\frac{3\beta^2L^2d^2}{2}
\end{align}
where (i) follows from~\eqref{nices} and (ii) follows from Lemmas~\ref{unifoT},~\ref{coordinate} and~\ref{evs1}. Then, based on item (3) in Lemma~\ref{unifoT}, we obtain 
\begin{align}\label{sscxa}
\mathbb{E}_k\|H_j(\mathbf{x}^k)-H_j(\mathbf{x}^{qk_0})\|^2&=\mathbb{E}_k\left\|	\frac{d(f_{a_j}(\mathbf{x}^{k}+\beta \mathbf{u}_j^{k}) -f_{a_j}(\mathbf{x}^k))}{\beta}\mathbf{u}_j^k-\frac{d(f_{a_j}(\mathbf{x}^{qk_0}+\beta \mathbf{u}_j^{k}) -f_{a_j}(\mathbf{x}^{qk_0}))}{\beta}\mathbf{u}_j^{k}\right\|^2\nonumber
\\&\leq  3dL^2\|\mathbf{x}^k-\mathbf{x}^{qk_0}\|^2+\frac{3L^2d^2\beta^2}{2}.
\end{align}
Combining~\eqref{sjkx} and~\eqref{sscxa} finishes the proof.

\subsection{Proof of Theorem~\ref{th:svrg}}
Since $K=qh$ and 
$\nabla f_\beta(\mathbf{x})$ is $L$-Lipschitz, we have,  
for   $qm\leq k\leq q(m+1)-1,m=0,...,h-1$
\begin{align*}
f_\beta(\mathbf{x}^{k+1}) &\leq  f_\beta(\mathbf{x}^{k})+\langle \nabla f_\beta(\mathbf{x}^k),\mathbf{x}^{k+1}-\mathbf{x}^k \rangle +  \frac{L\eta^2}{2}\| \mathbf{v}^k\|^2  = f_\beta(\mathbf{x}^{k})-\eta\langle \nabla f_\beta(\mathbf{x}^k),\mathbf{v}^k \rangle +  \frac{L\eta^2}{2}\| \mathbf{v}^k\|^2. 
\end{align*}
Taking the expectation over the above inequality and noting from Lemma~\ref{unifoT}  that $\mathbb{E}(\mathbf{v}^k\,|\,\mathbf{x}^0,..., \mathbf{x}^k)= \nabla f_\beta(\mathbf{x}^k)-\nabla f_\beta(\mathbf{x}^{qm})+\mathbf{v}^{qm}$, we have
\begin{align}\label{betaxs}
\mathbb{E}f_\beta(\mathbf{x}^{k+1}) \leq& \mathbb{E} f_\beta(\mathbf{x}^{k})-\eta\mathbb{E} \langle \nabla f_\beta(\mathbf{x}^k)-\nabla f_\beta(\mathbf{x}^{qm})+\mathbf{v}^{qm} , \nabla f_\beta(\mathbf{x}^k)   \rangle   +  \frac{L\eta^2}{2}\mathbb{E}\| \mathbf{v}^k\|^2 \nonumber
\\\leq &\mathbb{E} f_\beta(\mathbf{x}^{k})-\eta\mathbb{E}\|\nabla f_\beta(\mathbf{x}^k)\|^2+
\eta\mathbb{E} \langle \nabla f_\beta(\mathbf{x}^{qm})-\mathbf{v}^{qm} , \nabla f_\beta(\mathbf{x}^k)   \rangle   +  \frac{L\eta^2}{2}\mathbb{E}\| \mathbf{v}^k\|^2\nonumber
\\\leq & \mathbb{E} f_\beta(\mathbf{x}^{k})-\eta\mathbb{E}\|\nabla f_\beta(\mathbf{x}^k)\|^2+
\frac{\eta}{2}\mathbb{E}\| \nabla f_\beta(\mathbf{x}^{qm})-\mathbf{v}^{qm}\|^2+\frac{\eta}{2}\mathbb{E}\|\nabla f_\beta(\mathbf{x}^k)\|^2       +  \frac{L\eta^2}{2}\mathbb{E}\| \mathbf{v}^k\|^2\nonumber
\\\leq & \mathbb{E} f_\beta(\mathbf{x}^{k})-\frac{\eta}{2}\mathbb{E}\|\nabla f_\beta(\mathbf{x}^k)\|^2+
\frac{\eta}{2}\mathbb{E}\| \nabla f_\beta(\mathbf{x}^{qm})-\mathbf{v}^{qm}\|^2     +  \frac{L\eta^2}{2}\mathbb{E}\| \mathbf{v}^k\|^2 \nonumber
\\\overset{\text{(i)}}\leq&\mathbb{E} f_\beta(\mathbf{x}^{k})-\frac{\eta}{2}\left(  \frac{1}{2}\|\nabla f(\mathbf{x}^k)\|^2-   \frac{\beta^2L^2d^2}{4}  \right)+
\frac{\eta}{2}\mathbb{E}\| \nabla f_\beta(\mathbf{x}^{qm})-\mathbf{v}^{qm}\|^2     +  \frac{L\eta^2}{2}\mathbb{E}\| \mathbf{v}^k\|^2
\end{align}
where (i) follows from the inequality that $\|\mathbf{a}\|^2\geq \frac{1}{2}\|\mathbf{b}\|^2-\|\mathbf{b-a}\|^2$. Using an approach similar to~\eqref{sjkx}, we obtain
\begin{align}\label{otheruse}
\mathbb{E}\| \nabla f_\beta(\mathbf{x}^{qm})-\mathbf{v}^{qm}\|^2\leq& 3\|\mathbf{v}^{qm}-\hat \nabla_{\text{\normalfont coord}}f(\mathbf{x}^{qm})\|^2+3\|\nabla f_\beta(\mathbf{x}^{qm})- \nabla f(\mathbf{x}^{qm})\|^2\nonumber
\\&+3\|\nabla f(\mathbf{x}^{qm})-\hat \nabla_{\text{\normalfont coord}}f(\mathbf{x}^{qm})\|^2 \nonumber
\\\leq&\frac{9I(|\mathcal{S}_1|<n)}{|\mathcal{S}_1|}\left(  2L^2d\delta^2+\sigma^2\right)+3L^2d\delta^2+\frac{3\beta^2L^2d^2}{4}, 
\end{align}
which, in conjunction with~\eqref{betaxs}, implies that 
\begin{align}\label{csqqs}
\mathbb{E}f_\beta(\mathbf{x}^{k+1})\leq& \mathbb{E}f_\beta(\mathbf{x}^{k})-\frac{\eta}{4}\mathbb{E}\|\nabla f(\mathbf{x}^k)\|^2+\frac{\eta}{2}\left(\beta^2L^2d^2 +\frac{9I(|\mathcal{S}_1|<n)}{|\mathcal{S}_1|}\left(  2L^2d\delta^2+\sigma^2\right)+3L^2d\delta^2  \right) \nonumber
\\&+\frac{3L\eta^2}{2}\mathbb{E}\left( \|\nabla f_\beta(\mathbf{x}^k)-\mathbf{v}^k\|^2 +\|\nabla f_\beta(\mathbf{x}^k)-\nabla f(\mathbf{x}^k)\|^2+\|\nabla f(\mathbf{x}^k)\|^2\right) \nonumber
\\\overset{\text{(i)}}\leq&  \mathbb{E}f_\beta(\mathbf{x}^{k})-\left(\frac{\eta}{4}-\frac{3L\eta^2}{2}\right)\mathbb{E}\|\nabla f(\mathbf{x}^k)\|^2+\frac{\eta}{2}\left(\beta^2L^2d^2 +\frac{9I(|\mathcal{S}_1|<n)}{|\mathcal{S}_1|}\left(  2L^2d\delta^2+\sigma^2\right)+3L^2d\delta^2  \right) \nonumber
\\&+\frac{3\eta^2L^3d^2\beta^2}{8}+\frac{3L\eta^2}{|\mathcal{S}_2|}\left(   3d L^2\mathbb{E}\|\mathbf{x}^k-\mathbf{x}^{qm}\|^2+\frac{3L^2\beta^2d^2}{2}  \right)\nonumber
\\&+\frac{3L\eta^2}{2}\left(\frac{18I(|\mathcal{S}_1|<n)}{|\mathcal{S}_1|}\left(  2L^2d\delta^2+\sigma^2\right)+6L^2d\delta^2+\frac{3\beta^2L^2d^2}{2}\right)
\end{align}
where (i) follows from Lemma~\ref{le:vks}. To simplify notation, we define 
\begin{align}\label{para22}
\chi=& \beta^2L^2d^2 +\frac{9I(|\mathcal{S}_1|<n)}{|\mathcal{S}_1|}\left(  2L^2d\delta^2+\sigma^2\right)+3L^2d\delta^2
\end{align}
which, in conjunction with~\eqref{csqqs}, implies that 
\begin{align}\label{newas}
\mathbb{E}f_\beta(\mathbf{x}^{k+1})\leq& \mathbb{E}f_\beta(\mathbf{x}^{k})-\left(\frac{\eta}{4}-\frac{3L\eta^2}{2}\right)\mathbb{E}\|\nabla f(\mathbf{x}^k)\|^2+\frac{9dL^3\eta^2}{|\mathcal{S}_2|}\mathbb{E}\| \mathbf{x}^k-\mathbf{x}^{qm}  \|^2+\frac{3\eta^2L}{2}\left(\frac{L^2d^2\beta^2}{4}+\frac{3L^2d^2\beta^2}{|\mathcal{S}_2|}\right)\nonumber
\\&+\left(\frac{\eta}{2}+3L\eta^2\right)\chi.
\end{align}

We introduce a Lyapunov function $R_k^m=\mathbb{E}\left(  f_\beta(\mathbf{x}^k) +c^m_k\|\mathbf{x}^k-\mathbf{x}^{qm}\|^2 \right)$ for  $qm\leq k\leq q(m+1),m=0,...,h-1$, where $\{c_k^m\}$ are  constants such that $c^m_{q(m+1)}=0$. Then, we  obtain that for any $qm\leq k\leq q(m+1)-1$
\begin{align}\label{ck1s}
R_{k+1}^m=&\mathbb{E}\left(  f_\beta(\mathbf{x}^{k+1}) +c^m_{k+1}\|\mathbf{x}^{k+1}-\mathbf{x}^{k}+\mathbf{x}^k-\mathbf{x}^{qm}\|^2 \right)\nonumber
\\\leq&\mathbb{E} f_\beta(\mathbf{x}^{k+1})+c^m_{k+1}\eta^2\mathbb{E}\|\mathbf{v}^k\|^2+c^m_{k+1}\mathbb{E}\|\mathbf{x}^k-\mathbf{x}^{qm}\|^2-2c^m_{k+1}\eta\mathbb{E}\langle\mathbf{v}^k,  \mathbf{x}^k-\mathbf{x}^{qm} \rangle  \nonumber
\\\overset{\text{(i)}}=&\mathbb{E} f_\beta(\mathbf{x}^{k+1})+c^m_{k+1}\eta^2\mathbb{E}\|\mathbf{v}^k\|^2+c^m_{k+1}\mathbb{E}\|\mathbf{x}^k-\mathbf{x}^{qm}\|^2-2c^m_{k+1}\eta\mathbb{E}\langle \nabla f_\beta(\mathbf{x}^k)-\nabla f_\beta(\mathbf{x}^{qm})+\mathbf{v}^{qm},  \mathbf{x}^k-\mathbf{x}^{qm} \rangle \nonumber
\\\overset{\text{(ii)}}\leq&\mathbb{E} f_\beta(\mathbf{x}^{k+1})+c^m_{k+1}\eta^2\mathbb{E}\|\mathbf{v}^k\|^2+(c^m_{k+1}+c^m_{k+1}\eta g)\mathbb{E}\|\mathbf{x}^k-\mathbf{x}^{qm}\|^2 \nonumber
\\&+\frac{2c^m_{k+1}\eta}{g}\mathbb{E}\left( \|\nabla f_\beta(\mathbf{x}^k)\|^2+\|\nabla f_\beta(\mathbf{x}^{qm})-\mathbf{v}^{qm}\|^2\right)\nonumber
\\\leq&\mathbb{E} f_\beta(\mathbf{x}^{k+1})+c^m_{k+1}\eta^2\mathbb{E}\left(2\|\mathbf{v}^k-\nabla f_\beta(\mathbf{x}^k) \|^2+\frac{\beta^2L^2d^2}{2}\right)+\left(c^m_{k+1}+c^m_{k+1}\eta g\right)\mathbb{E}\|\mathbf{x}^k-\mathbf{x}^{qm}\|^2 \nonumber
\\&+\frac{4c^m_{k+1}\eta}{g}\mathbb{E}\|\nabla f(\mathbf{x}^k)\|^2+\frac{c^m_{k+1}\eta \beta^2L^2d^2}{g}+\frac{2c^m_{k+1}\eta}{g}\chi\nonumber
\\\overset{\text{(iii)}}\leq& \mathbb{E} f_\beta(\mathbf{x}^{k+1})+\left(c^m_{k+1}+c^m_{k+1}\eta g+\frac{12c^m_{k+1}\eta^2 L^2d}{|\mathcal{S}_2|}\right)\mathbb{E}\|\mathbf{x}^k-\mathbf{x}^{qm}\|^2 +c^m_{k+1}\eta^2\left(4\chi+\frac{6L^2\beta^2d^2}{|\mathcal{S}_2|}\right)\nonumber
\\&+\frac{4c^m_{k+1}\eta}{g}\mathbb{E}\|\nabla f(\mathbf{x}^k)\|^2+\frac{c^m_{k+1}\eta \beta^2L^2d^2}{g}+\frac{2c^m_{k+1}\eta}{g}\chi
\end{align}
where (i) follows from the fact that $\mathbb{E}(\mathbf{v}^k\,|\,\mathbf{x}^0,....,\mathbf{x}^k)=\nabla f_\beta(\mathbf{x}^k)-\nabla f_\beta(\mathbf{x}^{qm})+\mathbf{v}^{qm}$, (ii) follows from the fact that $-2\langle \mathbf{a},\mathbf{b}\rangle\leq \|\mathbf{a}\|^2/g+g\|\mathbf{b}\|^2$ holds for any constant $g>0$ and (iii) follows from Lemma~\ref{le:vks}. Combining~\eqref{newas} and~\eqref{ck1s}, we obtain that 
\begin{align}\label{ck1ssss}
R_{k+1}^m\leq& \mathbb{E}f_\beta(\mathbf{x}^{k})-\left(\frac{\eta}{4}-\frac{4c^m_{k+1}\eta}{g}-\frac{3L\eta^2}{2}\right)\mathbb{E}\|\nabla f(\mathbf{x}^k)\|^2  \nonumber
\\&+\left(c^m_{k+1}+c^m_{k+1}\eta g+\frac{12c^m_{k+1}\eta^2 L^2d}{|\mathcal{S}_2|}+\frac{9dL^3\eta^2}{|\mathcal{S}_2|}\right)\mathbb{E}\|\mathbf{x}^k-\mathbf{x}^{qm}\|^2 \nonumber
\\&+\left( \frac{\eta}{2}+\frac{2c^m_{k+1}\eta}{g}+4c^m_{k+1}\eta^2+3L\eta^2 \right)\chi+\left(\frac{3\eta^2L}{2}\left(\frac{1}{4}+\frac{3}{|\mathcal{S}_2|}\right)+\frac{c^m_{k+1}\eta }{g}\right)L^2d^2\beta^2.
\end{align}
We define the following recursion for $qm\leq k\leq q(m+1)-1,m=0,...,h-1$
\begin{align}\label{ckmm}
c^m_k&=c^m_{k+1}+c^m_{k+1}\eta g+\frac{12c^m_{k+1}\eta^2 L^2d}{|\mathcal{S}_2|}+\frac{9dL^3\eta^2}{|\mathcal{S}_2|}
\end{align}
which, in conjunction with~\eqref{ck1ssss}, implies that 
\begin{align}\label{ainish}
R_{k+1}^m\leq& R_k^m-\left(\frac{\eta}{4}-\frac{4c^m_{k+1}\eta}{g}-\frac{3L\eta^2}{2}\right)\mathbb{E}\|\nabla f(\mathbf{x}^k)\|^2 \nonumber
\\&+\left( \frac{\eta}{2}+\frac{2c^m_{k+1}\eta}{g}+4c^m_{k+1}\eta^2+3L\eta^2 \right)\chi+\left(\frac{3\eta^2L}{2}\left(\frac{1}{4}+\frac{3}{|\mathcal{S}_2|}\right)+\frac{c^m_{k+1}\eta }{g}\right)L^2d^2\beta^2 \nonumber
\\\leq& R_k^m-\left(\frac{\eta}{4}-\frac{4c^m_{k+1}\eta}{g}-\frac{3L\eta^2}{2}\right)\mathbb{E}\|\nabla f(\mathbf{x}^k)\|^2 \nonumber
\\&+\left( \frac{\eta}{2}+\frac{2c^m_{k+1}\eta}{g}+4c^m_{k+1}\eta^2+3L\eta^2 \right)\chi+\left(6\eta^2L+\frac{c^m_{k+1}\eta }{g}\right)L^2d^2\beta^2
\end{align}
where the last inequality follows from the fact that $1/4+3/|\mathcal{S}_2|\leq 4$.
Letting $\theta=\eta g +12\eta^2 dL^2/|\mathcal{S}_2|$  and noting that $c_{q(m+1)}^m=0$, we obtain from~\eqref{ckmm} that for $qm\leq k\leq q(m+1)-1,m=0,...,h-1$
\begin{align*}
c_k^m\leq c_{qm}^m=\frac{9dL^3\eta^2}{|\mathcal{S}_2|}\frac{(1+\theta)^q-1}{\theta},
\end{align*}
which, in conjunction with~\eqref{ainish} and the parameter selection in~\eqref{ppsse}, implies that   
\begin{align*}
R_{k+1}^m\leq& R_k^m-\lambda \mathbb{E}\|\nabla f(\mathbf{x}^k)\|^2+\tau.
\end{align*}
Telescoping the above inequality over $k$ from $qm$ to $q(m+1)-1$ and noting that $R_{qm}^m=\mathbb{E}f_\beta(\mathbf{x}^{qm})$ and $R_{q(m+1)}^m=\mathbb{E} f_\beta(\mathbf{x}^{q(m+1)})$, we obtain 
\begin{align*}
\mathbb{E}f_\beta (\mathbf{x}^{q(m+1)})\leq \mathbb{E}f_\beta (\mathbf{x}^{qm})-\lambda \sum_{k=qm}^{q(m+1)-1} \mathbb{E}\|\nabla f(\mathbf{x}^k)\|^2+q\tau. 
\end{align*}
Then, telescoping the above inequality over $m$ from $0$ to $h-1$,  we obtain 
\begin{align*}
\mathbb{E}f_\beta (\mathbf{x}^{K})\leq \mathbb{E}f_\beta (\mathbf{x}^{0})-\lambda \sum_{k=0}^K\mathbb{E}\|\nabla f(\mathbf{x}^k)\|^2+K\tau,
\end{align*}
which can be rewritten as 
\begin{align}\label{sklip}
\frac{1}{K+1}\sum_{k=0}^K\mathbb{E}\|\nabla f(\mathbf{x}^k)\|^2\leq \frac{f_\beta(\mathbf{x}^0)-f_{\beta}(\mathbf{x}_\beta^*)}{\lambda(K+1)}+\frac{\tau}{\lambda},
\end{align}
where $\mathbf{x}_\beta^*=\arg\min_{\mathbf{x}}f_\beta(\mathbf{x})$. Since the output $\mathbf{x}^\zeta$ of  Algorithm~\ref{ours:2} is generated from $\{\mathbf{x}^0,....,\mathbf{x}^K \}$ uniformly at random, we have 
\begin{align*}
\mathbb{E}\|\nabla f(\mathbf{x}^\zeta)\|^2=\frac{1}{K+1}\sum_{k=0}^K\|\nabla f(\mathbf{x}^k)\|^2,
\end{align*}
which, in conjunction with~\eqref{sklip}, finishes the proof. 
\subsection{Proof of Corollary~\ref{co1:svrg}}\label{profco1}
We prove two cases with $n\leq K$ and $n> K$, separately. 

First we suppose $n\leq K$. In this case, we have $|\mathcal{S}_1|=n$. Recall from~\eqref{cchoose} that  
\begin{align}\label{thess}
c=\frac{9dL^3\eta^2}{|\mathcal{S}_2|}\frac{(1+\theta)^q-1}{\theta}
\end{align}
where 
$\theta=\eta g+12\eta^2dL^2/|\mathcal{S}_2|$. 
Based on the parameter selection in~\eqref{pacos},  we have 
\begin{align}\label{opis}
\theta\overset{\text{(i)}}\leq \frac{1}{2q}+\frac{3}{100}\frac{1}{q}\frac{1}{q}< \frac{1}{q} \text{ and } \theta>\frac{1}{2q}.
\end{align}
which, in conjunction with~\eqref{thess}, yields 
\begin{align}\label{cniu}
c\leq \frac{18(e-1)dL^3\eta^2q}{|\mathcal{S}_2|}\leq \frac{9(e-1)L}{200q}
\end{align}
where (i) follows from the fact that $(1+\theta)^q\leq (1+1/q)^q<e$ and $e$ is the Euler's number.
Since $g=4000 d\eta^2L^3q/|\mathcal{S}_2|$, we obtain from~\eqref{cniu} that $c/g\leq 9(e-1)/2000$. Then, we obtain from~\eqref{ppsse} that 
\begin{align*}
\lambda&\geq 0.144\eta,\;\chi= \frac{4}{K},  \tau\leq \frac{5\eta}{K},
\end{align*}
which, in conjunction with~\eqref{hilys}, implies that 
\begin{align}
\mathbb{E}\|\nabla f(\mathbf{x}^\zeta)\|^2\leq  \frac{140L(f_\beta(\mathbf{x}^0)-f_{\beta}(\mathbf{x}^*))}{(K+1)}+\frac{35}{K}\leq \mathcal{O}\left(\frac{1}{K}\right).
\end{align}
We choose $K=C\epsilon^{-1}$, where $C$ is a positive constant. Then, based on the above inequality, we have, for $C$ large enough, our Algorithm~\ref{ours:2}
achieves $\mathbb{E}\|f(\mathbf{x}^\zeta)\|^2\leq \epsilon$,  and  the total number of function queries  is 
\begin{align}\label{faac1}
\left\lceil \frac{K}{q}\right\rceil nd + K |\mathcal{S}_2|d\leq \mathcal{O}\left(nd+\frac{nd}{\epsilon n^{1/3}}+\frac{n^{2/3}d}{\epsilon}\right)=\mathcal{O}\left(nd+\frac{n^{2/3}d}{\epsilon}\right)\leq \mathcal{O}\left(\frac{n^{2/3}d}{\epsilon}\right)\leq \mathcal{O}\left( \frac{d}{\epsilon^{5/3}}  \right)
\end{align}
where the last two inequalities follow from the assumption that $n\leq K=C\epsilon^{-1}$.

Next, we suppose $n>K$. In this case, we have $|\mathcal{S}_1|=K$. Similarly to the case when $n\leq K$, we obtain 
\begin{align}
&c/g\leq \frac{9(e-1)}{2000},\; \lambda\geq 0.144\eta,\;\chi= \frac{9\sigma^2+4}{K}+\frac{18}{K^2}, \nonumber
\\ &\tau\leq\eta\left(\frac{18}{K^2}+\frac{9\sigma^2+4}{K}\right)+\frac{\eta}{K}\leq \frac{18\eta}{K^2}+ \frac{(9\sigma^2+5)\eta}{K}
\end{align}
which, in conjunction with~\eqref{hilys}, implies that 
\begin{align*}
\mathbb{E}\|\nabla f(\mathbf{x}^\zeta)\|^2\leq  \frac{140L(f_\beta(\mathbf{x}^0)-f_{\beta}(\mathbf{x}_\beta^*))}{(K+1)}+ \frac{125}{K^2}+\frac{63\sigma^2+35}{K}
\end{align*}
We choose $K=C\epsilon^{-1}$, where $C>0$ is a positive constant. Then, based on the above inequality, we have, for $C$ large enough,  our Algorithm~\ref{ours:2}
achieves $\mathbb{E}\|f(\mathbf{x}^\zeta)\|^2\leq \epsilon$,  and  the total number of function queries  is 
\begin{align*}
\left\lceil \frac{K}{q}\right\rceil Kd + K |\mathcal{S}_2|d&\leq  Kd+K^{5/3}d+K^{5/3}d+Kd=2K^{5/3}d+2Kd\leq \mathcal{O}(K^{5/3}d)\nonumber
\\&\leq\mathcal{O}(d\epsilon^{-5/3})\leq \mathcal{O}(\epsilon^{-1}n^{2/3}d),
\end{align*}
where the last inequality follows from the assumption that $n>K\geq C\epsilon^{-1}$. 

Combining the above two cases finishes the proof. 

\subsection{Proof of Corollary~\ref{co:svrg2}}
We prove two cases with  $n\leq \big\lceil (K/d)^{3/5}\big\rceil $ and $n> \big\lceil (K/d)^{3/5}\big\rceil $, separately. 

First we suppose $n\leq\big\lceil (K/d)^{3/5}\big\rceil$, and thus we have $|\mathcal{S}_1|=n$ and $q=nd$. Based on~\eqref{cchoose}, we have 
\begin{align}\label{ops1s}
c=9dL^3\eta^2\frac{(1+\theta)^q-1}{\theta}
\end{align}
where 
$\theta=\eta g+12\eta^2dL^2$. 
Based on the parameter selection in~\eqref{p22},  we have, 
\begin{align}\label{opis11}
\theta= \frac{1}{2q}+\frac{3}{100n^{1/3}}\frac{1}{q}< \frac{1}{q} \text{ and } \theta>\frac{1}{2q}.
\end{align}
Combining~\eqref{ops1s} and~\eqref{opis11} yields 
\begin{align}\label{ssqas}
c\leq18(e-1)d\eta^2q L^3\leq \frac{9(e-1)L}{200n^{1/3}}.
\end{align}
Since $g=4000d\eta^2qL^3$, we obtain from~\eqref{ssqas} that $c/g\leq 9(e-1)/2000$, which, in conjunction with~\eqref{ppsse} and~\eqref{p22}, implies that 
\begin{align*}
\lambda&\geq 0.144\eta,\;\chi= \frac{4n^{2/3}}{K}  \nonumber
\\ \tau&\leq\frac{4n^{2/3}\eta}{K}+\frac{n^{2/3}\eta}{K} <\frac{5\eta n^{2/3}}{K},
\end{align*}
which, in conjunction with~\eqref{hilys}, implies that  
\begin{align}
\mathbb{E}\|\nabla f(\mathbf{x}^\zeta)\|^2\leq  \frac{140 Ld n^{2/3}(f_\beta(\mathbf{x}^0)-f_{\beta}(\mathbf{x}_\beta^*))}{(K+1)}+\frac{35dn^{2/3}}{K}\leq \mathcal{O}\left(\frac{dn^{2/3}}{K}\right).
\end{align}
Let $K=\lceil C dn^{2/3} \epsilon^{-1} \rceil$ for a constant $C>0$, which, in conjunction with the assumption that $n\leq \big\lceil (K/d)^{3/5}\big\rceil$, implies that $n\leq \Theta(\epsilon^{-1})$.  Then, we have, for $C$ large enough,  $\mathbb{E}\|\nabla f(\mathbf{x}^\zeta)\|^2\leq \epsilon$, and the number of function queries is 
\begin{align}
\left\lceil \frac{K}{q}\right\rceil nd + K |\mathcal{S}_2|d\leq \mathcal{O}\left(nd+\frac{nd}{\epsilon n^{1/3}}+\frac{n^{2/3}d}{\epsilon}\right)=\mathcal{O}\left(nd+\frac{n^{2/3}d}{\epsilon}\right)\leq \mathcal{O}\left(\frac{n^{2/3}d}{\epsilon}\right)\leq \mathcal{O}\left( \frac{d}{\epsilon^{5/3}}  \right)
\end{align}
where the last two inequalities follow from the assumption that $n\leq \big\lceil (K/d)^{3/5}\big\rceil\leq \mathcal{O}(\epsilon^{-1})$.

Next, we suppose $n>\big\lceil (K/d)^{3/5}\big\rceil $, and thus we have $|\mathcal{S}_1|=\big\lceil (K/d)^{3/5}\big\rceil$. Similarly to the case when $n\leq \big\lceil (K/d)^{3/5}\big\rceil$, we obtain 
\begin{align}
&c/g\leq \frac{9(e-1)}{2000},\; \lambda \geq 0.144\eta,\;\chi\leq \frac{(9\sigma^2+4)d^{3/5}}{K^{3/5}}+\frac{18d^{6/5}}{K^{6/5}}, \nonumber
\\ &\tau\leq \eta\left(\frac{(9\sigma^2+5)d^{3/5}}{K^{3/5}} + \frac{18d^{6/5}}{K^{6/5}}\right)
\end{align}
which, in conjunction with~\eqref{hilys}, implies that  
\begin{align}\label{mops}
\mathbb{E}\|\nabla f(\mathbf{x}^\zeta)\|^2\leq&  \frac{140 Ld|\mathcal{S}_1|^{2/3}(f_\beta(\mathbf{x}^0)-f_{\beta}(\mathbf{x}_\beta^*))}{K}+\frac{(63\sigma^2+35)d|\mathcal{S}_1|^{2/3}}{K}+\frac{125d}{K|\mathcal{S}_1|^{1/3}} \nonumber
\\\leq& \mathcal{O} \left( \frac{d|\mathcal{S}_1|^{2/3}}{K}  \right)
\end{align}
where the first inequality follows from $|\mathcal{S}_1|=\big\lceil (K/d)^{3/5}\big\rceil$.  Let $K=Cd\epsilon^{-5/3}$, where $C>0$ is a large constant. Then, using~\eqref{mops} , we have, for $C$ large enough,  $\mathbb{E}\|\nabla f(\mathbf{x}^\zeta)\|^2\leq \epsilon$, and thus the number of function queries is 
\begin{align}
\left\lceil \frac{K}{q}\right\rceil |\mathcal{S}_1|d + K&\leq  \mathcal{O}\left(|\mathcal{S}_1|d+\frac{K}{q}|\mathcal{S}_1|d+K\right)\leq \mathcal{O}(K^{3/5}d^{2/5}+K)\leq\mathcal{O}(d\epsilon^{-5/3})\leq  \mathcal{O}\left(\frac{n^{2/3}d}{\epsilon}\right)
\end{align}
where the last two inequalities follow from $K=Cd\epsilon^{-5/3}$ and the assumption that $n>\big\lceil (K/
d)^{3/5}\big\rceil=C^{3/5}\epsilon^{-1}$. 

Combining the above two cases finishes the proof.

\section{Proof for ZO-SVRG-Coord}
\subsection{Proof of Lemma~\ref{le:newnew}}\label{appen:kaiyi}
For  any  $qk_0\leq  k\leq \min\{ q(k_0+1)-1, qh\},\,k_0=0,...,h,$ 
based on~\eqref{codnew}, we obtain 
\begin{align}\label{cords111}
\mathbb{E}&\|	\mathbf{v}^k-\hat \nabla_{\text{\normalfont coord}}f(\mathbf{x}^k)\|^2 \nonumber
\\&\leq 2\mathbb{E}\|\mathbf{v}^k-\mathbf{v}^{qk_0}-(\hat \nabla_{\text{\normalfont coord}}f(\mathbf{x}^k) -\hat \nabla_{\text{\normalfont coord}}f(\mathbf{x}^{qk_0})   )\|^2+2\mathbb{E}\|	\mathbf{v}^{qk_0}-\hat \nabla_{\text{\normalfont coord}}f(\mathbf{x}^{qk_0})\|^2. 
\end{align}
To simplify notation, we denote 
\begin{align}\label{newddefine}
G_{\mathcal{S}_2}(\mathbf{x})= \frac{1}{|\mathcal{S}_2|}\sum_{j=1}^{|\mathcal{S}_2|}\hat \nabla _{\text{coord}} f_{a_j}(\mathbf{x})\,\text{ and }\; H_{j}(\mathbf{x})=\hat \nabla _{\text{coord}} f_{a_j}(\mathbf{x}^t),
\end{align}
which, in conjunction with~\eqref{cords111}, implies that  
\begin{align}\label{gsss12}
\mathbb{E}\|	\mathbf{v}^k-\hat \nabla_{\text{\normalfont coord}}f(\mathbf{x}^k)\|^2=&2\underbrace{\mathbb{E}\left(\mathbb{E} \|G_{\mathcal{S}_2}(\mathbf{x}^k)-G_{\mathcal{S}_2}(\mathbf{x}^{qk_0})-(\hat \nabla_{\text{\normalfont coord}}f(\mathbf{x}^k) -\hat \nabla_{\text{\normalfont coord}}f(\mathbf{x}^{qk_0})   )\|^2\,|\,\mathbf{x}^{0},...,\mathbf{x}^{k} \right) }_{P}\nonumber
\\&+2\mathbb{E}\|	\mathbf{v}^{qk_0}-\hat \nabla_{\text{\normalfont coord}}f(\mathbf{x}^{qk_0})\|^2.
\end{align}
Conditioned on $\mathbf{x}^{0},...,\mathbf{x}^{k}$, we next provide an upper bound on the conditional expectation term $P$ in~\eqref{gsss12}. Using the  shorthand $\mathbb{E}_k(\cdot)$ to denote $\mathbb{E}(\cdot \,|\,\mathbf{x}^1,...,\mathbf{x}^k)$, we have  
\begin{small}
	\begin{align}\label{opps}
	\mathbb{E}_k \|G_{\mathcal{S}_2}&(\mathbf{x}^k)-G_{\mathcal{S}_2}(\mathbf{x}^{qk_0})-(\hat \nabla_{\text{\normalfont coord}}f(\mathbf{x}^k) -\hat \nabla_{\text{\normalfont coord}}f(\mathbf{x}^{qk_0})   )\|^2 \nonumber
	\\=&\mathbb{E}_k\bigg\| \frac{1}{|\mathcal{S}_2|}\sum_{j=1}^{|\mathcal{S}_2|}\big(  H_{j}(\mathbf{x}^k) -H_{j}(\mathbf{x}^{qk_0})-(\hat \nabla_{\text{\normalfont coord}}f(\mathbf{x}^k) -\hat \nabla_{\text{\normalfont coord}}f(\mathbf{x}^{qk_0})   ) \big)\bigg\|^2\nonumber
	\\=&\frac{1}{|\mathcal{S}_2|^2} \sum_{j=1}^{|\mathcal{S}_2|}\mathbb{E}_k\left\| H_{j}(\mathbf{x}^k) -H_{j}(\mathbf{x}^{qk_0})-(\hat \nabla_{\text{\normalfont coord}}f(\mathbf{x}^k) -\hat \nabla_{\text{\normalfont coord}}f(\mathbf{x}^{qk_0})   ) \right\|^2\nonumber
	\\&-2\sum_{i\neq j}\mathbb{E}_k\, \big\langle  H_{i}(\mathbf{x}^k) -H_{i}(\mathbf{x}^{qk_0})-(\hat \nabla_{\text{\normalfont coord}}f(\mathbf{x}^k) -\hat \nabla_{\text{\normalfont coord}}f(\mathbf{x}^{qk_0})   ),  \nonumber
	\\&\hspace{4cm}H_{j}(\mathbf{x}^k) -H_{j}(\mathbf{x}^{qk_0})-(\hat \nabla_{\text{\normalfont coord}}f(\mathbf{x}^k) -\hat \nabla_{\text{\normalfont coord}}f(\mathbf{x}^{qk_0})   )\big\rangle \nonumber
	\\\overset{\text{(i)}}=&\frac{1}{|\mathcal{S}_2|^2} \sum_{j=1}^{|\mathcal{S}_2|}\mathbb{E}_k\left\| H_{j}(\mathbf{x}^k) -H_{j}(\mathbf{x}^{qk_0})-(\hat \nabla_{\text{\normalfont coord}}f(\mathbf{x}^k) -\hat \nabla_{\text{\normalfont coord}}f(\mathbf{x}^{qk_0})   ) \right\|^2
	\end{align}
\end{small}
\hspace{-0.12cm}where (i) follows from the facts that  $a_i$ is independent of $a_j$ for any $i\neq j$,  $\mathbb{E}_k(H_{j}(\mathbf{x}^k))=\hat \nabla_{\text{\normalfont coord}}f(\mathbf{x}^k)$ and $\mathbb{E}_k(H_{j}(\mathbf{x}^{qk_0}))=\hat \nabla_{\text{\normalfont coord}}f(\mathbf{x}^{qk_0})$. Then, we further simplify~\eqref{opps} to  
\begin{small}
	\begin{align}\label{cordsss2}
	\frac{1}{|\mathcal{S}_2|}&\mathbb{E}_k\left\| H_{j}(\mathbf{x}^k) -H_{j}(\mathbf{x}^{qk_0})-(\hat \nabla_{\text{\normalfont coord}}f(\mathbf{x}^k) -\hat \nabla_{\text{\normalfont coord}}f(\mathbf{x}^{qk_0})   ) \right\|^2\nonumber
	\\&= \frac{1}{|\mathcal{S}_2|} \mathbb{E}_k\left\| H_{j}(\mathbf{x}^k) -H_{j}(\mathbf{x}^{qk_0})\right\|^2+\frac{1}{|\mathcal{S}_2|} \|\hat \nabla_{\text{\normalfont coord}}f(\mathbf{x}^k) -\hat \nabla_{\text{\normalfont coord}}f(\mathbf{x}^{qk_0}) \|^2\nonumber
	\\&\hspace{0.5cm}-\frac{2}{|\mathcal{S}_2|}\mathbb{E}_k \left \langle  H_{j}(\mathbf{x}^k)
	-H_{j}(\mathbf{x}^{qk_0}), \hat \nabla_{\text{\normalfont coord}}f(\mathbf{x}^k) -\hat \nabla_{\text{\normalfont coord}}f(\mathbf{x}^{qk_0}) \right\rangle  \nonumber
	\\&\overset{\text{(i)}}= \frac{1}{|\mathcal{S}_2|} \mathbb{E}_k\left\| H_{j}(\mathbf{x}^k) -H_{j}(\mathbf{x}^{qk_0})\right\|^2-\frac{1}{|\mathcal{S}_2|} \|\hat \nabla_{\text{\normalfont coord}}f(\mathbf{x}^k) -\hat \nabla_{\text{\normalfont coord}}f(\mathbf{x}^{qk_0}) \|^2  \nonumber
	\\&\overset{\text{(ii)}}\leq \frac{3}{|\mathcal{S}_2|}\mathbb{E}_k\left\| \hat \nabla _{\text{coord}} f_{a_j}(\mathbf{x}^k)- \nabla  f_{a_j}(\mathbf{x}^{k})\right\|^2 +\frac{3}{|\mathcal{S}_2|}\mathbb{E}_k\left\| \nabla  f_{a_j}(\mathbf{x}^{k})- \nabla  f_{a_j}(\mathbf{x}^{qk_0})\right\|^2 \nonumber
	\\&\hspace{0.5cm}+\frac{3}{|\mathcal{S}_2|}\mathbb{E}_k\left\| \hat \nabla _{\text{coord}} f_{a_j}(\mathbf{x}^{qk_0})- \nabla  f_{a_j}(\mathbf{x}^{qk_0})\right\|^2 \nonumber
	\\&\overset{\text{(iii)}}\leq\frac{6L^2d\delta^2}{|\mathcal{S}_2|}+\frac{3L^2}{|\mathcal{S}_2|}\|\mathbf{x}^k-\mathbf{x}^{qk_0}\|^2
	\end{align}
\end{small}
\hspace{-0.12cm}where (i) follows from the fact that $\mathbb{E}_k(H_{j}(\mathbf{x}^k))=\hat \nabla_{\text{\normalfont coord}}f(\mathbf{x}^k)$ and $\mathbb{E}_k(H_{j}(\mathbf{x}^{qk_0}))=\hat \nabla_{\text{\normalfont coord}}f(\mathbf{x}^{m-1})$, (ii) follows from the inequality that
$\|\mathbf{a+b+c}\|^2\leq 3(\|\mathbf{a}\|^2+\|\mathbf{b}\|^2+\|\mathbf{c}\|^2)$, and  (iii) follows from Lemma~\ref{coordinate} and Assumption~\ref{assumption}. 
Combining~\eqref{gsss12},~\eqref{cordsss2}, Lemma~\ref{evs1} and unconditioned on $\mathbf{x}^0,...,\mathbf{x}^k$, we have  
\begin{align*}
\mathbb{E}\|	\mathbf{v}^k-\hat \nabla_{\text{\normalfont coord}}f(\mathbf{x}^k)\|^2\leq & \frac{12L^2d\delta^2}{|\mathcal{S}_2|}+\frac{6L^2}{|\mathcal{S}_2|}\|\mathbf{x}^k-\mathbf{x}^{qk_0}\|^2+\frac{6I(|\mathcal{S}_1|<n)}{|\mathcal{S}_1|}\left(  2L^2d\delta^2+\sigma^2\right),
\end{align*}
which finishes the proof. 
\subsection{Proof of Theorem~\ref{th:newsa}}
Since $K=qh$ and 
$\nabla f_\beta(\mathbf{x})$ is $L$-Lipschitz, we have,  
for   $qm\leq k\leq q(m+1)-1,m=0,...,h-1$
\begin{align*}
f(\mathbf{x}^{k+1}) &\leq  f(\mathbf{x}^{k})+\langle \nabla f(\mathbf{x}^k),\mathbf{x}^{k+1}-\mathbf{x}^k \rangle +  \frac{L\eta^2}{2}\| \mathbf{v}^k\|^2  = f(\mathbf{x}^{k})-\eta\langle \nabla f(\mathbf{x}^k),\mathbf{v}^k \rangle +  \frac{L\eta^2}{2}\| \mathbf{v}^k\|^2. 
\end{align*}
Taking the expectation over the above inequality and noting  that $\mathbb{E}(\mathbf{v}^k\,|\,\mathbf{x}^0,..., \mathbf{x}^k)=\hat \nabla_{\text{\normalfont coord}}f(\mathbf{x}^k)-\hat \nabla_{\text{\normalfont coord}}f(\mathbf{x}^{qm})+\mathbf{v}^{qm}$, we have
\begin{align}\label{betaxsss}
\mathbb{E}f(\mathbf{x}^{k+1}) \leq& \mathbb{E} f(\mathbf{x}^{k})-\eta\mathbb{E} \langle \hat \nabla_{\text{\normalfont coord}}f(\mathbf{x}^k)-\hat \nabla_{\text{\normalfont coord}}f(\mathbf{x}^{qm})+\mathbf{v}^{qm} , \nabla f(\mathbf{x}^k)   \rangle   +  \frac{L\eta^2}{2}\mathbb{E}\| \mathbf{v}^k\|^2 \nonumber
\\\leq &\mathbb{E} f(\mathbf{x}^{k})-\eta\mathbb{E}\|\nabla f(\mathbf{x}^k)\|^2+
\eta\mathbb{E} \langle\nabla f(\mathbf{x}^k)-\hat \nabla_{\text{\normalfont coord}}f(\mathbf{x}^k)
+\hat \nabla_{\text{\normalfont coord}}f(\mathbf{x}^{qm})-\mathbf{v}^{qm} , \nabla f(\mathbf{x}^k)   \rangle   +  \frac{L\eta^2}{2}\mathbb{E}\| \mathbf{v}^k\|^2\nonumber
\\\leq & \mathbb{E} f(\mathbf{x}^{k})-\eta\mathbb{E}\|\nabla f(\mathbf{x}^k)\|^2+
\frac{\eta}{2}\mathbb{E}\|\nabla f(\mathbf{x}^k)-\hat \nabla_{\text{\normalfont coord}}f(\mathbf{x}^k)
+\hat \nabla_{\text{\normalfont coord}}f(\mathbf{x}^{qm})-\mathbf{v}^{qm} \|^2 \nonumber
\\&+\frac{\eta}{2}\mathbb{E}\|\nabla f(\mathbf{x}^k)\|^2       +  \frac{L\eta^2}{2}\mathbb{E}\| \mathbf{v}^k\|^2\nonumber
\\\leq & \mathbb{E} f(\mathbf{x}^{k})-\frac{\eta}{2}\mathbb{E}\|\nabla f(\mathbf{x}^k)\|^2+\eta\mathbb{E}\| \nabla f(\mathbf{x}^k)-\hat \nabla_{\text{\normalfont coord}}f(\mathbf{x}^k) \|^2 + 
\eta\mathbb{E}\| \hat \nabla_{\text{\normalfont coord}}f(\mathbf{x}^{qm})-\mathbf{v}^{qm} \|^2     +  \frac{L\eta^2}{2}\mathbb{E}\| \mathbf{v}^k\|^2 \nonumber
\\\overset{\text{(i)}}\leq&\mathbb{E} f(\mathbf{x}^{k})-\frac{\eta}{2} \|\nabla f(\mathbf{x}^k)\|^2+
\eta L^2d\delta^2    + \frac{3\eta I(|\mathcal{S}_1|<n)}{|\mathcal{S}_1|}\left(  2L^2d\delta^2+\sigma^2\right)+ \frac{L\eta^2}{2}\mathbb{E}\| \mathbf{v}^k\|^2
\end{align}
where (i) follows from Lemmas~\ref{coordinate} and~\ref{evs1}. 

We introduce a Lyapunov function $R_k^m=\mathbb{E}\left(  f(\mathbf{x}^k) +c^m_k\|\mathbf{x}^k-\mathbf{x}^{qm}\|^2 \right)$ for  $qm\leq k\leq q(m+1),m=0,...,h-1$, where $\{c_k^m\}$ are  constants such that $c^m_{q(m+1)}=0$. Then, we  obtain that for any $qm\leq k\leq q(m+1)-1$
\begin{align}\label{ck1sss}
R_{k+1}^m=&\mathbb{E}\left(  f(\mathbf{x}^{k+1}) +c^m_{k+1}\|\mathbf{x}^{k+1}-\mathbf{x}^{k}+\mathbf{x}^k-\mathbf{x}^{qm}\|^2 \right)\nonumber
\\\leq&\mathbb{E} f(\mathbf{x}^{k+1})+c^m_{k+1}\eta^2\mathbb{E}\|\mathbf{v}^k\|^2+c^m_{k+1}\mathbb{E}\|\mathbf{x}^k-\mathbf{x}^{qm}\|^2-2c^m_{k+1}\eta\mathbb{E}\langle\mathbf{v}^k,  \mathbf{x}^k-\mathbf{x}^{qm} \rangle  \nonumber
\\\overset{\text{(i)}}=&\mathbb{E} f(\mathbf{x}^{k+1})+c^m_{k+1}\eta^2\mathbb{E}\|\mathbf{v}^k\|^2+c^m_{k+1}\mathbb{E}\|\mathbf{x}^k-\mathbf{x}^{qm}\|^2-2c^m_{k+1}\eta\mathbb{E}\langle \hat \nabla_{\text{\normalfont coord}}f(\mathbf{x}^k)-\hat \nabla_{\text{\normalfont coord}}f(\mathbf{x}^{qm})+\mathbf{v}^{qm},  \mathbf{x}^k-\mathbf{x}^{qm} \rangle \nonumber
\\\leq&\mathbb{E} f(\mathbf{x}^{k+1})+c^m_{k+1}\eta^2\mathbb{E}\|\mathbf{v}^k\|^2+(c^m_{k+1}+c^m_{k+1}\eta g)\mathbb{E}\|\mathbf{x}^k-\mathbf{x}^{qm}\|^2 \nonumber
\\&+\frac{2c^m_{k+1}\eta}{g}\mathbb{E}\left( \|\hat \nabla_{\text{\normalfont coord}}f(\mathbf{x}^k)\|^2+\|\hat \nabla_{\text{\normalfont coord}}f(\mathbf{x}^{qm})-\mathbf{v}^{qm}\|^2\right) \nonumber
\\\overset{\text{(ii)}}\leq&\mathbb{E} f(\mathbf{x}^{k+1})+c^m_{k+1}\eta^2\mathbb{E}\|\mathbf{v}^k\|^2+(c^m_{k+1}+c^m_{k+1}\eta g)\mathbb{E}\|\mathbf{x}^k-\mathbf{x}^{qm}\|^2 \nonumber
\\&+\frac{4c^m_{k+1}\eta}{g}\left(\mathbb{E} \| \nabla f(\mathbf{x}^k)\|^2+L^2d\delta^2\right)+\frac{12c^m_{k+1}\eta}{g}  \frac{ I(|\mathcal{S}_1|<n)}{|\mathcal{S}_1|}\left(  2L^2d\delta^2+\sigma^2\right).
\end{align}
where (i) follows from the definition of $\mathbf{v}^k$ and (ii) follows from Lemma~\ref{le:newnew}.
Combining~\eqref{betaxsss} and~\eqref{ck1sss}, we obtain 
\begin{align}\label{ninis}
R_{k+1}^m\leq& \mathbb{E} f(\mathbf{x}^{k})-\left(\frac{\eta}{2}-\frac{4c^m_{k+1}\eta}{g}\right) \|\nabla f(\mathbf{x}^k)\|^2 + \left(\frac{L}{2}+c^m_{k+1}\right)\eta^2\mathbb{E}\| \mathbf{v}^k\|^2 +(c^m_{k+1}+c^m_{k+1}\eta g)\mathbb{E}\|\mathbf{x}^k-\mathbf{x}^{qm}\|^2 \nonumber
\\&+\frac{4c^m_{k+1}\eta}{g}L^2d\delta^2+\eta L^2d\delta^2   +\left(\frac{12c^m_{k+1}\eta}{g}+3\eta\right)  \frac{ I(|\mathcal{S}_1|<n)}{|\mathcal{S}_1|}\left(  2L^2d\delta^2+\sigma^2\right).
\end{align}
Based on Lemma~\ref{le:newnew}, we obtain
\begin{align*}
\mathbb{E}\|\mathbf{v}^k\|^2\leq&3\mathbb{E}\|\mathbf{v}^k-\hat \nabla_{\text{\normalfont coord}}f(\mathbf{x}^k)\|^2  + 3\mathbb{E}\|\nabla f(\mathbf{x}^k)-\hat \nabla_{\text{\normalfont coord}}f(\mathbf{x}^k)\|^2  +3\mathbb{E}\|\nabla f(\mathbf{x}^k)\|^2   \nonumber
\\\leq&  \frac{36L^2d\delta^2}{|\mathcal{S}_2|}+\frac{18L^2}{|\mathcal{S}_2|}\|\mathbf{x}^k-\mathbf{x}^{qm}\|^2+\frac{18I(|\mathcal{S}_1|<n)}{|\mathcal{S}_1|}\left(  2L^2d\delta^2+\sigma^2\right)+3L^2d\delta^2 +3\mathbb{E}\|\nabla f(\mathbf{x}^k)\|^2,
\end{align*}
which, in conjunction with~\eqref{ninis}, implies that  
\begin{align}\label{ninisss}
R_{k+1}^m\leq& \mathbb{E} f(\mathbf{x}^{k})-\left(\frac{\eta}{2}-\frac{4c^m_{k+1}\eta}{g}-3\left(\frac{L}{2}+c^m_{k+1}\right)\eta^2\right) \|\nabla f(\mathbf{x}^k)\|^2 \nonumber
\\&+\left(c^m_{k+1}+c^m_{k+1}\eta g+\frac{18L^2}{|\mathcal{S}_2|} \left(\frac{L}{2}+c^m_{k+1}\right)\eta^2\right)\mathbb{E}\|\mathbf{x}^k-\mathbf{x}^{qm}\|^2 \nonumber
\\& +\left(\frac{12c^m_{k+1}\eta}{g}+3\eta+18 \left(\frac{L}{2}+c^m_{k+1}\right)\eta^2\right)  \frac{ I(|\mathcal{S}_1|<n)}{|\mathcal{S}_1|}\left(  2L^2d\delta^2+\sigma^2\right)\nonumber
\\&+\frac{4c^m_{k+1}\eta}{g}L^2d\delta^2+\eta L^2d\delta^2  +\left(\frac{36}{|\mathcal{S}_2|}+3\right) \left(\frac{L}{2}+c^m_{k+1}\right)\eta^2L^2d\delta^2.
\end{align}
Let $c_k^m:=(1+\theta)c^m_{k+1}+\frac{9L^3\eta^2}{|\mathcal{S}_2|}$, where $\theta=\eta g+\frac{18L^2\eta^2}{|\mathcal{S}_2|}$. Then, we rewrite~\eqref{ninisss} as 
\begin{align}\label{nipsas}
R_{k+1}^m\leq& R_{k}^m-\left(\frac{\eta}{2}-\frac{4c^m_{k+1}\eta}{g}- 3\left(\frac{L}{2}+c^m_{k+1}\right)\eta^2\right)\mathbb{E} \|\nabla f(\mathbf{x}^k)\|^2 \nonumber
\\& +\left(\frac{12c^m_{k+1}\eta}{g}+3\eta+18 \left(\frac{L}{2}+c^m_{k+1}\right)\eta^2\right)  \frac{ I(|\mathcal{S}_1|<n)}{|\mathcal{S}_1|}\left(  2L^2d\delta^2+\sigma^2\right)\nonumber
\\&+\frac{4c^m_{k+1}\eta}{g}L^2d\delta^2+\eta L^2d\delta^2  +\left(\frac{36}{|\mathcal{S}_2|}+3\right) \left(\frac{L}{2}+c^m_{k+1}\right)\eta^2L^2d\delta^2.
\end{align}
Note that  for $qm\leq k\leq q(m+1)-1,m=0,...,h-1$
\begin{align*}
c_k^m\leq c=\frac{9L^3\eta^2}{|\mathcal{S}_2|}\frac{(1+\theta)^q-1}{\theta}.
\end{align*}
To simplify notation, we define 
\begin{align}\label{papass}
\lambda &=\frac{\eta}{2}-\frac{4c\eta}{g}-3\left( \frac{L}{2} +c \right)\eta^2 \nonumber
\\\chi &=  \left(\frac{12c\eta}{g}+3\eta+18 \left(\frac{L}{2}+c\right)\eta^2\right) \frac{I(|\mathcal{S}_1|<n)}{|\mathcal{S}_1|}\left(  2L^2d\delta^2+\sigma^2\right) \nonumber
\\\tau &=\left( \frac{4c}{g} +1\right)\eta L^2d\delta^2+\left( \frac{36}{|\mathcal{S}_2|}+3 \right)\left(\frac{L}{2}+c\right)\eta^2L^2d\delta^2,
\end{align}
which, in conjunction with~\eqref{nipsas}, implies that   
\begin{align*}
R_{k+1}^m\leq& R_k^m-\lambda \mathbb{E}\|\nabla f(\mathbf{x}^k)\|^2+\tau+\chi.
\end{align*}
Telescoping the above inequality over $k$ from $qm$ to $q(m+1)-1$ and noting that $R_{qm}^m=\mathbb{E}f(\mathbf{x}^{qm})$ and $R_{q(m+1)}^m=\mathbb{E} f(\mathbf{x}^{q(m+1)})$, we obtain 
\begin{align*}
\mathbb{E}f(\mathbf{x}^{q(m+1)})\leq \mathbb{E}f (\mathbf{x}^{qm})-\lambda \sum_{k=qm}^{q(m+1)-1} \mathbb{E}\|\nabla f(\mathbf{x}^k)\|^2+q\tau+q\chi
\end{align*}
Then, telescoping the above inequality over $m$ from $0$ to $h-1$,  we obtain 
\begin{align*}
\mathbb{E}f (\mathbf{x}^{K})\leq \mathbb{E}f (\mathbf{x}^{0})-\lambda  \sum_{k=0}^K\mathbb{E}\|\nabla f(\mathbf{x}^k)\|^2+K\tau+K\chi,
\end{align*}
which, in conjunction with the definition of $\mathbf{x}^\zeta$, implies that 
\begin{align}\label{zetasss}
\mathbb{E}\|\nabla f(\mathbf{x}^\zeta)\|^2 \leq \frac{\Delta}{\lambda K}+\frac{\tau+\chi}{\lambda}.
\end{align}
where $\Delta:=f(\mathbf{x}^0)-f(\mathbf{x}^*)$. 

Let $g=1/(2\eta q)$. Then, based on the selected parameters in~\eqref{ppsassss} and the definition of $\theta$, we have 
\begin{align*}
\frac{1}{2q}<\theta\leq \frac{1}{2q}+\frac{2}{25q}\frac{1}{q}\leq\frac{1}{q},
\end{align*}
which, in conjunction with the definition of $c$, implies that 
\begin{align}\label{cvalue}
c\leq \frac{18(e-1)L^3\eta^2q}{|\mathcal{S}_2|}\leq \frac{2(e-1)L}{25 q},
\end{align}
and $c/g\leq 0.02$.

Next, we prove two cases when $n\leq K$ and $n>K$, separately. First suppose $n\leq K$. In such a case, we have $|\mathcal{S}_1|=n$ and $q=\lceil n^{1/3}\rceil$. Then, based on~\eqref{papass},~\eqref{ppsassss} and~\eqref{cvalue}, we obtain 
\begin{align*}
\lambda&\geq 0.22\eta,\; \chi=  0, \;\tau<\frac{5\eta}{K},
\end{align*} 
which, in conjunction with~\eqref{zetasss}, yields 
\begin{align}
\mathbb{E}\|\nabla f(\mathbf{x}^\zeta)\|^2 \leq \frac{69\Delta+23}{K}\leq \mathcal{O}\left( \frac{1}{K} \right).
\end{align} 
Let $K=C\epsilon^{-1}$, where $C$ is a constant.   Then, we have, for $C$ large enough,  $\mathbb{E}\|\nabla f(\mathbf{x}^\zeta)\|^2\leq \epsilon$, and the number of function queries is 
\begin{align}
\left\lceil\frac{K}{q}\right\rceil nd+K|\mathcal{S}_2|d\leq \mathcal{O}\left( nd+\frac{dn^{2/3}}{\epsilon}\right)\leq \mathcal{O}\left( \frac{dn^{2/3}}{\epsilon}\right)\leq\mathcal{O}\left(\min\left\{\frac{n^{2/3}d}{\epsilon}, \frac{d}{\epsilon^{5/3}} \right\}\right)
\end{align}
where the last two inequalities follow from the assumption that $n\leq K=C\epsilon^{-1}$.

Next, we suppose $n>K$. In this case, we obtain 
\begin{align*}
\lambda&\geq 0.22\eta,\;\chi\leq \frac{5\eta}{K}\left( \frac{2}{K}+\sigma^2\right),\;\tau\leq \frac{5\eta}{K} 
\end{align*}
which, in conjunction with~\eqref{zetasss}, yields 
\begin{align}
\mathbb{E}\|\nabla f(\mathbf{x}^\zeta)\|^2 \leq \frac{69\Delta+23+23\sigma^2}{K}+\frac{46}{K^2}
\end{align} 
Let $K=C\epsilon^{-1}$, where $C$ is a constant.   Then, we have, for $C$ large enough,  $\mathbb{E}\|\nabla f(\mathbf{x}^\zeta)\|^2\leq \epsilon$, and the number of function queries is 
\begin{align}
\left\lceil\frac{K}{q}\right\rceil Kd+K|\mathcal{S}_2|d \leq \mathcal{O}\left(dK^{5/3}\right)\leq \mathcal{O}\left(\min\left\{\frac{n^{2/3}d}{\epsilon}, \frac{d}{\epsilon^{5/3}} \right\}\right)
\end{align}
where the last inequality follows from the assumption that $n>K=C\epsilon^{-1}$. 

Combining the above two cases finish the proof. 

\section{ Proofs for ZO-SPIDER-Coord}
\subsection{Auxiliary Lemma}
The following lemma provides an upper bound on the error of $\mathbf{v}^k$ for estimating the second moment of $\|	\mathbf{v}^k-\hat \nabla_{\text{\normalfont coord}}f(\mathbf{x}^k)\|$.
\begin{lemma}\label{le:coord}
	For  any given $k_0\leq \lfloor K /q\rfloor$ and $qk_0\leq k\leq \min\{q(k_0+1)-1,K\}$, we have 
	\begin{small}
		\begin{align}\label{bonusB}
		&\mathbb{E}\|	\mathbf{v}^k-\hat \nabla_{\text{\normalfont coord}}f(\mathbf{x}^k)\|^2\leq\frac{3\eta^2L^2}{|\mathcal{S}_2|}\sum_{t=qk_0}^{k-1}\mathbb{E} \|  \mathbf{v}^{t}\|^2 +(k-qk_0)\frac{6L^2d\delta^2}{|\mathcal{S}_2|}+\frac{3I(|\mathcal{S}_1|<n)}{|\mathcal{S}_1|}\left(  2L^2d\delta^2+\sigma^2\right).
		\end{align}
	\end{small}
	\hspace{-0.15cm}	where  we define $\sum_{t=qk_0}^{qk_0-1}\mathbb{E}\|\mathbf{v}^t\|^2=0$ for simplicity. 
\end{lemma}

\begin{proof}	
	First we consider the case when $k\geq qk_0+1$. 
	For $qk_0+1\leq m\leq k$, we have 
	\begin{align}\label{martin}
	\mathbf{v}^m-\hat \nabla_{\text{\normalfont coord}}f(\mathbf{x}^m)=\mathbf{v}^{qk_0}-\hat \nabla_{\text{\normalfont coord}}f(\mathbf{x}^{qk_0})+\sum_{t=qk_0+1}^m(\mathbf{v}^t-\mathbf{v}^{t-1}-(\hat \nabla_{\text{\normalfont coord}}f(\mathbf{x}^t) -\hat \nabla_{\text{\normalfont coord}}f(\mathbf{x}^{t-1}   )  )).
	\end{align}
	Recall that  $\mathbf{v}^t$ is given by 
	\begin{align}\label{tyu}
	\mathbf{v}^t=\frac{1}{|\mathcal{S}_2|}\sum_{j=1}^{|\mathcal{S}_2|}\hat \nabla _{\text{coord}} f_{a_j}(\mathbf{x}^t)- \frac{1}{|\mathcal{S}_2|}\sum_{j=1}^{|\mathcal{S}_2|}\hat \nabla _{\text{coord}} f_{a_j}(\mathbf{x}^{t-1})+
	\mathbf{v}^{t-1}.
	\end{align}
	We then have for any $qk_0\leq t\leq m$ , $\mathbb{E}(\mathbf{v}^t-\mathbf{v}^{t-1}-(\hat \nabla_{\text{\normalfont coord}}f(\mathbf{x}^t) -\hat \nabla_{\text{\normalfont coord}}f(\mathbf{x}^{t-1}   ))\,| \mathbf{x}^0,...,\mathbf{x}^t\,)=0$,  
	which, in conjunction with~\eqref{martin}, implies that the sequence $(\mathbf{v}^t-\hat \nabla_{\text{\normalfont coord}}f(\mathbf{x}^t),t=qk_0,...,m )$ is a martingale. Then, based on the property of square-integrable martingales~\citep{fang2018spider}, we can obtain, for $qk_0+1\leq m\leq k$, 
	\begin{align*}
	\mathbb{E}\|	\mathbf{v}^m-\hat \nabla_{\text{\normalfont coord}}f(\mathbf{x}^m)\|^2=&\mathbb{E}\|\mathbf{v}^{qk_0}-\hat \nabla_{\text{\normalfont coord}}f(\mathbf{x}^{qk_0})\|^2 \nonumber
	\\&	+\sum_{t=qk_0+1}^m\mathbb{E}\|\mathbf{v}^t-\mathbf{v}^{t-1}-(\hat \nabla_{\text{\normalfont coord}}f(\mathbf{x}^t) -\hat \nabla_{\text{\normalfont coord}}f(\mathbf{x}^{t-1}   ))\|^2.
	\end{align*}
	The above equality further 
	implies that 
	\begin{align}\label{cords}
	\mathbb{E}&\|	\mathbf{v}^m-\hat \nabla_{\text{\normalfont coord}}f(\mathbf{x}^m)\|^2 \nonumber
	\\&=\mathbb{E}\|\mathbf{v}^m-\mathbf{v}^{m-1}-(\hat \nabla_{\text{\normalfont coord}}f(\mathbf{x}^m) -\hat \nabla_{\text{\normalfont coord}}f(\mathbf{x}^{m-1})   )\|^2+\mathbb{E}\|	\mathbf{v}^{m-1}-\hat \nabla_{\text{\normalfont coord}}f(\mathbf{x}^{m-1})\|^2.
	\end{align}
	Based on~\eqref{tyu} and using the same notations as in~\eqref{newddefine}, we have
	$\mathbf{v}^m-\mathbf{v}^{m-1}=G_{\mathcal{S}_2}(\mathbf{x}^m)-G_{\mathcal{S}_2}(\mathbf{x}^{m-1}),$
	which, in conjunction with~\eqref{cords}, implies
	\begin{align}\label{gs2}
	\mathbb{E}\|	\mathbf{v}^m-\hat \nabla_{\text{\normalfont coord}}f(\mathbf{x}^m)\|^2=&\underbrace{\mathbb{E}\left(\mathbb{E} \|G_{\mathcal{S}_2}(\mathbf{x}^m)-G_{\mathcal{S}_2}(\mathbf{x}^{m-1})-(\hat \nabla_{\text{\normalfont coord}}f(\mathbf{x}^m) -\hat \nabla_{\text{\normalfont coord}}f(\mathbf{x}^{m-1})   )\|^2\,|\,\mathbf{x}^{0},...,\mathbf{x}^{m} \right) }_{Q}\nonumber
	\\&+\mathbb{E}\|	\mathbf{v}^{m-1}-\hat \nabla_{\text{\normalfont coord}}f(\mathbf{x}^{m-1})\|^2.
	\end{align}
	Conditioned on $\mathbf{x}^{0},...,\mathbf{x}^{m}$, we next provide an upper bound on the conditional expectation term $Q$ in~\eqref{gs2}. Using the  shorthand $\mathbb{E}_m(\cdot)$ to denote $\mathbb{E}(\cdot \,|\,\mathbf{x}^1,...,\mathbf{x}^m)$, we have  
	\begin{align*}
	\mathbb{E}_m \|G_{\mathcal{S}_2}&(\mathbf{x}^m)-G_{\mathcal{S}_2}(\mathbf{x}^{m-1})-(\hat \nabla_{\text{\normalfont coord}}f(\mathbf{x}^m) -\hat \nabla_{\text{\normalfont coord}}f(\mathbf{x}^{m-1})   )\|^2
	\\=&\mathbb{E}_m\bigg\| \frac{1}{|\mathcal{S}_2|}\sum_{j=1}^{|\mathcal{S}_2|}\big(  H_{j}(\mathbf{x}^m) -H_{j}(\mathbf{x}^{m-1})-(\hat \nabla_{\text{\normalfont coord}}f(\mathbf{x}^m) -\hat \nabla_{\text{\normalfont coord}}f(\mathbf{x}^{m-1})   ) \big)\bigg\|^2
	\\=&\frac{1}{|\mathcal{S}_2|^2} \sum_{j=1}^{|\mathcal{S}_2|}\mathbb{E}_m\left\| H_{j}(\mathbf{x}^m) -H_{j}(\mathbf{x}^{m-1})-(\hat \nabla_{\text{\normalfont coord}}f(\mathbf{x}^m) -\hat \nabla_{\text{\normalfont coord}}f(\mathbf{x}^{m-1})   ) \right\|^2
	\\&-2\sum_{i\neq j}\mathbb{E}_m\, \big\langle  H_{i}(\mathbf{x}^m) -H_{i}(\mathbf{x}^{m-1})-(\hat \nabla_{\text{\normalfont coord}}f(\mathbf{x}^m) -\hat \nabla_{\text{\normalfont coord}}f(\mathbf{x}^{m-1})   ), 
	\\&\hspace{4cm}H_{j}(\mathbf{x}^m) -H_{j}(\mathbf{x}^{m-1})-(\hat \nabla_{\text{\normalfont coord}}f(\mathbf{x}^m) -\hat \nabla_{\text{\normalfont coord}}f(\mathbf{x}^{m-1})   )\big\rangle
	\\\overset{\text{(i)}}=&\frac{1}{|\mathcal{S}_2|^2} \sum_{j=1}^{|\mathcal{S}_2|}\mathbb{E}_m\left\| H_{j}(\mathbf{x}^m) -H_{j}(\mathbf{x}^{m-1})-(\hat \nabla_{\text{\normalfont coord}}f(\mathbf{x}^m) -\hat \nabla_{\text{\normalfont coord}}f(\mathbf{x}^{m-1})   ) \right\|^2
	\end{align*}
	where (i) follows from the facts that  $a_i$ is independent of $a_j$ for any $i\neq j$,  $\mathbb{E}_m(H_j(\mathbf{x}^m))=\hat \nabla_{\text{\normalfont coord}}f(\mathbf{x}^m)$, and $\mathbb{E}_m(H_j(\mathbf{x}^{m-1}))=\hat \nabla_{\text{\normalfont coord}}f(\mathbf{x}^{m-1})$. 
	Then, we further simplify the above equation to 
	\begin{align}
	\frac{1}{|\mathcal{S}_2|}&\mathbb{E}_m\left\| H_{j}(\mathbf{x}^m) -H_{j}(\mathbf{x}^{m-1})-(\hat \nabla_{\text{\normalfont coord}}f(\mathbf{x}^m) -\hat \nabla_{\text{\normalfont coord}}f(\mathbf{x}^{m-1})   ) \right\|^2\nonumber
	\\&= \frac{1}{|\mathcal{S}_2|} \mathbb{E}_m\left\| H_{j}(\mathbf{x}^m) -H_{j}(\mathbf{x}^{m-1})\right\|^2+\frac{1}{|\mathcal{S}_2|} \|\hat \nabla_{\text{\normalfont coord}}f(\mathbf{x}^m) -\hat \nabla_{\text{\normalfont coord}}f(\mathbf{x}^{m-1}) \|^2\nonumber
	\\&\hspace{0.5cm}-\frac{2}{|\mathcal{S}_2|}\mathbb{E}_m \left \langle  H_{j}(\mathbf{x}^m)
	-H_{j}(\mathbf{x}^{m-1}), \hat \nabla_{\text{\normalfont coord}}f(\mathbf{x}^m) -\hat \nabla_{\text{\normalfont coord}}f(\mathbf{x}^{m-1}) \right\rangle  \nonumber
	\\&\overset{\text{(i)}}= \frac{1}{|\mathcal{S}_2|} \mathbb{E}_m\left\| H_{j}(\mathbf{x}^m) -H_{j}(\mathbf{x}^{m-1})\right\|^2-\frac{1}{|\mathcal{S}_2|} \|\hat \nabla_{\text{\normalfont coord}}f(\mathbf{x}^m) -\hat \nabla_{\text{\normalfont coord}}f(\mathbf{x}^{m-1}) \|^2  \nonumber
	\\&\leq \frac{1}{|\mathcal{S}_2|} \mathbb{E}_m\left\| H_{j}(\mathbf{x}^m) -H_{j}(\mathbf{x}^{m-1})\right\|^2\nonumber
	\end{align}
	\begin{align}\label{cords2}
	\hspace{1.7cm}&\overset{\text{(ii)}}\leq \frac{3}{|\mathcal{S}_2|}\mathbb{E}_m\left\| \hat \nabla _{\text{coord}} f_{a_j}(\mathbf{x}^m)- \nabla  f_{a_j}(\mathbf{x}^{m})\right\|^2 +\frac{3}{|\mathcal{S}_2|}\mathbb{E}_m\left\| \nabla  f_{a_j}(\mathbf{x}^{m})- \nabla  f_{a_j}(\mathbf{x}^{m-1})\right\|^2 \nonumber
	\\&\hspace{0.5cm}+\frac{3}{|\mathcal{S}_2|}\mathbb{E}_m\left\| \hat \nabla _{\text{coord}} f_{a_j}(\mathbf{x}^{m-1})- \nabla  f_{a_j}(\mathbf{x}^{m-1})\right\|^2 \nonumber
	\\&\overset{\text{(iii)}}\leq\frac{6L^2d\delta^2}{|\mathcal{S}_2|}+\frac{3L^2}{|\mathcal{S}_2|}\|\mathbf{x}^m-\mathbf{x}^{m-1}\|^2=\frac{6L^2d\delta^2}{|\mathcal{S}_2|}+\frac{3\eta^2L^2}{|\mathcal{S}_2|}\|\mathbf{v}^{m-1}\|^2
	\end{align}
	where (i) follows from the fact that $\mathbb{E}_m(H_{j}(\mathbf{x}^m))=\hat \nabla_{\text{\normalfont coord}}f(\mathbf{x}^m)$ and $\mathbb{E}_m(H_{j}(\mathbf{x}^{m-1}))=\hat \nabla_{\text{\normalfont coord}}f(\mathbf{x}^{m-1})$, (ii) follows from the inequality that
	$\|\mathbf{a+b+c}\|^2\leq 3(\|\mathbf{a}\|^2+\|\mathbf{b}\|^2+\|\mathbf{c}\|^2)$, and  (iii) follows from Lemma~\ref{coordinate} and Assumption~\ref{assumption}. Combining~\eqref{cords} and~\eqref{cords2}  and unconditioned on $\mathbf{x}^0,...,\mathbf{x}^m$, we obtain 
	\begin{align}\label{nuews}
	\mathbb{E}&\|	\mathbf{v}^m-\hat \nabla_{\text{\normalfont coord}}f(\mathbf{x}^m)\|^2 \leq\frac{6L^2d\delta^2}{|\mathcal{S}_2|}+\frac{3\eta^2L^2}{|\mathcal{S}_2|}\|\mathbf{v}^{m-1}\|^2+\mathbb{E}\|	\mathbf{v}^{m-1}-\hat \nabla_{\text{\normalfont coord}}f(\mathbf{x}^{m-1})\|^2.
	\end{align}
	Telescoping the above inequality over $m$ from $qk_0+1$ to k, we obtain 
	\begin{align}\label{ssc1s}
	\mathbb{E}\|	\mathbf{v}^k-\hat \nabla_{\text{\normalfont coord}}f(\mathbf{x}^k)\|^2&\leq\frac{3\eta^2L^2}{|\mathcal{S}_2|}\sum_{t=qk_0+1}^k\mathbb{E} \|  \mathbf{v}^{t-1}\|^2+(k-qk_0)\frac{6L^2d\delta^2}{|\mathcal{S}_2|}\nonumber
	\\&+\mathbb{E}\|	\mathbf{v}^{qk_0}-\hat \nabla_{\text{\normalfont coord}}f(\mathbf{x}^{qk_0})\|^2.
	\end{align}
	Using  Lemma~\ref{evs1} 
	and ~\eqref{ssc1s}  yields~\eqref{bonusB}.  
	For the case when $k=qk_0$, it can be checked that~\eqref{bonusB} also holds.  
\end{proof}

\subsection{Proof of Theorem~\ref{mainTT}}
Noting that $f(\cdot)$ has a $L$-Lipschitz  gradient, we have, for  any given $k_0\leq \lfloor K /q\rfloor$ and $qk_0\leq m\leq \min\{q(k_0+1)-1,K\}$, 
\begin{align}
f(\mathbf{x}^{m+1})&\leq f(\mathbf{x}^{m}) +\langle  \nabla f(\mathbf{x}^m),\mathbf{x}^{m+1}-\mathbf{x}^{m} \rangle + \frac{L}{2}\|\mathbf{x}^{m+1}-\mathbf{x}^{m} \|^2 \nonumber
\\&=f(\mathbf{x}^m)-\eta\langle \nabla f(\mathbf{x}^m) -\mathbf{v}^m,\mathbf{v}^m  \rangle -\eta\|\mathbf{v}^m\|^2+\frac{L\eta^2}{2}\|\mathbf{v}^m\|^2  \nonumber
\\&\leq f(\mathbf{x}^m) +\frac{\eta}{2}\|\nabla f(\mathbf{x}^m) -\mathbf{v}^m\|^2+\frac{\eta}{2}\|\mathbf{v}^m\|^2-\eta\|\mathbf{v}^m\|^2+\frac{L\eta^2}{2}\|\mathbf{v}^m\|^2  \nonumber
\\&= f(\mathbf{x}^m) +\frac{\eta}{2}\|\nabla f(\mathbf{x}^m) -\mathbf{v}^m\|^2-\Big(\frac{\eta}{2}-\frac{L\eta^2}{2}\Big)\|\mathbf{v}^m\|^2. \nonumber
\end{align}
Taking the expectation over the above inequality yields
\begin{align*}
\mathbb{E}f(\mathbf{x}^{m+1}) &\leq \mathbb{E} f(\mathbf{x}^m) +\eta \left(\mathbb{E}\|\hat \nabla_{\text{\normalfont coord}}f(\mathbf{x}^m) -\mathbf{v}^m\|^2+\mathbb{E}\| \nabla f(\mathbf{x}^m) -\hat \nabla_{\text{\normalfont coord}}f(\mathbf{x}^m)\|^2 \right)-\Big(\frac{\eta}{2}-\frac{L\eta^2}{2}\Big)\mathbb{E}\|\mathbf{v}^m\|^2  \nonumber
\\&\leq \mathbb{E} f(\mathbf{x}^m) +\eta \left(\mathbb{E}\|\hat \nabla_{\text{\normalfont coord}}f(\mathbf{x}^m) -\mathbf{v}^m\|^2+L^2d\delta^2 \right)-\Big(\frac{\eta}{2}-\frac{L\eta^2}{2}\Big)\mathbb{E}\|\mathbf{v}^m\|^2,
\end{align*}
which, in conjunction with Lemma~\ref{le:coord}, implies that 
\begin{align}\label{simnf}
\mathbb{E}f(\mathbf{x}^{m+1})\leq& \mathbb{E} f(\mathbf{x}^m) + \frac{3\eta^3L^2}{|\mathcal{S}_2|}\sum_{t=qk_0}^{m-1}\mathbb{E} \|  \mathbf{v}^{t}\|^2+(m-qk_0)\frac{6\eta L^2d\delta^2}{|\mathcal{S}_2|}+\frac{3\eta I(|\mathcal{S}_1|<n)}{|\mathcal{S}_1|}\left(  2L^2d\delta^2+\sigma^2\right)+\eta L^2d\delta^2 \nonumber
\\& -\Big(\frac{\eta}{2}-\frac{L\eta^2}{2}\Big)\mathbb{E}\|\mathbf{v}^m\|^2. 
\end{align}
To simplify notation, we define 
\begin{align}\label{pisw}
\pi(\mathcal{S}_1,\delta)=\frac{3 I(|\mathcal{S}_1|<n)}{|\mathcal{S}_1|}\left(  2L^2d\delta^2+\sigma^2\right)+ L^2d\delta^2.
\end{align}
Then, telescoping~\eqref{simnf} over $m$ from  $qk_0$ to $ k$ yields 
\begin{small}
	\begin{align}\label{kko}
	\mathbb{E}f(\mathbf{x}^{k+1})\leq &\mathbb{E}f(\mathbf{x}^{qk_0})+\frac{3L^2\eta^3}{|\mathcal{S}_2|}\sum_{t_1=qk_0}^k\sum_{t_2=qk_0}^{t_1-1}\mathbb{E}\|\mathbf{v}^{t_2}\|^2+\sum_{t_1=qk_0}^k(t_1-qk_0)\frac{6\eta L^2d\delta^2}{|\mathcal{S}_2|}  \nonumber
	\\&+(k-qk_0+1)\eta\pi(\mathcal{S}_1,\delta)-\Big(\frac{\eta}{2}-\frac{L\eta^2}{2}\Big)\sum_{t_1=qk_0}^k\mathbb{E}\|\mathbf{v}^k\|^2. \nonumber
	\\\overset{\text{(i)}}\leq&\mathbb{E}f(\mathbf{x}^{qk_0})+\frac{3L^2\eta^3}{|\mathcal{S}_2|}\sum_{t_1=qk_0}^k\sum_{t_2=qk_0}^{k}\mathbb{E}\|\mathbf{v}^{t_2}\|^2+\frac{(k-qk_0)(k-qk_0+1)}{|\mathcal{S}_2|} 3\eta L^2d\delta^2\nonumber
	\\&+(k-qk_0+1)\eta\pi(\mathcal{S}_1,\delta)-\Big(\frac{\eta}{2}-\frac{L\eta^2}{2}\Big)\sum_{t_1=qk_0}^k\mathbb{E}\|\mathbf{v}^k\|^2.\nonumber
	\\
	\leq&\mathbb{E}f(\mathbf{x}^{qk_0})+\frac{3L^2\eta^3(k-k_0+1)}{|\mathcal{S}_2|}\sum_{t=qk_0}^k\mathbb{E}\|\mathbf{v}^{t}\|^2+\frac{(k-qk_0)(k-qk_0+1)}{|\mathcal{S}_2|} 3\eta L^2d\delta^2 \nonumber
	\\&+(k-qk_0+1)\eta\pi(\mathcal{S}_1,\delta)-\Big(\frac{\eta}{2}-\frac{L\eta^2}{2}\Big)\sum_{t_1=qk_0}^k\mathbb{E}\|\mathbf{v}^k\|^2.\nonumber
	\\
	\leq&\mathbb{E}f(\mathbf{x}^{qk_0})-\left( \frac{\eta}{2}-\frac{\eta^2L}{2}- \frac{3L^2\eta^3(k-k_0+1)}{|\mathcal{S}_2|}\right)\sum_{t=qk_0}^k\mathbb{E}\|\mathbf{v}^{t}\|^2 \nonumber
	\\&+\frac{(k-qk_0)(k-qk_0+1)}{|\mathcal{S}_2|} 3\eta L^2d\delta^2 +(k-qk_0+1)\eta\pi(\mathcal{S}_1,\delta).
	\end{align}
\end{small}
\hspace{-0.12cm}where (i) follows from the fact that $\sum_{t_2=qk_0}^{t_1-1}\mathbb{E}\|\mathbf{v}^{t_2}\|^2\leq \sum_{t_2=qk_0}^{k}\mathbb{E}\|\mathbf{v}^{t_2}\|^2$ for $t_1-1\leq k-1<k$.  Without loss of generality we suppose $k^{\star} q<K\leq ( k^\star+1)q-1 $, where $k^\star=\lfloor K/q \rfloor$.  Then, 
based on~\eqref{kko}, we have, after $K$ iterations, 
\begin{align}\label{xK}
\mathbb{E}f(\mathbf{x}^K)\leq&\mathbb{E}f(\mathbf{x}^{0})-\left( \frac{\eta}{2}-\frac{\eta^2L}{2}- \frac{3L^2\eta^3(k-k^\star)}{|\mathcal{S}_2|}\right)\sum_{t=qk^\star}^{K-1}\mathbb{E}\|\mathbf{v}^{t}\|^2+\frac{(K-qk^\star)^2}{|\mathcal{S}_2|} 3\eta L^2d\delta^2 \nonumber
\\& +(K-qk^\star)\eta\pi(\mathcal{S}_1,\delta)+\sum_{t=1}^{k^\star}\left(\mathbb{E}f(\mathbf{x}^{tq})-\mathbb{E}f(\mathbf{x}^{{(t-1)}q})\right).
\end{align}
The term  $\mathbb{E}f(\mathbf{x}^{tq})-\mathbb{E}f(\mathbf{x}^{{(t-1)}q})$ in the above inequality can be upper-bounded by 
\begin{align}\label{tqx}
\mathbb{E}f(\mathbf{x}^{tq})-\mathbb{E}f(\mathbf{x}^{{(t-1)}q})\overset{\text{(i)}}\leq& \left( \frac{\eta}{2}-\frac{\eta^2L}{2}- \frac{3L^2\eta^3q}{|\mathcal{S}_2|}\right)\sum_{t_1=(t-1)q}^{tq-1}\mathbb{E}\|\mathbf{v}^{t_1}\|^2 +\frac{q^2}{|\mathcal{S}_2|} 3\eta L^2d\delta^2\nonumber
\\& +q\eta\pi(\mathcal{S}_1,\delta),
\end{align}
where (i) is obtained by letting $k_0=(t-1)$ and $k=tq-1$ in~\eqref{kko}. Combining~\eqref{xK} and~\eqref{tqx} yields 
\begin{align}\label{teleoog}
\mathbb{E}f(\mathbf{x}^K)\leq&\mathbb{E}f(\mathbf{x}^0)-\left( \frac{\eta}{2}-\frac{\eta^2L}{2}- \frac{3L^2\eta^3q}{|\mathcal{S}_2|}\right)\sum_{t=0}^{K-1}\mathbb{E}\|\mathbf{v}^{t}\|^2+\frac{(K-qk^\star)q}{|\mathcal{S}_2|} 3\eta L^2d\delta^2\nonumber
\\&+\frac{k^\star q^2}{|\mathcal{S}_2|} 3\eta L^2d\delta^2+K\eta\pi(\mathcal{S}_1,\delta) \nonumber
\\\leq&\mathbb{E}f(\mathbf{x}^0)-\left( \frac{\eta}{2}-\frac{\eta^2L}{2}- \frac{3L^2\eta^3q}{|\mathcal{S}_2|}\right)\sum_{t=0}^{K-1}\mathbb{E}\|\mathbf{v}^{t}\|^2+\frac{Kq}{|\mathcal{S}_2|}3\eta L^2d\delta^2\nonumber
\\&+K\eta \pi(\mathcal{S}_1,\delta),
\end{align}
which, in conjunction with~\eqref{pisw}, yields 
\begin{align}\label{niubibis}
\mathbb{E}f(\mathbf{x}^K)
\leq&\mathbb{E}f(\mathbf{x}^0)-\left( \frac{\eta}{2}-\frac{\eta^2L}{2}- \frac{3L^2\eta^3q}{|\mathcal{S}_2|}\right)\sum_{t=0}^{K-1}\mathbb{E}\|\mathbf{v}^{t}\|^2+\frac{3Kq}{|\mathcal{S}_2|} \eta L^2d\delta^2 \nonumber
\\&+K\eta\left( \frac{3I(|\mathcal{S}_1|<n)}{|\mathcal{S}_1|}\left(  2L^2d\delta^2+\sigma^2\right)+L^2d\delta^2\right).
\end{align}
Plugging the notations in Theorem~\ref{mainTT} into~\eqref{niubibis} yields 
\begin{align}\label{poisc}
\phi\sum_{t=0}^K\mathbb{E}\|\mathbf{v}^t\|^2\leq \Delta+\frac{3(K+1)q}{|\mathcal{S}_2|} \eta L^2d\delta^2 +(K+1)\eta\left( \frac{3I(|\mathcal{S}_1|<n)}{|\mathcal{S}_1|}\left(  2L^2d\delta^2+\sigma^2\right)+L^2d\delta^2\right),
\end{align}
where $\Delta:=f(\mathbf{x}^{0})-f(\mathbf{x}^*)$ with $\mathbf{x}^*:=\arg\min_{\mathbf{x}} f(\mathbf{x})$.

As $\zeta$ is generated from $\{0,...,K\}$ uniformly  at random, we have the output $\mathbf{x}^{\zeta}$ satisfies 
\begin{align}\label{nicep2}
\mathbb{E}\|f(\mathbf{x}^{\zeta})\|^2 \leq& 3\mathbb{E}\left( \|f(\mathbf{x}^{\zeta})-\hat \nabla_{\text{coord}}f(\mathbf{x}^{\zeta})\|^2+\|\mathbf{v}^\zeta-\hat \nabla_{\text{coord}}f(\mathbf{x}^{\zeta})\|^2
+\|\mathbf{v}^\zeta\|^2\right)\nonumber
\\\leq&  3L^2d\delta^2+\underbrace{\frac{3}{K+1}\sum_{k=0}^{K}\mathbb{E}\|\mathbf{v}^k-\hat \nabla_{\text{coord}}f(\mathbf{x}^{k})\|^2}_{\text{(A)}} +\frac{3}{K+1}\sum_{k=0}^{K}\mathbb{E}\|\mathbf{v}^k\|^2.
\end{align}
We next upper-bound the second term (A) in the above inequality.
First note that 
\begin{align}
\frac{1}{K+1} \sum_{t=0}^K\mathbb{E}\| \hat \nabla_{\text{coord}}f(\mathbf{x}^t)-\mathbf{v}^t\| =\frac{1}{K+1}\left( \sum_{p=0}^{k^\star-1}\sum_{t=pq}^{(p+1)q-1} \mathbb{E}\|\hat \nabla_{\text{coord}}f(\mathbf{x}^t)-\mathbf{v}^t\|^2+\sum_{t=qk^\star}^K\mathbb{E}\|\hat \nabla_{\text{coord}}f(\mathbf{x}^t)-\mathbf{v}^t\|^2 \right).\nonumber
\end{align}
Applying Lemma~\ref{le:coord} to the above equation yields 
\begin{align*}
\sum_{t=0}^K&\mathbb{E}\| \hat \nabla_{\text{coord}}f(\mathbf{x}^t)-\mathbf{v}^t\|-\sum_{t=qk^\star}^K\mathbb{E}\|\hat \nabla_{\text{coord}}f(\mathbf{x}^t)-\mathbf{v}^t\|^2
\\\leq&   \sum_{p=0}^{k^\star-1}\sum_{t=pq}^{(p+1)q-1} \left(\frac{3L^2\eta^2}{|\mathcal{S}_2|}\sum_{t_1=pq}^{t-1}\mathbb{E}\|\mathbf{v}^{t_1}\|^2+(t-pq)\frac{6L^2d\delta^2}{|\mathcal{S}_2|} +\pi(\mathcal{S}_1,\delta)\right) 
\\\leq&\sum_{p=0}^{k^\star-1}\left(   \frac{3qL^2\eta^2}{|\mathcal{S}_2|}\sum_{t_1=pq}^{(p+1)q-1}\mathbb{E}\|\mathbf{v}^{t_1}\|^2+q(q-1)\frac{3L^2d\delta^2}{|\mathcal{S}_2|}+ q\pi(\mathcal{S}_1,\delta) \right)
\\ \leq &\frac{3qL^2\eta^2}{|\mathcal{S}_2|}\sum_{p=0}^{k^\star-1}\sum_{t_1=pq}^{(p+1)q-1}\mathbb{E}\|\mathbf{v}^{t_1}\|^2+qk^\star(q-1)\frac{3L^2d\delta^2}{|\mathcal{S}_2|}+qk^\star\pi(\mathcal{S}_1,\delta),
\end{align*}
which, by applying Lemma~\ref{le:coord} to $\sum_{t=qk^\star}^K\mathbb{E}\|\hat \nabla_{\text{coord}}f(\mathbf{x}^t)-\mathbf{v}^t\|^2$, yields 
\begin{align}\label{repuse}
\sum_{t=0}^K\mathbb{E}\| \hat \nabla_{\text{coord}}f(\mathbf{x}^t)-\mathbf{v}^t\|\leq& \frac{3qL^2\eta^2}{|\mathcal{S}_2|}\sum_{t=0}^{K}\mathbb{E}\|\mathbf{v}^t\|^2+\frac{(K-qk^\star)q+q^2k^\star}{|\mathcal{S}_2|}3L^2d\delta^2+(K+1)\pi(\mathcal{S}_1,\delta).
\end{align}
The above inequality can be further simplified to
\begin{align}\label{ggsw}
\frac{1}{K+1} \sum_{t=0}^K\mathbb{E}\| \hat \nabla_{\text{coord}}f(\mathbf{x}^t)-\mathbf{v}^t\|\leq& \frac{3qL^2\eta^2}{|\mathcal{S}_2|}\frac{1}{K+1}\sum_{t=0}^{K}\mathbb{E}\|\mathbf{v}^t\|^2+\frac{3q}{|\mathcal{S}_2|}L^2d\delta^2
+ \pi(\mathcal{S}_1,\delta).
\end{align}
Combining~\eqref{poisc}, \eqref{nicep2} and~\eqref{ggsw} yields 
\begin{align*}
\mathbb{E}\|f(\mathbf{x}^{\zeta})\|^2 \leq&  3L^2d\delta^2+\frac{1}{\phi}\left( \frac{9q\eta^2L^2}{|\mathcal{S}_2|} +3  \right)\left(   \frac{\Delta}{K+1}+\eta \big(\theta+L^2d\delta^2\big)\right)+3\theta,
\end{align*}
which finishes the proof.
\subsection{Proof of Corollary~\ref{comain}}
We prove two cases when $n\leq K$ and $n> K$, separately. 

First we suppose $n\leq K$. Under the selection of parameters in~\eqref{pabu1}, we have $|\mathcal{S}_1|=n,q=|\mathcal{S}_1|=\lceil n^{1/2}\rceil$, and thus obtain 
\begin{align*}
\phi= \eta\left(\frac{1}{2}-\frac{1}{8}-\frac{3}{16}\right)=\frac{3\eta}{16},\;\theta=\frac{3}{K},
\end{align*} 
which, in conjunction with~\eqref{uioj}, yields
\begin{align*}
\mathbb{E}\|f(\mathbf{x}^\zeta)\|^2\leq \frac{3}{K}+\left(\frac{16}{3\eta}\frac{9}{16}+\frac{16}{\eta}\right)\left(\frac{\Delta}{K}+\eta\left( \frac{4}{K}\right)\right)+\frac{9}{K}=\frac{76\Delta L+88}{K}\leq \mathcal{O}\left(\frac{1}{K}\right).
\end{align*} 
We choose $K=C\epsilon^{-1}$, where $C>0$ is a constant. Then, based on the above inequality, we have, for $C$ large enough, our Algorithm~\ref{ours:3}
achieves $\mathbb{E}\|f(\mathbf{x}^\zeta)\|^2\leq \epsilon$,  and  the total number of function queries  can be bounded as 
\begin{align}\label{faac1}
\left\lceil \frac{K}{q}\right\rceil nd + K |\mathcal{S}_2|d\leq Kn^{1/2}d +nd+ K n^{1/2}d+Kd\leq \mathcal{O}(nd+\epsilon^{-1}n^{1/2}d)\leq   \mathcal{O}(\epsilon^{-1}n^{1/2}d)\leq \mathcal{O}(d\epsilon^{-3/2}),
\end{align}
where the last two inequalities follow from the assumption that $n\leq K=C\epsilon^{-1}$.

Next, we suppose $n>K$. In this case, we have $|\mathcal{S}_1|=K, q=|\mathcal{S}_1|=\lceil K^{1/2}\rceil$, and 
\begin{align*}
\phi= \eta\left(\frac{1}{2}-\frac{1}{8}-\frac{3}{16}\right)=\frac{3\eta}{16},\;\theta=\frac{3}{K}+\frac{3}{K}\left( \frac{2}{K}+\sigma^2\right),
\end{align*} 
which, in conjunction with~\eqref{uioj}, yields
\begin{align*}
\mathbb{E}\|f(\mathbf{x}^\zeta)\|^2\leq \frac{3}{K} +\frac{19}{\eta}\left(\frac{\Delta}{K}+\eta\left(\frac{4+3\sigma^2}{K}+\frac{6}{K^2}\right)\right)+\frac{9}{K}=\frac{76\Delta L+88+57\sigma^2}{K}+\frac{114}{K^2}\leq \mathcal{O}\left( \frac{1}{K}\right).
\end{align*}
We choose $K=C\epsilon^{-1}$, where $C>0$ is a constant. Then, for $C$ large enough, our Algorithm~\ref{ours:3}  achieves  $\mathbb{E}\|f(\mathbf{x}^\zeta)\|^2\leq \epsilon$, and  the total number of function queries  can be bounded as 
\begin{align}\label{facc2}
\left\lceil \frac{K}{q}\right\rceil Kd + K |\mathcal{S}_2|d&\leq  Kd+K^{3/2}d+K^{3/2}d+Kd=2K^{3/2}d+2Kd\leq \mathcal{O}(K^{3/2}d)\nonumber
\\&\leq\mathcal{O}(d\epsilon^{-3/2})\leq \mathcal{O}(\epsilon^{-1}n^{1/2}d),
\end{align}
where the last inequality follows from the assumption that $n>K\geq C\epsilon^{-1}$. 

Combining~\eqref{faac1} and~\eqref{facc2}  implies that the number of function queries required by Algorithm~\ref{ours:3} is at most $\mathcal{O}\left(\min\{\epsilon^{-1}n^{1/2}d, d\epsilon^{-3/2} \}\right)$.

\subsection{Proof of Corollary~\ref{nonbatch} }\label{co3}
We prove two cases when $n\leq \lceil K^{2/3}\rceil  $ and $n> \lceil K^{2/3}\rceil  $, separately. 

First we suppose $n\leq \lceil K^{2/3}\rceil $, and thus we have $q=|\mathcal{S}_1|=n, \eta=1/(4L\sqrt{n}), \delta=1/(L\sqrt{nKd})$. Then,  we obtain 
\begin{align*}
\phi=\eta\left( \frac{1}{2}-\frac{1}{8\sqrt{n}}-\frac{3}{16} \right)\geq \frac{3\eta}{16},\; \theta=\frac{3}{K},
\end{align*}
which, in conjunction with~\eqref{uioj}, yields 
\begin{align*}
\mathbb{E}\|f(\mathbf{x}^\zeta)\|^2\leq \frac{76\sqrt{n}L\delta}{K}+\frac{66}{K}+\frac{22}{nK}= \frac{76\sqrt{|\mathcal{S}_1|}L\delta}{K}+\frac{66}{K}+\frac{22}{|\mathcal{S}_1|K} \leq \mathcal{O}\left(\frac{\sqrt{|\mathcal{S}_1|}}{K}\right).
\end{align*}
Let $K=C\sqrt{n}\epsilon^{-1}$ for a positive constant $C$, which, combined with $n \leq \lceil K^{2/3}\rceil$, implies that $n\leq \mathcal{O}(\epsilon^{-1})$.
Then, our Algorithm~\ref{ours:3}  achieves  $\mathbb{E}\|f(\mathbf{x}^\zeta)\|^2\leq \epsilon$, and  the total number of function queries  can be bounded as 
\begin{align}
\left\lceil \frac{K}{q}\right\rceil nd + K d \leq nd+2Kd\leq \mathcal{O}\left(nd + \frac{n^{1/2}d}{\epsilon}\right)\leq  \mathcal{O}(\epsilon^{-1}n^{1/2}d) \leq \mathcal{O}\left( \frac{d}{\epsilon^{3/2}} \right),
\end{align}
where the last two inequalities follow from the assumption that $n\leq \mathcal{O}( \epsilon^{-1})$. 

Next, we suppose $n>\lceil K^{2/3}\rceil$. In this case, we have $|\mathcal{S}_1|=\lceil K^{2/3} \rceil$, and thus
\begin{align}
\phi \geq \eta\left( \frac{1}{2}-\frac{1}{8}-\frac{3}{16} \right)\geq \frac{3\eta}{16},\; \theta=\frac{3}{K}+\frac{6}{K^{7/3}}+\frac{3\sigma^2}{K^{2/3}},
\end{align}
which, in conjunction with~\eqref{uioj}, yields 
\begin{align*}
\mathbb{E}\|f(\mathbf{x}^\zeta)\|^2\leq \mathcal{O}\left( \frac{K^{1/3}}{K}\right)\leq \mathcal{O}\left( \frac{\sqrt{|\mathcal{S}_1|}}{K}  \right),
\end{align*}
where the last inequality follows from the assumption that $\lceil K ^{2/3}\rceil< n$.  Let $K=C\epsilon^{-3/2}$, where $C>0$ is a constant. 
Then, for $C$ large enough, our Algorithm~\ref{ours:3}  achieves  $\mathbb{E}\|f(\mathbf{x}^\zeta)\|^2\leq \epsilon$, and  the total number of function queries can be bounded as 
\begin{align}
\left\lceil \frac{K}{q}\right\rceil |\mathcal{S}_1|d + K d \leq 2Kd+|\mathcal{S}_1|d\leq 3Kd \leq \mathcal{O}\left( \frac{d}{\epsilon^{3/2}}\right) \leq \mathcal{O}\left( \frac{n^{1/2}d}{\epsilon}\right)
\end{align}
where the last inequality follows from the assumption that $n>\lceil K ^{2/3}\rceil=C^{2/3}\epsilon^{-1}$. 

Combining the above two cases finishes the proof. 

\section{Proof for ZO-SPIDER-Coord  under PL Condition}\label{plpp}
\subsection{Proof of Theorem~\ref{th:gd}}
Let $\mathbf{x}^*=\arg\min_{\mathbf{x}} f(\mathbf{x})$. Then, for  any $qk_0\leq m\leq q(k_0+1)-1, k_0=0,....,h-1$ $(h=K/q)$, we have
\begin{align}
&f(\mathbf{x}^{m+1})\nonumber
\\&\leq f(\mathbf{x}^{m}) +\langle  \nabla f(\mathbf{x}^m),\mathbf{x}^{m+1}-\mathbf{x}^{m} \rangle + \frac{L}{2}\|\mathbf{x}^{m+1}-\mathbf{x}^{m} \|^2 \nonumber
\\&=f(\mathbf{x}^m)-\frac{\eta}{2}\langle \nabla f(\mathbf{x}^m) -\mathbf{v}^m,\mathbf{v}^m  \rangle -\frac{\eta}{2}\|\mathbf{v}^m\|^2-\frac{\eta}{2}\langle  \nabla f(\mathbf{x}^m),\mathbf{v}^m-\nabla f(\mathbf{x}^m)  \rangle -\frac{\eta}{2}\|\nabla f(\mathbf{x}^m)\|^2+\frac{L\eta^2}{2}\|\mathbf{v}^m\|^2  \nonumber
\\&\leq f(\mathbf{x}^m) +\frac{\eta}{2}\|\nabla f(\mathbf{x}^m) -\mathbf{v}^m\|^2-\frac{\eta}{4}\|\mathbf{v}^m\|^2-\frac{\eta}{4}\|\nabla f(\mathbf{x}^m)\|^2+\frac{L\eta^2}{2}\|\mathbf{v}^m\|^2  \nonumber
\\&= f(\mathbf{x}^m) +\frac{\eta}{2}\|\nabla f(\mathbf{x}^m) -\mathbf{v}^m\|^2-\Big(\frac{\eta}{4}-\frac{L\eta^2}{2}\Big)\|\mathbf{v}^m\|^2-\frac{\eta}{4}\|\nabla f(\mathbf{x}^m)\|^2, \nonumber
\\&\overset{\text{(i)}}\leq   f(\mathbf{x}^m) +\frac{\eta}{2}\|\nabla f(\mathbf{x}^m) -\mathbf{v}^m\|^2-\Big(\frac{\eta}{4}-\frac{L\eta^2}{2}\Big)\|\mathbf{v}^m\|^2-\frac{\eta}{4\gamma} (f(\mathbf{x}^m)-f(\mathbf{x}^*))
\end{align}
where (i) follows from Definition~\ref{d1}. Taking expectation over the above inequality and using Lemma~\ref{le:coord}, we have
\begin{align}\label{sioop}
\mathbb{E}(f(\mathbf{x}^{m+1})-f(\mathbf{x}^*))\leq& \left(1-\frac{\eta}{4\gamma} \right)\mathbb{E}(f(\mathbf{x}^m)-f(\mathbf{x}^*))+ \frac{3\eta^3L^2}{|\mathcal{S}_2|}\sum_{t=qk_0}^{m}\mathbb{E} \|  \mathbf{v}^{t}\|^2+(m-qk_0)\frac{6\eta L^2d\delta^2}{|\mathcal{S}_2|}+\eta L^2d\delta^2\nonumber
\\&-\Big(\frac{\eta}{4}-\frac{L\eta^2}{2}\Big)\mathbb{E}\|\mathbf{v}^m\|^2. 
\end{align}
To simplify notation, we let $\alpha:=1-\eta/(4\gamma)$. Then, telescoping~\eqref{sioop} over $m$ from $qk_0$ to $q(k_0+1)-1$ yields
\begin{align}
\mathbb{E}(f(\mathbf{x}^{q(k_0+1)})&-f(\mathbf{x}^*))\leq \alpha^q \mathbb{E}(f(\mathbf{x}^{qk_0})-f(\mathbf{x}^*)) + \sum_{m=qk_0}^{q(k_0+1)-1}\frac{3\eta^3L^2}{|\mathcal{S}_2|}\alpha^{q(k_0+1)-m-1}\sum_{t=qk_0}^m\mathbb{E}\|\mathbf{v}^t\|^2 \nonumber
\\&+\sum_{m=qk_0}^{q(k_0+1)-1}\left(\frac{b_\gamma}{B_\gamma}+1\right)\alpha^{q(k_0+1)-m-1}\eta L^2d\delta^2 -\Big(\frac{\eta}{4}-\frac{L\eta^2}{2}\Big)\sum_{m=qk_0}^{q(k_0+1)-1}\alpha^{q(k_0+1)-m-1} \mathbb{E}\|\mathbf{v}^m\|^2 \nonumber
\\=& \alpha^q \mathbb{E}(f(\mathbf{x}^{qk_0})-f(\mathbf{x}^*)) +\frac{3\eta^3L^2}{|\mathcal{S}_2|} \sum_{m=qk_0}^{q(k_0+1)-1}\sum_{t=m}^{q(k_0+1)-1}\alpha^{q(k_0+1)-t-1}\mathbb{E}\|\mathbf{v}^m\|^2 \nonumber
\\&+\sum_{m=qk_0}^{q(k_0+1)-1}\left(\frac{b_\gamma}{B_\gamma}+1\right)\alpha^{q(k_0+1)-m-1}\eta L^2d\delta^2 -\Big(\frac{\eta}{4}-\frac{L\eta^2}{2}\Big)\sum_{m=qk_0}^{q(k_0+1)-1}\alpha^{q(k_0+1)-m-1} \mathbb{E}\|\mathbf{v}^m\|^2 \nonumber
\end{align}
where the first inequality follows from the fact that $(m-qk_0)/|\mathcal{S}_2|\leq q/(B_\gamma \gamma L)\leq b_\gamma/B_\gamma$.
Noting that $\alpha^{q(k_0+1)-m-1}\geq \alpha^q$ and $\sum_{t=m}^{q(k_0+1)-1}\alpha^{q(k_0+1)-t-1}=(1-\alpha^{q(k_0+1)-m})/(1-\alpha)<1/(1-\alpha)=4\gamma/\eta$, we obtain from the above inequality that 
\begin{align}\label{telesss}
\mathbb{E}(f&(\mathbf{x}^{q(k_0+1)})-f(\mathbf{x}^*)) \nonumber
\\\leq& \alpha^q \mathbb{E}(f(\mathbf{x}^{qk_0})-f(\mathbf{x}^*)) +\frac{12\gamma\eta^2L^2}{|\mathcal{S}_2|} \sum_{m=qk_0}^{q(k_0+1)-1}\mathbb{E}\|\mathbf{v}^m\|^2 \nonumber
\\&+\frac{1-\alpha^{q}}{1-\alpha}\left(\frac{b_\gamma}{B_\gamma}+1\right)\eta L^2d\delta^2 -\Big(\frac{\eta}{4}-\frac{L\eta^2}{2}\Big)\sum_{m=qk_0}^{q(k_0+1)-1}\alpha^{q} \mathbb{E}\|\mathbf{v}^m\|^2 \nonumber
\\ \overset{\text{(i)}}\leq& \alpha^q \mathbb{E}(f(\mathbf{x}^{qk_0})-f(\mathbf{x}^*)) +\frac{12\eta^2L}{B_\gamma}\sum_{m=qk_0}^{q(k_0+1)-1}\mathbb{E}\|\mathbf{v}^m\|^2 \nonumber
\\&+\frac{1-\alpha^{q}}{1-\alpha}\left(\frac{b_\gamma}{B_\gamma}+1\right)\eta L^2d\delta^2 -\Big(\frac{\eta}{4}-\frac{L\eta^2}{2}\Big)\left(1-\frac{b_\gamma}{16q}\right)^q\sum_{m=qk_0}^{q(k_0+1)-1} \mathbb{E}\|\mathbf{v}^m\|^2 \nonumber
\\=&  \alpha^q \mathbb{E}(f(\mathbf{x}^{qk_0})-f(\mathbf{x}^*))+\frac{1-\alpha^{q}}{1-\alpha}\left(\frac{b_\gamma}{B_\gamma}+1\right)\eta L^2d\delta^2-\eta\left(\frac{1}{8}\left(1-\frac{b_\gamma}{16q}\right)^q -\frac{3 }{B_\gamma}  \right)\sum_{m=qk_0}^{q(k_0+1)-1} \mathbb{E}\|\mathbf{v}^m\|^2 \nonumber
\\\leq &  \alpha^q \mathbb{E}(f(\mathbf{x}^{qk_0})-f(\mathbf{x}^*))+\frac{1-\alpha^{q}}{1-\alpha}\left(\frac{b_\gamma}{B_\gamma}+1\right)\eta L^2d\delta^2
\end{align}
where (i) follows from the facts that $\frac{q}{b_\gamma L} <\gamma$ and $|\mathcal{S}_2|=\lceil \gamma L B_\gamma\rceil$ and the last inequality  follows from the condition that $\frac{1}{8}\left(1-\frac{b_\gamma}{16q}\right)^q -\frac{3 }{B_\gamma}>0$. 

Telescoping~\eqref{telesss} over $k_0$ from $0$ to $h-1$ yields 
\begin{align}\label{sopop}
\mathbb{E}(f(\mathbf{x}^{K})-f(\mathbf{x}^*))  \leq& \left(1-\frac{\eta}{4\gamma}\right)^K(f(\mathbf{x}^{0})-f(\mathbf{x}^*)) +\sum_{k_0=0}^{h-1}\alpha^{qk_0}\frac{1-\alpha^q}{1-\alpha}\left(\frac{b_\gamma}{B_\gamma}+1\right)\eta L^2d\delta^2  \nonumber
\\=&\left(1-\frac{\eta}{4\gamma}\right)^K(f(\mathbf{x}^{0})-f(\mathbf{x}^*))+\frac{1-\alpha^k}{1-\alpha^q}\frac{1-\alpha^q}{1-\alpha}\left(\frac{b_\gamma}{B_\gamma}+1\right)\eta L^2d\delta^2 \nonumber
\\\leq& \left(1-\frac{1}{16L\gamma}\right)^K(f(\mathbf{x}^{0})-f(\mathbf{x}^*))+4\gamma \left(\frac{b_\gamma}{B_\gamma}+1\right)L^2d\delta^2.
\end{align}
From~\eqref{sopop}, we require the total number $K=\mathcal{O}(\gamma \log \left(1/\epsilon\right))$ and $\delta=\mathcal{O}(\sqrt{\epsilon}/(L\sqrt{\gamma d}))$ to achieve $ \mathbb{E}(f(\mathbf{x}^{K})-f(\mathbf{x}^*)) <\epsilon$. Thus, the total number of function queries is 
$\left\lceil  \frac{K}{q}\right\rceil nd +K|\mathcal{S}_2|d = \mathcal{O}\left( d(\gamma n^{1/2}+\gamma^2)\log \left( \frac{1}{\epsilon}\right)\right)$. 
Then, the proof is complete.

\section{Proofs for PROX-ZO-SPIDER-Coord }
\subsection{Auxiliary Lemma}
We first prove the following useful lemma. 
\begin{lemma}\label{proxT}
	Let Assumption~\ref{assumption} hold, and define 
	\begin{align}\label{plllss}
	&\tau=\frac{\eta}{2}-\frac{L\eta^2}{2}-\frac{3q}{|\mathcal{S}_2|}\eta^3L^2 \nonumber \\&C=\frac{6I(|\mathcal{S}_1|<n)}{|\mathcal{S}_1|}\left(  2L^2d\delta^2+\sigma^2\right)+2L^2d\delta^2.
	\end{align}
	Then, we have,  
	\begin{align}\label{egah}
	\mathbb{E}\|G(\mathbf{x}^\zeta,\nabla f(\mathbf{x}^\zeta),\eta) \|^2\leq& \frac{2}{\tau}\left( \frac{\Delta_\psi}{K} +\frac{3q}{|\mathcal{S}_2|}\eta L^2d\delta^2+ \frac{\eta}{2}C\right) +\frac{12qL^2\eta^2}{|\mathcal{S}_2|\tau}\left(\frac{\Delta_\psi}{K} +\frac{3q}{|\mathcal{S}_2|}\eta L^2d\delta^2+ \frac{\eta}{2}C\right)\nonumber
	\\&+\frac{12q}{|\mathcal{S}_2|}L^2d\delta^2+4C, 
	\end{align}
	where $\Delta_\psi=\psi(\mathbf{x}^{0})-\psi(\mathbf{x}^{*})$ with $\mathbf{x}^*=\arg\min_{\mathbf{x}\in\mathbb{R}^d}\psi(\mathbf{x})$.
\end{lemma}

\begin{proof}
	We first introduce the following notation for our proof
	\begin{align}\label{xgg}
	G(\mathbf{x},\mathbf{g},\eta)=\frac{1}{\eta}(\mathbf{x}-\mathbf{x}_g), \text{ where } \mathbf{x}_g=\arg\min_{\mathbf{z}\in\mathbb{R}^d}\left\{ \langle \mathbf{g},\mathbf{z}\rangle  +\frac{1}{2\eta} \|\mathbf{z}-\mathbf{x}\|^2+h(\mathbf{z}) \right\}.
	\end{align}
	Note that when $\mathbf{g}=\nabla f(\mathbf{x})$, $G(\mathbf{x},\mathbf{g},\eta)$ becomes the generalized projected gradient of  the objective $\Psi(\cdot)$ at  $\mathbf{x}$. The following lemma provides important properties of $G(\mathbf{x},\mathbf{g},\eta)$ by Lemma 1 and Proposition 1 in~\citealt{ghadimi2016mini}.
	\begin{lemma} \label{g2g}
		For any $\mathbf{g, g}_1$ and $\mathbf{g}_2$ in $\mathbb{R}^d$, we have 
		\begin{itemize}
			\item[(i)] $\langle  \mathbf{g}, G(\mathbf{x},\mathbf{g},\eta) \rangle\geq \|G(\mathbf{x},\mathbf{g},\eta)\|^2+(h(\mathbf{x}_g)-h(\mathbf{x}))/\eta$, where $\mathbf{x}_g$ is defined by~\eqref{xgg}.
			\item[(ii)]  $\|G(\mathbf{x},\mathbf{g}_1,\eta)-G(\mathbf{x},\mathbf{g}_2,\eta)\|\leq \|\mathbf{g}_1-\mathbf{g}_2\|$.
		\end{itemize} 
	\end{lemma}
	Based on the above results, we now prove Lemma~\ref{proxT}.
	Using an approach similar to Lemma~\ref{le:coord}, we obtain, for any given $k_0\leq \lfloor K /q\rfloor$ and $qk_0\leq k\leq \min\{q(k_0+1)-1,K\}$,
	\begin{align*}
	\mathbb{E}\|	\mathbf{v}^k-\hat \nabla_{\text{\normalfont coord}}f(\mathbf{x}^k)\|^2&\leq\frac{3L^2}{|\mathcal{S}_2|}\sum_{t=qk_0}^{k-1}\mathbb{E} \|  \mathbf{x}^{t+1}-\mathbf{x}^t\|^2+(k-qk_0)\frac{6L^2d\delta^2}{|\mathcal{S}_2|}+\frac{3I(|\mathcal{S}_1|<n)}{|\mathcal{S}_1|}\left(  2L^2d\delta^2+\sigma^2\right), 
	\end{align*}
	which, based on the proximal gradient step  and~\eqref{xgg}, implies that  $\mathbf{x}^k-\mathbf{x}^{k+1}=\eta G(\mathbf{x}^k,\mathbf{v}^k,\eta)$. Thus 
	\begin{align*}
	\mathbb{E}\|	\mathbf{v}^k-\hat \nabla_{\text{\normalfont coord}}f(\mathbf{x}^k)\|^2\leq& \frac{3L^2\eta^2}{|\mathcal{S}_2|}\sum_{t=qk_0}^{k-1}\mathbb{E}\|G(\mathbf{x}^t,\mathbf{v}^t,\eta)\|^2+(k-qk_0)\frac{6L^2d\delta^2}{|\mathcal{S}_2|} \nonumber
	\\&+\frac{3I(|\mathcal{S}_1|<n)}{|\mathcal{S}_1|}\left(  2L^2d\delta^2+\sigma^2\right),
	\end{align*}
	which, in conjunction with~\eqref{coordinate},  implies that 
	\begin{align}\label{vk2}
	\mathbb{E}\|	\mathbf{v}^k-\nabla f(\mathbf{x}^k)\|^2\leq& 2\mathbb{E}\|	\mathbf{v}^k-\hat \nabla_{\text{\normalfont coord}}f(\mathbf{x}^k)\|^2+2\mathbb{E}\|	\nabla f(\mathbf{x}^k)-\hat \nabla_{\text{\normalfont coord}}f(\mathbf{x}^k)\|^2    \nonumber 
	\\\leq& \frac{6L^2\eta^2}{|\mathcal{S}_2|}\sum_{t=qk_0}^{k-1}\mathbb{E}\|G(\mathbf{x}^t,\mathbf{v}^t,\eta)\|^2+(k-qk_0)\frac{12L^2d\delta^2}{|\mathcal{S}_2|} \nonumber
	\\&+\frac{6I(|\mathcal{S}_1|<n)}{|\mathcal{S}_1|}\left(  2L^2d\delta^2+\sigma^2\right)+2L^2d\delta^2.
	\end{align}
	Recalling that  the gradient $\nabla f(\mathbf{x})$ is $L$-Lipschitz, we have, for any $k_0\leq \lfloor K /q\rfloor$ and $qk_0\leq k\leq \min\{q(k_0+1)-1,K\}$,   
	\begin{align}
	f(\mathbf{x}^{k+1})\leq & f(\mathbf{x}^{k}) +  \langle \nabla f(\mathbf{x}^k),\mathbf{x}^{k+1}-\mathbf{x}^{k} \rangle + \frac{L}{2}\|\mathbf{x}^{k+1}-\mathbf{x}^k\|^2 \nonumber
	\\=&  f(\mathbf{x}^{k}) -\eta  \langle \nabla f(\mathbf{x}^k),G(\mathbf{x}^k,\mathbf{v}^k,\eta) \rangle + \frac{L\eta^2}{2}\|G(\mathbf{x}^k,\mathbf{v}^k,\eta)\|^2   \nonumber
	\\=& f(\mathbf{x}^{k}) -\eta  \langle \nabla f(\mathbf{x}^k)-\mathbf{v}^k,G(\mathbf{x}^k,\mathbf{v}^k,\eta) \rangle - \eta\langle \mathbf{v}^k,G(\mathbf{x}^k,\mathbf{v}^k,\eta)\rangle+ \frac{L\eta^2}{2}\|G(\mathbf{x}^k,\mathbf{v}^k,\eta)\|^2 \nonumber
	\\\leq &  f(\mathbf{x}^{k}) +\frac{\eta}{2}  \| \nabla f(\mathbf{x}^k)-\mathbf{v}^k\|^2+\frac{\eta}{2}\|G(\mathbf{x}^k,\mathbf{v}^k,\eta) \|^2 - \eta\langle \mathbf{v}^k,G(\mathbf{x}^k,\mathbf{v}^k,\eta)\rangle+ \frac{L\eta^2}{2}\|G(\mathbf{x}^k,\mathbf{v}^k,\eta)\|^2 \nonumber,
	\end{align}
	which, in conjunction with Lemma~\ref{g2g}, implies that 
	\begin{align}\label{betak}
	f(\mathbf{x}^{k+1})\leq f(\mathbf{x}^{k}) +\frac{\eta}{2}  \| \nabla f(\mathbf{x}^k)-\mathbf{v}^k\|^2-\left(\frac{\eta}{2}-\frac{L\eta^2}{2}\right)\|G(\mathbf{x}^k,\mathbf{v}^k,\eta) \|^2 -(h(\mathbf{x}^{k+1})-h(\mathbf{x}^{k})).
	\end{align}
	Let $\psi(\mathbf{x})=f(\mathbf{x})+h(\mathbf{x})$. Then, taking expectation over~\eqref{betak} yields 
	\begin{align}
	\mathbb{E}\psi(\mathbf{x}^{k+1}) \leq \mathbb{E}\psi(\mathbf{x}^k) +\frac{\eta}{2}  \mathbb{E}\| \nabla f(\mathbf{x}^k)-\mathbf{v}^k\|^2-\left(\frac{\eta}{2}-\frac{L\eta^2}{2}\right)\mathbb{E}\|G(\mathbf{x}^k,\mathbf{v}^k,\eta) \|^2. \nonumber
	\end{align}
	Telescoping  the above inequality yields, for $qk_0\leq k\leq \min\{q(k_0+1)-1,K\}$,  
	\begin{align*}
	\mathbb{E}\psi(\mathbf{x}^{k+1}) \leq \mathbb{E}\psi(\mathbf{x}^{k_0q}) +\frac{\eta}{2}\sum_{t=k_0q}^k \mathbb{E} \| \nabla f(\mathbf{x}^t)-\mathbf{v}^t\|^2-\left(\frac{\eta}{2}-\frac{L\eta^2}{2}\right)\sum_{t=k_0q}^k\mathbb{E}\|G(\mathbf{x}^t,\mathbf{v}^t,\eta) \|^2,
	\end{align*}
	which,  recalling the definition of $C$ in~\eqref{plllss} 
	and using~\eqref{vk2},  implies that 
	\begin{align}
	\mathbb{E}  \psi(\mathbf{x}^{k+1}) \leq& \mathbb{E}\psi(\mathbf{x}^{k_0q}) +\frac{\eta}{2}\sum_{t=k_0q}^k \bigg(\frac{6L^2\eta^2}{|\mathcal{S}_2|}\sum_{p=qk_0}^{t-1}\mathbb{E}\|G(\mathbf{x}^p,\mathbf{v}^p,\eta)\|^2+(t-qk_0)\frac{12L^2d\delta^2}{|\mathcal{S}_2|}+C\bigg)\nonumber
	\\& -\left(\frac{\eta}{2}-\frac{L\eta^2}{2}\right)\sum_{t=k_0q}^k\|G(\mathbf{x}^t,\mathbf{v}^t,\eta) \|^2 \nonumber
	\\\leq& \mathbb{E}\psi(\mathbf{x}^{k_0q}) + \frac{3L^2\eta^3}{|\mathcal{S}_2|}\sum_{t=k_0q}^k\sum_{p=qk_0}^{t-1}\mathbb{E}\|G(\mathbf{x}^p,\mathbf{v}^p,\eta)\|^2+\frac{\eta}{2}\sum_{t=k_0q}^k(t-qk_0)\frac{12L^2d\delta^2}{|\mathcal{S}_2|} \nonumber
	\\& +\frac{\eta}{2}\sum_{t=k_0q}^kC-\left(\frac{\eta}{2}-\frac{L\eta^2}{2}\right)\sum_{t=k_0q}^k\|G(\mathbf{x}^t,\mathbf{v}^t,\eta) \|^2 \nonumber
	\\\leq& \mathbb{E}\psi(\mathbf{x}^{k_0q}) + \frac{3L^2\eta^3}{|\mathcal{S}_2|}(k-k_0q+1)\sum_{p=qk_0}^{k-1}\mathbb{E}\|G(\mathbf{x}^p,\mathbf{v}^p,\eta)\|^2+\frac{3(k-qk_0+1)(k-qk_0)}{|\mathcal{S}_2|}\eta L^2 d\sigma^2 \nonumber
	\\& +\frac{\eta}{2}(k-qk_0+1)C-\left(\frac{\eta}{2}-\frac{L\eta^2}{2}\right)\sum_{t=k_0q}^k\|G(\mathbf{x}^t,\mathbf{v}^t,\eta) \|^2.
	\end{align}
	Using the above inequality and letting $K^*=\lfloor K/q\rfloor$, we obtain  
	\begin{align}
	\mathbb{E} \psi(\mathbf{x}^{K})-\mathbb{E} \psi(\mathbf{x}^{0}) \leq&  \mathbb{E} \psi(\mathbf{x}^{K})-\mathbb{E} \psi(\mathbf{x}^{qK^*}) + \sum_{i=1}^{K^*} \left( \mathbb{E} \psi(\mathbf{x}^{qi})-\mathbb{E} \psi(\mathbf{x}^{q(i-1)}) \right) \nonumber
	\\ \leq& \frac{3L^2\eta^3}{|\mathcal{S}_2|}(K-qK^*)\sum_{p=qK^*}^{K-2}\mathbb{E}\|G(\mathbf{x}^p,\mathbf{v}^p,\eta)\|^2+\frac{3(K-qK^*)(K-qK^*-1)}{|\mathcal{S}_2|}\eta L^2d\delta^2\nonumber
	\\& +\frac{\eta}{2}(K-qK^*)C-\left(\frac{\eta}{2}-\frac{L\eta^2}{2}\right)\sum_{t=qK^*}^{K-1}\|G(\mathbf{x}^t,\mathbf{v}^t,\eta) \|^2 \nonumber
	\\ &+ \sum_{i=1}^{K^*}\bigg( \frac{3\eta^3L^2q}{|\mathcal{S}_2|}  \sum_{p=q(i-1)}^{qi-1}\mathbb{E}\|G(\mathbf{x}^p,\mathbf{v}^p,\eta)\|^2+\frac{3q^2\eta L^2d\delta^2}{|\mathcal{S}_2|}+\frac{q\eta}{2} C \nonumber
	\\&-\Big(\frac{\eta}{2}-\frac{L\eta^2}{2}\Big)\sum_{t=q(i-1)}^{qi-1}\|G(\mathbf{x}^t,\mathbf{v}^t,\eta) \|^2
	\bigg) \nonumber
	\\\leq&  \frac{3L^2\eta^3}{|\mathcal{S}_2|}(K-qK^*)\sum_{p=qK^*}^{K-2}\mathbb{E}\|G(\mathbf{x}^p,\mathbf{v}^p,\eta)\|^2+\frac{3(K-qK^*)q}{|\mathcal{S}_2|}\eta L^2d\delta^2 \nonumber
	\\& +\frac{\eta}{2}(K-qK^*)C-\left(\frac{\eta}{2}-\frac{L\eta^2}{2}\right)\sum_{t=qK^*}^{K-1}\|G(\mathbf{x}^t,\mathbf{v}^t,\eta) \|^2 \nonumber
	\\ &+ \frac{3\eta^3L^2q}{|\mathcal{S}_2|}  \sum_{p=0}^{qK^*-1}\mathbb{E}\|G(\mathbf{x}^p,\mathbf{v}^p,\eta)\|^2+\frac{3K^*q^2\eta L^2d\delta^2}{|\mathcal{S}_2|}+\frac{K^*q\eta}{2} C \nonumber
	\\&-\Big(\frac{\eta}{2}-\frac{L\eta^2}{2}\Big)\sum_{t=0}^{qK^*-1}\|G(\mathbf{x}^t,\mathbf{v}^t,\eta) \|^2
	\nonumber
	\\\leq &-\Big(\frac{\eta}{2}-\frac{L\eta^2}{2}-\frac{3\eta^3L^2q}{|\mathcal{S}_2|}\Big) \sum_{t=0}^{K-1}\mathbb{E}\|G(\mathbf{x}^t,\mathbf{v}^t,\eta) \|^2+\frac{3Kq}{|\mathcal{S}_2|}\eta L^2d\delta^2+ \frac{K\eta}{2}C.
	\end{align}
	Based on the above inequality, we have 
	\begin{align}\label{leqqq}
	\Big(\frac{\eta}{2}-\frac{L\eta^2}{2}-\frac{3\eta^3L^2q}{|\mathcal{S}_2|}\Big) \sum_{t=0}^{K-1}\mathbb{E}\|G(\mathbf{x}^t,\mathbf{v}^t,\eta) \|^2  \leq \Delta_\psi +\frac{3Kq}{|\mathcal{S}_2|}\eta L^2d\delta^2+ \frac{K\eta}{2}C,
	\end{align}
	where $\Delta_\psi=\psi(\mathbf{x}^{0})-\psi(\mathbf{x}^{*})$ with  $\mathbf{x}^*=\arg\min_{\mathbf{x}\in\mathbb{R}^d}\psi(\mathbf{x})$. To simplify notation, we define 
	\begin{align*}
	\tau=\frac{\eta}{2}-\frac{L\eta^2}{2}-\frac{3q}{|\mathcal{S}_2|}\eta^3L^2.
	\end{align*}
	Then, \eqref{leqqq} is simplified to  
	\begin{align}\label{stwo}
	\frac{1}{K} \sum_{t=0}^{K-1}\mathbb{E}\|G(\mathbf{x}^t,\mathbf{v}^t,\eta) \|^2  \leq \frac{\Delta_\psi}{K\tau} +\frac{1}{\tau}\left(\frac{3q}{|\mathcal{S}_2|}\eta L^2d\delta^2+ \frac{\eta}{2}C\right).
	\end{align}
	Using the inequality that $\|\mathbf{x}+\mathbf{y}\|^2\leq 2\|\mathbf{x}\|^2+2\|\mathbf{y}\|^2$, we have 
	\begin{align}\label{egxx}
	\mathbb{E}\|G(\mathbf{x}^\zeta,\nabla f(\mathbf{x}^\zeta),\eta) \|^2 \leq  2\mathbb{E}\|G(\mathbf{x}^\zeta,\mathbf{v}^\zeta,\eta) \|^2+2\mathbb{E}\|G(\mathbf{x}^\zeta,\mathbf{v}^\zeta,\eta) -G(\mathbf{x}^\zeta,\nabla f(\mathbf{x}^\zeta),\eta) \|^2.
	\end{align}
	To upper-bound the first term of the right side of~\eqref{egxx}, we have
	\begin{small}
		\begin{align}\label{first}
		\mathbb{E}\|G(\mathbf{x}^\zeta,\mathbf{v}^\zeta,\eta) \|^2&=\frac{1}{K}\sum_{t=1}^K\mathbb{E}\|G(\mathbf{x}^t,\mathbf{v}^t,\eta) \|^2  
		\leq \frac{\Delta_\psi}{K\tau} +\frac{1}{\tau}\left(\frac{3q}{|\mathcal{S}_2|}\eta L^2d\delta^2+ \frac{\eta}{2}C\right).
		\end{align}
	\end{small}
	\hspace{-0.1cm}For the second term of~\eqref{egxx}, we have 
	\begin{align}\label{second}
	\mathbb{E}\|G(\mathbf{x}^\zeta,\mathbf{v}^\zeta,\eta) -G(\mathbf{x}^\zeta,\nabla f(\mathbf{x}^\zeta),\eta) \|^2\overset{\text{\normalfont{(i)}}}\leq& \mathbb{E}\|\mathbf{v}^\zeta -\nabla f(\mathbf{x}^\zeta)\|^2   \nonumber
	\leq \frac{1}{K}\sum_{t=1}^K\mathbb{E}\|\mathbf{v}^t-\nabla f(\mathbf{x}^t)\|^2\nonumber
	\\ \overset{\text{\normalfont{(ii)}}}\leq &  \frac{6qL^2\eta^2}{|\mathcal{S}_2|}\frac{1}{K}\sum_{t=0}^{K-1}\mathbb{E}\|G(\mathbf{x}^t,\mathbf{v}^t,\eta)\|^2+\frac{6q}{|\mathcal{S}_2|}L^2d\delta^2+C\nonumber
	\\\leq& \frac{6qL^2\eta^2}{|\mathcal{S}_2|\tau}\left(\frac{\Delta_\psi}{K} +\frac{3q}{|\mathcal{S}_2|}\eta L^2d\delta^2+ \frac{\eta}{2}C\right)\nonumber
	\\&+\frac{6q}{|\mathcal{S}_2|}L^2d\delta^2+2C 
	\end{align}
	where (i) follows from Lemma~\ref{g2g}, (ii) follows from~\eqref{vk2} and the last inequality following from~\eqref{stwo}.  Combining~\eqref{egxx},~\eqref{first} and~\eqref{second} yields 
	\begin{align}
	\mathbb{E}\|G(\mathbf{x}^\zeta,\nabla f(\mathbf{x}^\zeta),\eta) \|^2\leq& \frac{2}{\tau}\left( \frac{\Delta_\psi}{K} +\frac{3q}{|\mathcal{S}_2|}\eta L^2d\delta^2+ \frac{\eta}{2}C\right) \nonumber
	\\&+\frac{12qL^2\eta^2}{|\mathcal{S}_2|\tau}\left(\frac{\Delta_\psi}{K} +\frac{3q}{|\mathcal{S}_2|}\eta L^2d\delta^2+ \frac{\eta}{2}C\right)\nonumber
	\\&+\frac{12q}{|\mathcal{S}_2|}L^2d\delta^2+4C, 
	\end{align}
	which finishes the proof.
\end{proof}

\subsection{Proof of Theorem~\ref{coprox}}
Based on Lemma~\ref{proxT}, we next prove our Theorem~\ref{coprox}. 
We prove two cases with $n\leq K$ and $n>K$, respectively. 

First we suppose $n\leq K$. Based on the selected parameters, we have $|\mathcal{S}_1|=n,q=|\mathcal{S}_1|=\lceil n^{1/2}\rceil$, and thus obtain
\begin{align*}
\tau=\frac{3\eta}{16},\,C=\frac{2}{K},\; \frac{qL^2\eta^2}{|\mathcal{S}_2|}=\frac{1}{16}
\end{align*}
which, in conjunction with~\eqref{egah} in Lemma~\ref{proxT}, implies that 
\begin{align*}
\mathbb{E}\|G(\mathbf{x}^\zeta,\nabla f(\mathbf{x}^\zeta),\eta) \|^2\leq \frac{32}{3\eta}\left( \frac{\Delta_\psi}{K}+\frac{4\eta}{K} \right)+\frac{4}{\eta}\left( \frac{\Delta_\psi}{K}+\frac{4\eta}{K}  \right)+\frac{20}{K}\leq \frac{60\Delta_\psi+80}{K} \leq \mathcal{O}\left(\frac{1}{K}\right).
\end{align*}
We choose  $K=C\epsilon^{-1}$, where $C$ is a positive constant. Then, based on the above inequality, for $C$ large enough, our PROX-ZO-SPIDER-Coord achieves an $\epsilon$-approximate stationary  point, i.e., $\mathbb{E}\|G(\mathbf{x}^\zeta,\nabla f(\mathbf{x}^\zeta),\eta) \|^2\leq \epsilon$, and  the total number of function queries is 
\begin{align}\label{pols1}
\left\lceil \frac{K}{q}\right\rceil nd + K |\mathcal{S}_2|d&\leq Kn^{1/2}d +nd+ K n^{1/2}d+Kd  \nonumber 
\\&\leq \mathcal{O}(nd+\epsilon^{-1}n^{1/2}d)\leq \ \mathcal{O}(\epsilon^{-1}n^{1/2}d)\leq \mathcal{O}(d\epsilon^{-3/2}),
\end{align}
where the last two inequalities follow from the assumption that $n\leq K=C\epsilon^{-1}$. 

Next, we suppose $n>K$. In this case, we have $|\mathcal{S}_1|=K,q=|\mathcal{S}_1|=\lceil K^{1/2}\rceil$, and 
\begin{align}
\tau=\frac{3\eta}{16},\,C=\frac{6\sigma^2+2}{K}+\frac{12}{K^2},\; \frac{qL^2\eta^2}{|\mathcal{S}_2|}=\frac{1}{16},
\end{align}
which, in conjunction with~\eqref{egah}, implies that 
\begin{align}
\mathbb{E}\|G(\mathbf{x}^\zeta,\nabla f(\mathbf{x}^\zeta),\eta) \|^2\leq& \frac{32}{3\eta}\left(\frac{\Delta_\psi}{K} +\frac{4\eta}{K}+\frac{3\eta\sigma^2}{K} +\frac{6\eta}{K^2} \right)  \nonumber
\\&+\frac{4}{\eta}\left(\frac{\Delta_\psi}{K} +\frac{4\eta}{K}+\frac{3\eta\sigma^2}{K} +\frac{6\eta}{K^2} \right)+\frac{20+24\sigma^2}{K}+\frac{48}{K^2} \nonumber
\\\leq&\frac{60\Delta_\psi L+80+69\sigma^2}{K} +\frac{138}{K^2}\leq \mathcal{O}\left(\frac{1}{K}\right).
\end{align}
We choose  $K=C\epsilon^{-1}$, where $C>0$ is a constant. Then, based on the above inequality, for $C$ large enough, our PROX-ZO-SPIDER-Coord achieves  $\mathbb{E}\|G(\mathbf{x}^\zeta,\nabla f(\mathbf{x}^\zeta),\eta) \|^2\leq \epsilon$, and  the total number of function queries is 
\begin{align}\label{pols2}
\left\lceil \frac{K}{q}\right\rceil Kd + K |\mathcal{S}_2|d&\leq  Kd+K^{3/2}d+K^{3/2}d+Kd=2K^{3/2}d+2Kd\leq \mathcal{O}(K^{3/2}d)\nonumber
\\&\leq\mathcal{O}(d\epsilon^{-3/2})\leq \mathcal{O}(\epsilon^{-1}n^{1/2}d),
\end{align}
where the last two inequalities follow from the assumption that $n>k\geq 	C\epsilon^{-1}$. 

Combining~\eqref{pols1} and~\eqref{pols2} in  these two cases finishes the proof.

\section{Proof for ZO-SVRG-Coord-Rand-C}
Based on (3) in Lemma~\ref{unifoT}, we first  establish the following key lemma.
\begin{lemma}\label{le:convex}
	Under Assumption~\ref{assumption}, we have, for  any  $qk_0\leq  k\leq \min\{ q(k_0+1)-1, qh\},\,k_0=0,...,h,$
	\begin{align*}
	\mathbb{E}\|\mathbf{v}^k\|^2\leq&  18dL\mathbb{E}(f(\mathbf{x}^k)-f(\mathbf{x}_\beta^*)+f(\mathbf{x}^{qk_0})-f(\mathbf{x}_\beta^*)) \nonumber
	\\&+9L^2d\delta^2+\frac{45\beta^2L^2d^2}{4} +\frac{27I(|\mathcal{S}|<n)}{|\mathcal{S}|}\left(  2L^2d\delta^2+\sigma^2\right)
	\end{align*} 
	where $\mathbf{x}_\beta^*=\arg\min_{\mathbf{x}}f_\beta(\mathbf{x})$. 
\end{lemma}
\begin{proof}
	To simplify notation, we define 
	\begin{align*}
	\widehat \nabla f_{i_k}(\mathbf{x}) =\frac{d(f_{i_k}(\mathbf{x}+\beta \mathbf{u}^{k}) -f_{i_k}(\mathbf{x}))}{\beta}\mathbf{u}^k.
	\end{align*}
	Based on the definition of $\mathbf{v}^k$ in ZO-SVRG-Coord-Rand-C, we have 
	\begin{align}
	\mathbb{E}\|\mathbf{v}^k\|^2=&\mathbb{E}\|\widehat \nabla f_{i_k}(\mathbf{x}^k)-\widehat \nabla f_{i_k}(\mathbf{x}^{qk_0})+\mathbf{v}^{qk_0}\|^2\nonumber
	\\\overset{\text{(i)}}\leq&3\mathbb{E}\|\widehat \nabla f_{i_k}(\mathbf{x}^k)-\widehat \nabla f_{i_k}(\mathbf{x}_\beta^*)\|^2+3\mathbb{E}\|\widehat \nabla f_{i_k}(\mathbf{x}^{qk_0})-\widehat \nabla f_{i_k}(\mathbf{x}_\beta^*)-(\nabla f_\beta(\mathbf{x}^{qk_0}) - \nabla f_\beta(\mathbf{x}_\beta^*))    \|^2\nonumber
	\\&+ 3\mathbb{E}\| \mathbf{v}^{qk_0}-\nabla f_\beta (\mathbf{x}^{qk_0})   \|^2 \nonumber
	\\\overset{\text{(ii)}}\leq&3\mathbb{E}\|\widehat \nabla f_{i_k}(\mathbf{x}^k)-\widehat \nabla f_{i_k}(\mathbf{x}_\beta^*)\|^2+3\mathbb{E}\|\widehat \nabla f_{i_k}(\mathbf{x}^{qk_0})-\widehat \nabla f_{i_k}(\mathbf{x}_\beta^*)  \|^2 +3\mathbb{E}\| \mathbf{v}^{qk_0}-\nabla f_\beta (\mathbf{x}^{qk_0})   \|^2 \nonumber
	\\\overset{\text{(iii)}}\leq& 9d\mathbb{E}\|\nabla f_{i_k}(\mathbf{x}^k)-\nabla f_{i_k}(\mathbf{x}_\beta^*)\|^2+\frac{9L^2d^2\beta^2}{2}+9d\mathbb{E}\|\nabla f_{i_k}(\mathbf{x}^{qk_0})-\nabla f_{i_k}(\mathbf{x}_\beta^*)\|^2+\frac{9L^2d^2\beta^2}{2} \nonumber
	\\&+ 3\mathbb{E}\| \mathbf{v}^{qk_0}-\nabla f_\beta (\mathbf{x}^{qk_0})   \|^2 \nonumber
	\\\overset{\text{(iv)}}\leq&18dL\mathbb{E}(f_{i_k}(\mathbf{x}^k)-f_{i_k}(\mathbf{x}_\beta^*))+18dL\mathbb{E}(f_{i_k}(\mathbf{x}^{qk_0})-f_{i_k}(\mathbf{x}_\beta^*))+9L^2d^2\beta^2 + 3\mathbb{E}\| \mathbf{v}^{qk_0}-\nabla f_\beta (\mathbf{x}^{qk_0})   \|^2 \nonumber
	\\=&18dL\mathbb{E}(f(\mathbf{x}^k)-f(\mathbf{x}_\beta^*)+f(\mathbf{x}^{qk_0})-f(\mathbf{x}_\beta^*))+9L^2d^2\beta^2+ 3\mathbb{E}\| \mathbf{v}^{qk_0}-\nabla f_\beta (\mathbf{x}^{qk_0})   \|^2 \nonumber
	\\\overset{\text{(v)}}\leq & 18dL\mathbb{E}(f(\mathbf{x}^k)-f(\mathbf{x}_\beta^*)+f(\mathbf{x}^{qk_0})-f(\mathbf{x}_\beta^*))+9L^2d\delta^2+\frac{45\beta^2L^2d^2}{4} \nonumber
	\\&+\frac{27I(|\mathcal{S}|<n)}{|\mathcal{S}|}\left(  2L^2d\delta^2+\sigma^2\right)
	\end{align}
	where (i) follows from the equality $\|\mathbf{a+b+c}\|^2\leq 3(\|\mathbf{a}\|^2+\|\mathbf{b}\|^2+\|\mathbf{c}\|^2)$ and the fact that $\nabla f_\beta(\mathbf{x}_\beta^*)=0$, (ii) follows from the inequality $\mathbb{E}\|\mathbf{a}-\mathbb{E}(\mathbf{a})\|^2\leq \mathbb{E}\|\mathbf{a}\|^2$, (iii) follows from  item (3) in Lemma~\ref{unifoT}, (iv) follows from Lemma 5 in~\cite{reddi2016stochastic} that for convex and smooth function $f_{i_k}(\cdot)$,
	\begin{align*}
	\|\nabla f_{i_k}(\mathbf{x})- \nabla f_{i_k}(\mathbf{y}) \|^2\leq 2L(f_{i_k}(\mathbf{x})-f_{i_k}(\mathbf{y})-\langle \nabla f_{i_k}(\mathbf{x}),\mathbf{x-y}  \rangle ),
	\end{align*}
	and (v) follows from~\eqref{otheruse}.  
\end{proof}

Based on Lemma~\ref{unifoT}, we provide the following useful lemma as follows. 
\begin{lemma}\label{th:svrgconvex} 
	let Assumption~\ref{assumption} hold, and define the quantity
	\begin{small}
		\begin{align}
		\lambda=&\eta\left(9L^2d\delta^2+\frac{45\beta^2L^2d^2}{4} +\frac{27I(|\mathcal{S}|<n)}{|\mathcal{S}|}\left(  2L^2d\delta^2+\sigma^2\right) \right) \nonumber
		\\&+2\beta^2L+2 \sqrt{3L^2d\delta^2+\frac{3\beta^2L^2d^2}{4}}\Gamma, 
		\end{align}
	\end{small}
	\hspace{-0.12cm}where $\Gamma=\max_{0\leq k\leq K}\mathbb{E}\|\mathbf{x}^k-\mathbf{x}_\beta^*\|$.
	Then, ZO-SVRG-Coord-Rand-C satisfies 
	\begin{align}\label{main:convex}
	\mathbb{E}\big(f(\mathbf{x}^{K})-f(\mathbf{x}^*)\big)\leq& \alpha^h( f(\mathbf{x}^{0})-f(\mathbf{x}^*_\beta)-\Delta) +\Delta+\beta^2L,
	\end{align}
	where $\mathbf{x}^*=\arg\min_{\mathbf{x}}f(\mathbf{x})$, $\alpha=18d\eta^2L/(2\eta-18d\eta^2L)$ and $\Delta$ is given by
	\begin{align*}
	\Delta=\frac{\Gamma^2/q+\eta \lambda}{2\eta-18d\eta^2L}\left( 1+\frac{18d\eta^2L}{2\eta-36d\eta^2L} \right).
	\end{align*}
\end{lemma}
\begin{proof}
	For $qm\leq k \leq q(m+1)-1, m=0,....,h-1$, we obtain the following sequence of inequalities 
	\begin{align}\label{nicesaa} 
	\mathbb{E} \|\mathbf{x}^{k+1}-\mathbf{x}_\beta^*\|=&\eta^2\mathbb{E}\|\mathbf{v}^k\|^2+\mathbb{E}\|\mathbf{x}^k-\mathbf{x}_\beta^*\|^2-2\eta\mathbb{E}\langle \mathbf{v}^k,\mathbf{x}^k-\mathbf{x}_\beta^*\rangle  \nonumber
	\\= & \eta^2\mathbb{E}\|\mathbf{v}^k\|^2+\mathbb{E}\|\mathbf{x}^k-\mathbf{x}_\beta^*\|^2-2\eta\mathbb{E}\langle \nabla f_\beta(\mathbf{x}^k)-\nabla f_\beta(\mathbf{x}^{qm})+\mathbf{v}^{qm},\mathbf{x}^k-\mathbf{x}_\beta^*\rangle  \nonumber
	\\=& \eta^2\mathbb{E}\|\mathbf{v}^k\|^2+\mathbb{E}\|\mathbf{x}^k-\mathbf{x}_\beta^*\|^2-2\eta\mathbb{E}\langle \nabla f_\beta(\mathbf{x}^k),\mathbf{x}^k-\mathbf{x}_\beta^*\rangle +2\eta\mathbb{E}\langle\nabla f_\beta(\mathbf{x}^{qm})-\mathbf{v}^{qm},\mathbf{x}^k-\mathbf{x}_\beta^*\rangle \nonumber
	\\\overset{\text{(i)}}\leq& \eta^2\mathbb{E}\|\mathbf{v}^k\|^2+\mathbb{E}\|\mathbf{x}^k-\mathbf{x}_\beta^*\|^2-2\eta\mathbb{E}(f_\beta(\mathbf{x}^k)-f_\beta(\mathbf{x}_\beta^*)) +2\eta\mathbb{E}\|\nabla f_\beta(\mathbf{x}^{qm})-\mathbf{v}^{qm}\|\|\mathbf{x}^k-\mathbf{x}_\beta^*\| \nonumber
	\\\overset{\text{(ii)}}\leq& \mathbb{E}\|\mathbf{x}^k-\mathbf{x}_\beta^*\|^2+ 18d\eta^2L\mathbb{E}(f(\mathbf{x}^k)-f(\mathbf{x}_\beta^*)+f(\mathbf{x}^{qm})-f(\mathbf{x}_\beta^*)) -2\eta\mathbb{E}(f_\beta(\mathbf{x}^k)-f_\beta(\mathbf{x}_\beta^*)) \nonumber
	\\&+2\eta \sqrt{3L^2d\delta^2+\frac{3\beta^2L^2d^2}{4}} \Gamma +\eta^2\left(9L^2d\delta^2+\frac{45\beta^2L^2d^2}{4} +\frac{27I(|\mathcal{S}|<n)}{|\mathcal{S}|}\left(  2L^2d\delta^2+\sigma^2\right)\right) \nonumber
	\\\overset{\text{(iii)}}\leq& \mathbb{E}\|\mathbf{x}^k-\mathbf{x}_\beta^*\|^2+ 18d\eta^2L\mathbb{E}(f(\mathbf{x}^k)-f(\mathbf{x}_\beta^*)+f(\mathbf{x}^{qm})-f(\mathbf{x}_\beta^*)) -2\eta\mathbb{E}(f(\mathbf{x}^k)-f(\mathbf{x}_\beta^*)) \nonumber
	\\& +2\eta\beta^2L+2\eta \sqrt{3L^2d\delta^2+\frac{3\beta^2L^2d^2}{4}} \Gamma \nonumber
	\\&+\eta^2\left(9L^2d\delta^2+\frac{45\beta^2L^2d^2}{4} +\frac{27I(|\mathcal{S}|<n)}{|\mathcal{S}|}\left(  2L^2d\delta^2+\sigma^2\right)\right)
	\end{align}
	where (i) follows from the convexity of $f_\beta(\cdot)$ (see (c) of Lemma 4.1 in~\citealt{gao2014information}), (ii) follows from Lemma~\ref{le:convex}, (iii) follows from item (1) in Lemma~\ref{unifoT}. Then, telescoping~\eqref{nicesaa} over $k$ from $qm$ to $q(m+1)-1$, we obtain 
	\begin{align}\label{8lop}
	(2\eta-18d\eta^2L)\sum_{k=qm}^{q(m+1)-1}\mathbb{E}\left(f(\mathbf{x}^k)-f(\mathbf{x}^*_\beta)\right) \leq & \mathbb{E}(\|\mathbf{x}^{qm}-\mathbf{x}_\beta^*\|^2 - \|\mathbf{x}^{q(m+1)}-\mathbf{x}_\beta^*\|^2)+ q\eta\lambda\nonumber
	\\&+18dq\eta^2L\mathbb{E}(f(\mathbf{x}^{qm})-f(\mathbf{x}^*_\beta)).
	\end{align}
	Based on ZO-SVRG-Coord-Rand-C, we have 
	\begin{align*}
	\mathbb{E}f(\mathbf{x}^{q(m+1)})=\frac{1}{q}\sum_{k=qm}^{q(m+1)-1}\mathbb{E}f(\mathbf{x}^k),
	\end{align*}
	which, in conjunction with~\eqref{8lop}, implies that 
	\begin{align}\label{almoss}
	(2\eta-18d\eta^2L) \mathbb{E}\big(f(\mathbf{x}^{q(m+1)})-f(\mathbf{x}^*_\beta)\big)\leq \frac{\Gamma^2}{q}+\eta \lambda+18d\eta^2L\mathbb{E}(f(\mathbf{x}^{qm})-f(\mathbf{x}^*_\beta)). 
	\end{align}
	Then, based on  the selection of $\alpha$ and $\Delta$ in Theorem~\ref{th:svrgconvex}, we obtain from~\eqref{almoss} that 
	\begin{align*}
	\mathbb{E}\big(f(\mathbf{x}^{q(m+1)})-f(\mathbf{x}^*_\beta)\big)-\Delta\leq \alpha \left(\mathbb{E}\big(f(\mathbf{x}^{qm})-f(\mathbf{x}^*_\beta)\big)-\Delta \right).
	\end{align*} 
	Telescoping the above inequality over $m$ from $0$ to $h-1$, we obtain 
	\begin{align}\label{scsasw}
	\mathbb{E}\big(f(\mathbf{x}^{K})-f(\mathbf{x}^*_\beta)\big)-\Delta\leq \alpha^h( \mathbb{E}\big(f(\mathbf{x}^{0})-f(\mathbf{x}^*_\beta)\big)-\Delta).
	\end{align}
	Based on (1) in Lemma~\ref{unifoT} and the definition of $\mathbf{x}^*_\beta$, we have
	\begin{align*}
	f(\mathbf{x}^*)-f(\mathbf{x}^*_\beta)\geq f_\beta(\mathbf{x}^*)-\frac{\beta^2L}{2}-f_\beta(\mathbf{x}^*_\beta)-\frac{\beta^2L}{2}\geq -\beta^2L,
	\end{align*}
	which, in conjunction with~\eqref{scsasw}, yields
	\begin{align}
	\mathbb{E}\big(f(\mathbf{x}^{K})-f(\mathbf{x}^*)\big)-\Delta-\beta^2L\leq \alpha^h( \mathbb{E}\big(f(\mathbf{x}^{0})-f(\mathbf{x}^*_\beta)\big)-\Delta).
	\end{align}
	Then, the proof is complete. 
\end{proof}

\subsection{Proof of Theorem~\ref{coco:convex}}\label{prof:co3}
Based on Lemmas~\ref{unifoT} and~\ref{th:svrgconvex}, we now prove Theorem~\ref{coco:convex}. 
We prove two cases with  $n\leq \lceil c_s/\epsilon\rceil$ and $n<\lceil c_s/\epsilon\rceil$, separately. 

First suppose that $n\leq \lceil c_s/\epsilon\rceil$, and thus $|\mathcal{S}|=n$. Then, 
applying the parameters selected in Corollary~\ref{coco:convex} in Theorem~\ref{th:svrgconvex}, we obtain $\alpha=1/2$ and 
\begin{align*}
\lambda &=\frac{\epsilon^2}{3c^2_\delta dL}+\frac{5\epsilon^2}{12c^2_\beta  dL}+\frac{2\epsilon^2}{c_\beta^2d^2L} + 2\Gamma \sqrt{\frac{3\epsilon^2}{c_\delta^2} +\frac{3\epsilon^2}{4c_\beta^2}}, \nonumber
\\
\Delta&= \frac{3}{2}\left(\frac{\Gamma^2}{q\eta}+\lambda\right)\leq \frac{3}{2}\left(\frac{ 27L\Gamma^2\epsilon}{c_q}+\frac{\epsilon^2}{3c^2_\delta dL}+\frac{5\epsilon^2}{12c^2_\beta  dL}+\frac{2\epsilon^2}{c_\beta^2d^2L} + 2\Gamma \sqrt{\frac{3\epsilon^2}{c_\delta^2} +\frac{3\epsilon^2}{4c_\beta^2}}\right)
\end{align*}
which, in conjunction with~\eqref{main:convex}, implies that 
\begin{align}\label{niopsa}
\mathbb{E}\left(f(\mathbf{x}^K)-f(\mathbf{x}^*)\right)\leq& \frac{(f(\mathbf{x}^{0})-f(\mathbf{x}^*_\beta)) \epsilon}{c_h} + \frac{3}{2}\left(\frac{ 27L\Gamma^2\epsilon}{c_q}+\frac{\epsilon^2}{3c^2_\delta dL}+\frac{5\epsilon^2}{12c^2_\beta  dL}+\frac{2\epsilon^2}{c_\beta^2d^2L} + 2\Gamma \sqrt{\frac{3\epsilon^2}{c_\delta^2}  +\frac{3\epsilon^2}{4c_\beta^2}}\right) \nonumber
\\&+ \frac{\epsilon^2}{c_\beta^2d^2L}.
\end{align}
For $c_h,c_q,c_\beta,c_\delta$ large enough, we obtain from~\eqref{niopsa} that $\mathbb{E}\left(f(\mathbf{x}^K)-f(\mathbf{x}^*)\right)\leq \epsilon$, and the number of function queries required by ZO-SVRG-Coord-Rand-C  is at most 
\begin{align*}
\left\lceil \frac{K}{q} \right\rceil nd + K&=hnd+hq =\log_2 (c_h/\epsilon)nd+\log_2 (c_h/\epsilon) c_qd/\epsilon\leq \mathcal{O}\left(d(n+1/\epsilon)\log (1/\epsilon) \right) \nonumber
\\&\leq \mathcal{O}(d\min\{n,1/\epsilon\}\log(1/\epsilon)),
\end{align*}
where the last inequality follows from the assumption that $n\leq \left\lceil c_s/\epsilon \right \rceil$. 

Next, suppose $n>\lceil c_s/\epsilon\rceil$, and thus $|\mathcal{S}|=\lceil c_s/\epsilon\rceil$. Then, we obtain 
\begin{align*}
\lambda &=\frac{\epsilon^2}{3c^2_\delta dL}+\frac{5\epsilon^2}{12c^2_\beta  dL}+\frac{2\epsilon^2}{c_\beta^2d^2L} + 2\Gamma \sqrt{\frac{3\epsilon^2}{c_\delta^2} +\frac{3\epsilon^2}{4c_\beta^2}}+ \frac{2\epsilon^3+\sigma^2\epsilon}{c_sdL} \nonumber
\\
\Delta&= \frac{3}{2}\left(\frac{\Gamma^2}{q\eta}+\lambda\right)\leq \frac{3}{2}\left(\frac{ 27L\Gamma^2\epsilon}{c_q}+\frac{\epsilon^2}{3c^2_\delta dL}+\frac{5\epsilon^2}{12c^2_\beta  dL}+\frac{2\epsilon^2}{c_\beta^2d^2L} + 2\Gamma \sqrt{\frac{3\epsilon^2}{c_\delta^2} +\frac{3\epsilon^2}{4c_\beta^2}}+ \frac{2\epsilon^3+\sigma^2\epsilon}{c_sdL} \right)
\end{align*}
which, in conjunction with~\eqref{main:convex}, implies that 
\begin{align}\label{niopsa11}
\mathbb{E}\left(f(\mathbf{x}^K)-f(\mathbf{x}^*)\right)\leq& \frac{3}{2}\left(\frac{ 27L\Gamma^2\epsilon}{c_q}+\frac{\epsilon^2}{3c^2_\delta dL}+\frac{5\epsilon^2}{12c^2_\beta  dL}+\frac{2\epsilon^2}{c_\beta^2d^2L} + 2\Gamma \sqrt{\frac{3\epsilon^2}{c_\delta^2}  +\frac{3\epsilon^2}{4c_\beta^2}}+ \frac{2\epsilon^3+\sigma^2\epsilon}{c_sdL}\right) \nonumber
\\&+\frac{(f(\mathbf{x}^{0})-f(\mathbf{x}^*_\beta)) \epsilon}{c_h} +  \frac{\epsilon^2}{c_\beta^2d^2L}.
\end{align}
For $c_h,c_q,c_\beta,c_\delta,c_s$ large enough, we obtain from~\eqref{niopsa} that $\mathbb{E}\left(f(\mathbf{x}^K)-f(\mathbf{x}^*)\right)\leq \epsilon$, and the number of function queries required by our ZO-SVRG-Coord-Rand-C  is 
\begin{align*}
\left\lceil \frac{K}{q} \right\rceil |\mathcal{S}|d + K&=h|\mathcal{S}|d+hq \leq \mathcal{O}\left(d(1/\epsilon)\log (1/\epsilon) \right) \leq \mathcal{O}(d\min\{n,1/\epsilon\}\log(1/\epsilon)),
\end{align*}
where the last inequality follows from the assumption that $n>\lceil c_s/\epsilon\rceil$.

\section{Proofs for ZO-SPIDER-Coord-C}\label{appen:sarah}
\subsection{Auxiliary Lemma}
To prove the main theorem, we first establish two useful lemmas.  
\begin{lemma}\label{le:sarah}
	For  any $qk_0\leq m\leq q(k_0+1), k_0=0,...., h-1$,  we have 
	\begin{align*}
	\mathbb{E}\|	\mathbf{v}^k-\hat \nabla_{\text{\normalfont coord}}f(\mathbf{x}^k)\|^2 \leq& 6(k-qk_0)L^2d\delta^2+3\sum_{m=qk_0+1}^k\mathbb{E}\|\nabla f_{i_m}(\mathbf{x}^{m})-\nabla f_{i_{m}}(\mathbf{x}^{m-1})\|^2 \nonumber
	\\&+
	\frac{3I(|\mathcal{S}|<n)}{|\mathcal{S}|}\left(  2L^2d\delta^2+\sigma^2\right).
	\end{align*}
	where  we define $\sum_{t=qk_0}^{qk_0-1}\mathbb{E}\|\mathbf{v}^t\|^2=0$ for simplicity. 
\end{lemma}
\begin{proof}
	Using an approach similar  to~\eqref{nuews} in Lemma~\ref{le:coord} with $|\mathcal{S}_2|=1$, we obtain, for $qk_0+1\leq m\leq k$,
	\begin{align*}
	\mathbb{E}&\|	\mathbf{v}^m-\hat \nabla_{\text{\normalfont coord}}f(\mathbf{x}^m)\|^2 \leq 6L^2d\delta^2+3\|\nabla f_{i_m}(\mathbf{x}^{m})-\nabla f_{i_{m}}(\mathbf{x}^{m-1})\|^2+\mathbb{E}\|	\mathbf{v}^{m-1}-\hat \nabla_{\text{\normalfont coord}}f(\mathbf{x}^{m-1})\|^2.
	\end{align*} 
	Telescoping the above inequality over $m$ from $qk_0+1$ to $k$ yields 
	\begin{align}
	\mathbb{E}\|	\mathbf{v}^k-\hat \nabla_{\text{\normalfont coord}}f(\mathbf{x}^k)\|^2 \leq& 6(k-qk_0)L^2d\delta^2+3\sum_{m=qk_0+1}^k\|\nabla f_{i_m}(\mathbf{x}^{m})-\nabla f_{i_{m}}(\mathbf{x}^{m-1})\|^2 \nonumber
	\\&+\mathbb{E}\|	\mathbf{v}^{qk_0}-\hat \nabla_{\text{\normalfont coord}}f(\mathbf{x}^{qk_0})\|^2 \nonumber
	\end{align}
	which, in conjunction with Lemma~\ref{evs1}, finishes the proof. 
\end{proof}
\begin{lemma}\label{le:sarah2}
	For  any $qk_0\leq m\leq q(k_0+1), k_0=0,...., h-1$, we have 
	\begin{align}
	\sum_{m=qk_0+1}^k\mathbb{E}\|\hat \nabla _{\text{\normalfont coord}} f_{i_m}(\mathbf{x}^m)- \hat \nabla _{\text{\normalfont coord}} f_{i_m}(\mathbf{x}^{m-1})\|^2\leq \frac{L\eta}{2-L\eta}\mathbb{E}\|\mathbf{v}^{qk_0}\|^2.
	\end{align}
	\begin{proof}
		Define a smoothing function of $f(\mathbf{x})$ with regard to its $i^{th}$ coordinate as $f_{i,\delta}(\mathbf{x})=\mathbb{E}_{\mathbf{v}\sim \text{U}(-\delta,\delta)}(f(\mathbf{x}+\mathbf{v}\mathbf{e}_i))$, where $\text{U}(-\delta, \delta)$ denotes the uniform distribution over  the range $[-\delta, \delta]$. Then, based on Lemma 6 in~\cite{lian2016comprehensive}, the function $f_{i,\delta}(\mathbf{x})$ has  the following three useful properties:
		\begin{itemize}
			\item[(1)] $\mathbf{e}_i\mathbf{e}_i^T\nabla  f_{i,\delta}(\mathbf{x})=\frac{1}{2\delta}\left(f(\mathbf{x}+\delta\mathbf{e}_i)-f(\mathbf{x}+\delta\mathbf{e}_i) \right)\mathbf{e}_i$
			\item[(2)] If $f(\mathbf{x})$ has the $L$-Lipschitz gradient, then $f_{i,\delta}(\mathbf{x})$ also has the $L$-Lipschitz gradient.
			\item[(3)] If $f(\mathbf{x})$ is convex, then $f_{i,\delta}(\mathbf{x})$ is convex. 
		\end{itemize}
		Based the above preliminaries, we next prove Lemma~\ref{le:sarah2}. Recall from ZO-SPIDER-Coord-C that 
		\begin{align}\label{spos}
		\mathbf{v}^m=\hat \nabla _{\text{coord}} f_{i_m}(\mathbf{x}^m)- \hat \nabla _{\text{coord}} f_{i_m}(\mathbf{x}^{m-1})+
		\mathbf{v}^{k-1},
		\end{align} 
		where we recall that for $\mathbf{x}=\mathbf{x}^m$ and $\mathbf{x}^{m-1}$
		\begin{align}
		\hat \nabla_{\text{\normalfont coord}}f_{i_m}(\mathbf{x}^{m})=\sum_{i=1}^d\frac{1}{2\delta}(f_{i_m}(\mathbf{x}+\delta\mathbf{e}_i)-f_{i_m}(\mathbf{x}-\delta\mathbf{e}_i))\mathbf{e}_i=\sum_{i=1}^d\mathbf{e}_i\mathbf{e}_i^T\nabla  f_{i_m,i,\delta}(\mathbf{x}).
		\end{align}
		Then, based on~\eqref{spos}, we have 
		\begin{align}\label{102ss}
		\|\mathbf{v}^m\|^2=&\|\mathbf{v}^{m-1}\|^2+\|\hat \nabla _{\text{coord}} f_{i_m}(\mathbf{x}^m)- \hat \nabla _{\text{coord}} f_{i_m}(\mathbf{x}^{m-1})\|^2 \nonumber
		\\&
		+2\underbrace{\langle \hat \nabla _{\text{coord}} f_{i_m}(\mathbf{x}^m)- \hat \nabla _{\text{coord}} f_{i_m}(\mathbf{x}^{m-1}), \mathbf{v}^{m-1} \rangle}_{\text{(I)}}.
		\end{align}
		We next upper-bound the term (I) in the above equation using the convexity of function $f_{i_k}(\cdot)$. In specific, we have
		\begin{align}\label{newms}
		\text{(I)}=&-\frac{1}{\eta}\left\langle \sum_{i=1}^d\mathbf{e}_i\mathbf{e}_i^T\left(\nabla  f_{i_m,i,\delta}(\mathbf{x}^m)-\nabla  f_{i_m,i,\delta}(\mathbf{x}^{m-1})\right), \mathbf{x}^m-\mathbf{x}^{m-1}  \right\rangle \nonumber
		\\\overset{\text{(i)}}=&-\sum_{i=1}^d\frac{1}{\eta}\left\langle\nabla  f_{i_m,i,\delta}(\mathbf{e}_i\mathbf{e}_i^T\mathbf{x}^m)-\nabla  f_{i_m,i,\delta}(\mathbf{e}_i\mathbf{e}_i^T\mathbf{x}^{m-1}), \mathbf{e}_i\mathbf{e}_i^T\mathbf{x}^m-\mathbf{e}_i\mathbf{e}_i^T\mathbf{x}^{m-1}  \right\rangle \nonumber
		\\\overset{\text{(ii)}}\leq &-\sum_{i=1}^d\frac{1}{L\eta}\left\|\nabla  f_{i_m,i,\delta}(\mathbf{e}_i\mathbf{e}_i^T\mathbf{x}^m)-\nabla  f_{i_m,i,\delta}(\mathbf{e}_i\mathbf{e}_i^T\mathbf{x}^{m-1}) \right\|^2 \nonumber
		\\=&-\sum_{i=1}^d\frac{1}{L\eta}\left\|\mathbf{e}_i\mathbf{e}_i^T\left(\nabla  f_{i_m,i,\delta}(\mathbf{x}^m)-\nabla  f_{i_m,i,\delta}(\mathbf{x}^{m-1})\right) \right\|^2  \nonumber
		\\=&-\frac{1}{L\eta}\left\|\hat \nabla _{\text{coord}} f_{i_m}(\mathbf{x}^m)- \hat \nabla _{\text{coord}} f_{i_m}(\mathbf{x}^{m-1})\right\|^2.
		\end{align}
		where (i) follows from the definition of $\mathbf{e}_i$, (ii) follows from the convexity of $f_{i_m,i,\delta}(\cdot)$ and  Theorem 2.1.5 in~\cite{nesterov2013introductory}, and the last inequality follows from the definition of $\mathbf{e}_i$ and the $\ell_2$-norm.  Combining~\eqref{102ss} and~\eqref{newms} implies that 
		\begin{align*}
		\|\mathbf{v}^m\|^2=\|\mathbf{v}^{m-1}\|^2+\left(1-\frac{2}{L\eta} \right)\|\hat \nabla _{\text{coord}} f_{i_m}(\mathbf{x}^m)- \hat \nabla _{\text{coord}} f_{i_m}(\mathbf{x}^{m-1})\|^2
		\end{align*}
		Telescoping the above inequality over $m$ from $qk_0+1$ to $k$ and taking the expectation, we finish the proof.
	\end{proof}
\end{lemma}

Based on Lemmas~\ref{le:sarah} and~\ref{le:sarah2}, we next prove the following useful lemma. 
\begin{lemma}\label{csarah}
	Under Assmption~\ref{assumption}, we define 
	\begin{align}\label{para-sarah}
	\lambda=3qL^2d\delta^2&+\frac{6L\eta}{2-L\eta}L^2d\delta^2+\frac{3I(|\mathcal{S}|<n)}{|\mathcal{S}|}\left(  2L^2d\delta^2+\sigma^2\right).
	\end{align}
	Then, our ZO-SPIDER-Coord-C satisfies
	\begin{align*}
	\mathbb{E}\|\nabla f(\mathbf{x}^{K})\|^2\leq  \alpha^h\mathbb{E}\|\nabla f(\mathbf{x}^{0})\|^2 +\frac{1-\alpha^h}{1-\alpha} \Delta,
	\end{align*}
	with the parameters satisfying $\alpha= 6L(\eta+2L\eta^2)(2-L\eta)^{-1}(1/2-L\eta)^{-1}$ and 
	\begin{align}\label{para-sarah2}
	\Delta=\frac{\Gamma}{q(\frac{\eta}{2}-L\eta^2)} +\frac{1+2L\eta}{\frac{1}{2}-L\eta}\lambda, 
	\end{align}
	where $\Gamma=\max_{0\leq k\leq h}\{\mathbb{E}\left(f(\mathbf{x}^{qk})-f(\mathbf{x}^{*})\right)\}$ with $\mathbf{x}^*=\arg\min_{\mathbf{x}}f(\mathbf{x})$.
\end{lemma}
\begin{proof}
	Since $f(\cdot)$ has the $L$-Lipschitz  gradient,  we have, for $qk_0\leq m\leq q(k_0+1), k_0=0,...., h-1$, 
	\begin{align}
	f(\mathbf{x}^{m+1})&\leq f(\mathbf{x}^{m}) +\langle  \nabla f(\mathbf{x}^m),\mathbf{x}^{m+1}-\mathbf{x}^{m} \rangle + \frac{L}{2}\|\mathbf{x}^{m+1}-\mathbf{x}^{m} \|^2 \nonumber
	\\&=f(\mathbf{x}^m)-\eta\langle  \mathbf{v}^m-\nabla f(\mathbf{x}^m),\nabla f(\mathbf{x}^m)  \rangle -\eta\|\nabla f(\mathbf{x}^m)\|^2+\frac{L\eta^2}{2}\|\mathbf{v}^m\|^2  \nonumber
	\\&\leq f(\mathbf{x}^m) +\frac{\eta}{2}\|\nabla f(\mathbf{x}^m) -\mathbf{v}^m\|^2-\frac{\eta}{2}\|\nabla f(\mathbf{x}^m)\|^2+L\eta^2\|\mathbf{v}^m-\nabla f(\mathbf{x}^m)\|^2  +L\eta^2\|\nabla f(\mathbf{x}^m)\|^2\nonumber
	\\&= f(\mathbf{x}^m) +\left( \frac{\eta}{2}   +L\eta^2\right)\|\nabla f(\mathbf{x}^m) -\mathbf{v}^m\|^2-\Big(\frac{\eta}{2}-L\eta^2\Big)\|\nabla f(\mathbf{x}^m)\|^2, \nonumber
	\\&\leq f(\mathbf{x}^m) +\left( \eta  +2L\eta^2\right)\left(\| \hat \nabla_{\text{\normalfont coord}}f(\mathbf{x}^m) -\mathbf{v}^m\|^2+\|\nabla f(\mathbf{x}^m)- \hat\nabla_{\text{\normalfont coord}}f(\mathbf{x}^m)\|^2\right)-\Big(\frac{\eta}{2}-L\eta^2\Big)\|\nabla f(\mathbf{x}^m)\|^2. \nonumber
	\end{align}
	Taking expectation over the above inequality and using  Lemmas~\ref{coordinate},~\ref{le:sarah} and~\ref{le:sarah2}, we have
	\begin{align}
	\mathbb{E}f(\mathbf{x}^{m+1})\leq& \mathbb{E}f(\mathbf{x}^m) +\left( \eta  +2L\eta^2\right)\left(\mathbb{E}\| \hat \nabla_{\text{\normalfont coord}}f(\mathbf{x}^m) -\mathbf{v}^m\|^2+L^2d\delta^2\right)-\Big(\frac{\eta}{2}-L\eta^2\Big)\mathbb{E}\|\nabla f(\mathbf{x}^m)\|^2 \nonumber
	\\\leq& \mathbb{E}f(\mathbf{x}^m) +\left( \eta  +2L\eta^2\right)\Big( 6(k-qk_0)L^2d\delta^2+L^2d\delta^2+3\sum_{m=qk_0+1}^k\mathbb{E}\|\nabla f_{i_m}(\mathbf{x}^{m})-\nabla f_{i_{m}}(\mathbf{x}^{m-1})\|^2 \nonumber
	\\&+
	\frac{3I(|\mathcal{S}_1|<n)}{|\mathcal{S}_1|}\left(  2L^2d\delta^2+\sigma^2\right)\Big)-\Big(\frac{\eta}{2}-L\eta^2\Big)\mathbb{E}\|\nabla f(\mathbf{x}^m)\|^2  \nonumber
	\\\overset{\text{(i)}}\leq& \mathbb{E}f(\mathbf{x}^m) +\left( \eta  +2L\eta^2\right)\Big( 6(k-qk_0)L^2d\delta^2+L^2d\delta^2+ \frac{3L\eta}{2-L\eta}\mathbb{E}\|\mathbf{v}^{qk_0}\|^2+\frac{3I(|\mathcal{S}|<n)}{|\mathcal{S}|}\left(  2L^2d\delta^2+\sigma^2\right)\Big) \nonumber
	\\&
	-\Big(\frac{\eta}{2}-L\eta^2\Big)\mathbb{E}\|\nabla f(\mathbf{x}^m)\|^2  \nonumber
	\end{align}
	where (i) follows from Lemma~\ref{le:sarah2}.  
	Noting that $\mathbf{v}^{qk_0}= \hat\nabla_{\text{\normalfont coord}}f(\mathbf{x}^{qk_0})$ and telescoping the above inequality over $m$ from $qk_0$ to $q(k_0+1)-1$, we obtain 
	\begin{align}\label{newjes}
	&\sum_{m=qk_0}^{q(k_0+1)-1}\Big(\frac{\eta}{2}-L\eta^2\Big)\mathbb{E}\|\nabla f(\mathbf{x}^m)\|^2   \leq \mathbb{E}f(\mathbf{x}^{qk_0})-\mathbb{E}f(\mathbf{x}^{q(k_0+1)}) \nonumber
	\\&+\left( \eta  +2L\eta^2\right)\Big( 3q^2L^2d\delta^2+ \frac{6qL\eta}{2-L\eta}\mathbb{E}\|\nabla f(\mathbf{x}^{qk_0})\|^2+\frac{6qL\eta}{2-L\eta}L^2d\delta^2+\frac{3qI(|\mathcal{S}|<n)}{|\mathcal{S}|}\left(  2L^2d\delta^2+\sigma^2\right)\Big)
	\end{align}
	Combining~\eqref{para-sarah} with~\eqref{newjes}  implies that  
	\begin{align}
	&\sum_{m=qk_0}^{q(k_0+1)-1}\Big(\frac{\eta}{2}-L\eta^2\Big)\mathbb{E}\|\nabla f(\mathbf{x}^m)\|^2   \leq \mathbb{E}f(\mathbf{x}^{qk_0})-\mathbb{E}f(\mathbf{x}^{*}) +\left( \eta  +2L\eta^2\right)\Big(  \frac{6qL\eta}{2-L\eta}\mathbb{E}\|\nabla f(\mathbf{x}^{qk_0})\|^2+q\lambda\Big), \nonumber
	\end{align}
	which, in conjunction with the fact that $\mathbf{x}^{q(k_0+1)}$ is generated from $\{\mathbf{x}^{qk_0},...,\mathbf{x}^{q(k_0+1)-1}\}$ uniformly at random and~\eqref{para-sarah2}, yields
	\begin{align*}
	\mathbb{E}\|\nabla f(\mathbf{x}^{q(k_0+1)})\|^2\leq  \alpha\mathbb{E}\|\nabla f(\mathbf{x}^{qk_0})\|^2 + \Delta.
	\end{align*}
	Telescoping the above inequality over $k_0$ from $0$ to $h-1$ yields 
	\begin{align}
	\mathbb{E}\|\nabla f(\mathbf{x}^{K})\|^2\leq  \alpha^h\mathbb{E}\|\nabla f(\mathbf{x}^{0})\|^2 +\frac{1-\alpha^h}{1-\alpha} \Delta,
	\end{align}
	which finishes the proof. 
\end{proof}

\subsection{Proof of Theorem~\ref{co:sarass}}
Using Lemmas~\ref{le:sarah},~\ref{le:sarah2} and~\ref{csarah}, we prove Theorem~\ref{co:sarass}. 
We prove two cases with  $n\leq \lceil c_s/\epsilon\rceil$ and $n<\lceil c_s/\epsilon\rceil$, separately. 

First suppose that $n\leq \lceil c_s/\epsilon\rceil$, and thus $|\mathcal{S}|=n$. Then, 
applying the parameters selected in Corollary~\ref{co:sarass} in Theorem~\ref{csarah}, we obtain $\alpha\leq 1/2$ and $
\Delta\leq \mathcal{O}(\epsilon/c_q)$
which, in conjunction with~\eqref{main:convex}, implies that 
\begin{align}\label{niopsa1}
\mathbb{E}\|\nabla f(\mathbf{x}^{K})\|^2\leq & \mathcal{O}\left(  \frac{\epsilon}{c_h} + \frac{\epsilon}{c_q}\right).
\end{align}
For $c_h,c_q$ large enough, we obtain from~\eqref{niopsa1} that $\mathbb{E}\|\nabla f(\mathbf{x}^{K})\|^2\leq \epsilon$, and the number of function queries required by our ZO-SPIDER-Coord-C is at most 
\begin{align*}
\left\lceil \frac{K}{q} \right\rceil nd + Kd&=hnd+hqd =\log_2 (c_h/\epsilon)nd+\log_2 (c_h/\epsilon) c_qd/\epsilon\leq \mathcal{O}\left(d(n+1/\epsilon)\log (1/\epsilon) \right) \nonumber
\\&\leq \mathcal{O}(d\min\{n,1/\epsilon\}\log(1/\epsilon)),
\end{align*}
where the last inequality follows from the assumption that $n\leq \left\lceil c_s/\epsilon \right \rceil$. 

Next, suppose $n>\lceil c_s/\epsilon\rceil$, and thus $|\mathcal{S}|=\lceil c_s/\epsilon\rceil$. Then, we similarly obtain 
\begin{align*}
\mathbb{E}\|\nabla f(\mathbf{x}^{K})\|^2\leq& \mathcal{O}\left(\frac{\epsilon}{c_q}+\frac{\epsilon}{c_h}+\frac{\epsilon}{c_s}\right). 
\end{align*}
Then. for $c_h,c_q,c_s$ large enough, we obtain from~\eqref{niopsa} that $\mathbb{E}\|\nabla f(\mathbf{x}^{K})\|^2\leq \epsilon$, and the number of function queries required by  ZO-SPIDER-Coord-C  is  given by 
\begin{align*}
\left\lceil \frac{K}{q} \right\rceil |\mathcal{S}|d + Kd&=h|\mathcal{S}|d+hdq \leq \mathcal{O}\left(d(1/\epsilon)\log (1/\epsilon) \right) \leq \mathcal{O}(d\min\{n,1/\epsilon\}\log(1/\epsilon)),
\end{align*}
where the last inequality follows from the assumption that $n>\lceil c_s/\epsilon\rceil$.

\end{document}